\documentclass[accepted]{uai2026} 
                        
\usepackage[american]{babel}

\usepackage{natbib} 
\usepackage{mathtools} 
\usepackage{booktabs} 
\usepackage{tikz} 
\usepackage{amsthm}
\usepackage{amssymb}
\usepackage{hyperref}
\usepackage{dsfont}
\usepackage{stmaryrd}
\usepackage{multirow} 
\newtheorem{definition}{Definition}[section]
\newtheorem{theorem}{Theorem}[section]
\newtheorem{proposition}{Proposition}[section]

\usepackage{algorithm}
\usepackage{algpseudocode}

\renewcommand{\vec}[1]{\boldsymbol{#1}}
\renewcommand{\cite}[1]{\citep{#1}}
\newtheorem*{theorem*}{Theorem}

\title{Reliable Conformal Prediction for Ordinal Classification\\ Using the Ranked Probability Score}

\author[1,3]{\href{mailto:stefan.sh.haas@bmwgroup.com?Subject=Reliable Conformal Prediction for Ordinal Classification Using the Ranked Probability Score}{Stefan Haas}{}}
\author[1]{Luca Killmaier}
\author[1,2]{Alireza Javanmardi}
\author[1,2,4]{Eyke Hüllermeier}
\affil[1]{%
    Institute of Informatics\\
    LMU Munich\\
    Munich, Germany
}
\affil[2]{%
    Munich Center for Machine Learning (MCML)\\
    Munich, Germany\\
}
\affil[3]{%
    BMW Group\\
     Munich, Germany
  }
  
\affil[4]{%
   German Centre for Artificial Intelligence (DFKI, DSA)\\
     Kaiserslautern, Germany
}
  
\begin{document}
\maketitle

\begin{abstract}
Ordinal classification (OC) arises in high-stakes domains such as medicine and finance, where uncertainty quantification must account for the severity of ordinal errors. Conformal prediction (CP) provides distribution-free prediction sets with marginal coverage guarantees; however, its practical effectiveness depends critically on the choice of nonconformity function.
We introduce a CP method for ordinal classification based on the ranked probability score (RPS), a proper scoring rule defined over cumulative predictive distributions. Although it reflects ordinal risk quite naturally, it has largely been neglected in conformal ordinal prediction (COP). When used as a measure of nonconformity, RPS yields median-centered contiguous prediction sets by construction. The method is model-agnostic, supports both assessed and grouped ordered categorical outcomes, and permits efficient implementation compared to greedy interval selection procedures.
Across multiple ordinal image and tabular datasets, RPS-based CP produces contiguous prediction sets and strikes a favorable balance between prediction set width and the magnitude of ordinal miscoverage relative to existing CP methods.
\end{abstract}

\section{Introduction}\label{sec:intro}

Ordinal classification (OC), also known as ordinal regression in statistics~\cite{mccullagh1980regression}, refers to classification problems in which class labels exhibit a natural linear order. 
Representative applications include medical diagnosis~\cite{DBLP:journals/peerj-cs/AlbuquerqueCC21}, age estimation~\cite{DBLP:journals/prl/CaoMR20}, and credit risk assessment~\cite{DBLP:journals/sma/HirkHV19}. Despite its prevalence in high-stakes domains, most work in OC has primarily focused on improving point prediction performance~\cite{DBLP:journals/paa/ShiCR23,DBLP:conf/aaai/NachmaniGSSG25,DBLP:journals/eswa/PolatCT25}, whereas uncertainty quantification (UQ) has attracted attention only recently~\cite{DBLP:journals/ijar/HaasH25, DBLP:journals/corr/abs-2507-00733}. Here, uncertainty for a query $\vec{x}_q$ is typically represented by the conditional predictive distribution over ordered labels, $p(y \mid \vec{x}_q)$, produced by a probabilistic predictor.

Alternatively, uncertainty can be represented by set-valued predictions. Conformal prediction (CP) offers a principled, model-agnostic framework for post-hoc construction of prediction sets  ~\cite{DBLP:journals/jmlr/ShaferV08,vovk2005algorithmic,DBLP:conf/icml/VovkGS99,DBLP:journals/corr/abs-2107-07511}.
It calibrates an underlying (heuristic) uncertainty estimate to achieve finite-sample, distribution-free marginal coverage at a user-specified miscoverage rate $\alpha$. Instead of committing to a single label $y \in \mathcal{Y}$, CP outputs a set $\mathcal{C}_\alpha(\vec{x}_q) \subseteq \mathcal{Y}$ of plausible labels, whose size reflects predictive uncertainty at the query $\vec{x}_q$. However, producing sets in conformal ordinal prediction (COP) that are both informative and consistent with the ordinal structure of the label space remains challenging.

A key requirement for COP is that prediction sets be contiguous \cite{DBLP:conf/miccai/LuAP22,DBLP:conf/uai/XuGW23,DBLP:conf/nips/DeyMK23}. For example, consider age estimation~\cite{DBLP:journals/paa/ShiCR23,DBLP:conf/caepia/YunGGBGH24} from images, where a latent continuous variable (e.g., age) is discretized into ordered categories, with $\mathcal{Y} = \{\texttt{baby}, \texttt{child}, \texttt{teenager}, \texttt{adult}, \texttt{senior}\}$. A valid prediction set should only contain adjacent age categories, e.g., $\mathcal{C}_\alpha(\vec{x}_q) = \{\texttt{child}, \texttt{teenager}, \texttt{adult}\}$, whereas non-contiguous sets such as $\mathcal{C}_\alpha(\vec{x}_q) = \{\texttt{child}, \texttt{senior}\}$ may appear unreasonable.


\begin{figure*}[htb]
    \centering
    \includegraphics[width=0.6\textwidth]{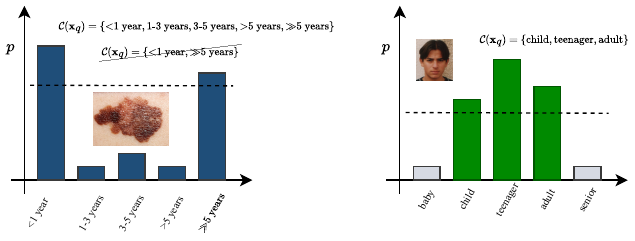}
    \caption{
(Left) Illustration of an assessed ordered categorical variable (survival prognosis for melanoma~\cite{NCI1985_AV8500_3850}) with extreme disagreement between physicians. To faithfully quantify uncertainty, the prediction set must be contiguous and include all intermediate classes between the conflicting assessments.
(Right) Illustration of a grouped ordered categorical variable (age estimation), where unimodal predictive modeling is well justified and naturally leads to contiguous prediction sets~\cite{lanitis2002toward,panis2016overview}.
    }
    \label{fig:assessed_grouped}
\end{figure*}

Another important category, alongside the previously described \emph{grouped} ordered categorical variables, is the class of \emph{assessed} ordered categorical variables~\cite{anderson1984regression}, in which human experts assign labels, as in financial risk assessment or medical survival prognosis. In these contexts, errors tend to be inherently larger due to inter-expert disagreement, which makes maintaining contiguity even more crucial for accurately capturing uncertainty.
For example, consider physician opinions regarding survival prognosis for stage IV melanoma, which may be polarized, resulting in large clusters at \texttt{<1 year} (very pessimistic group) and \texttt{$\gg$5 years} (optimistic group influenced by immunotherapy outcomes). To properly quantify uncertainty in such cases, the prediction set should not be limited to the two most frequent categories, i.e., $\mathcal{C}_\alpha(\vec{x}_q) = \{\texttt{<1 year}, \texttt{$\gg$5 years}\}$; instead, it should encompass the entire range of plausible categories, $\mathcal{C}_\alpha(\vec{x}_q) = \{\texttt{<1 year}, \texttt{1-3 years}, \texttt{3-5 years}, \texttt{>5 years}, \\ \texttt{$\gg$5 years}\}$. Otherwise, uncertainty measured by the size of the prediction set will be severely underestimated (see Figure~\ref{fig:assessed_grouped}).


Another important aspect of COP that has received limited attention is the \emph{severity of miscoverage} when the true label lies outside the prediction set. For instance, in financial risk assessment, if the true label is $y = \texttt{very high}$ while $\mathcal{C}_\alpha(x) = \{\texttt{low}, \texttt{moderate}\}$, the resulting miscoverage is substantial and could lead to catastrophic risk misestimation. Ideally, when coverage fails, the true label should lie as close as possible to the boundary of the prediction set, minimizing the impact of miscoverage. Under this criterion, a larger set such as $\mathcal{C}_\alpha(x) = \{\texttt{low}, \texttt{moderate}, \texttt{high}\}$ may be preferable to a smaller set $\mathcal{C}_\alpha(x) = \{\texttt{low}, \texttt{moderate}\}$, even though both fail to cover the true label and the larger set is less efficient in the classical CP sense. This observation motivates uncertainty quantification methods for OC that account not only for coverage and efficiency, but also for the ordinal distance incurred under miscoverage.

These challenges motivate a model-agnostic conformal approach to OC that can be combined with arbitrary loss functions, accommodates both grouped and assessed ordered targets, and produces meaningful contiguous prediction sets that accurately quantify uncertainty. Moreover, such a method should leverage the entire predictive probability distribution in an unbiased and computationally efficient manner when constructing prediction sets. These requirements are not met by existing approaches (cf.\ Section \ref{sec:related_work}).


In this paper, we advocate the ranked probability score (RPS)~\cite{epstein1969scoring} as a nonconformity measure for conformal prediction in OC. RPS is a proper scoring rule~\cite{gneiting2007strictly} for ordinal outcomes that incentivizes truthful probability estimation and explicitly accounts for the linear structure of the label space. Despite being well-established in the forecasting literature~\cite{murphy1970ranked,Murphy1971ANO}, it has only recently been recognized as a theoretically grounded metric for evaluating probabilistic ordinal classifiers in machine learning~\cite{DBLP:conf/miccai/Galdran23}. To the best of our knowledge, RPS has not yet been proposed as a nonconformity measure for CP in OC.
The main contributions of this paper are as follows:
\begin{itemize}
\item We propose RPS as a proper, model-agnostic nonconformity measure for CP in OC.
\item We provide theoretical guarantees for desirable properties in OC: RPS-based conformal prediction sets (i) \textbf{satisfy marginal coverage}, (ii) \textbf{are nested} with respect to the miscoverage level ($\alpha$), and (iii) \textbf{are contiguous}. 
\item Furthermore, we show that RPS-based prediction sets directly \textbf{optimize ordinal risk under oracle conditional coverage}, measured as set-based $l_1$ error, in contrast to mode-centered approaches which primarily target set efficiency.
\item We show that RPS-based conformal prediction is computationally efficient, scaling linearly with both the number of labels and the number of calibration points.
\item Finally, we empirically validate our approach on ordinal image and tabular datasets, showing that median-centered RPS-based prediction sets strike a favorable balance between interval width and ordinal miscoverage magnitude.
\end{itemize}

\section{Related Work}
\label{sec:related_work}
\paragraph*{Ordinal Classification}
addresses the problem of predicting discrete ordered labels as commonly encountered in many high-stakes domains, including medicine~\cite{DBLP:journals/artmed/Dorado-MorenoPG17,prodeau2019ordinal,DBLP:journals/mta/TariqSN25} and finance~\cite{DBLP:journals/sma/HirkHV19}. Unlike multinomial classification, where class labels are unordered, OC must account for the inherent order among classes, implying that misclassification costs typically increase with an increasing gap between predicted label $\hat{y}$ and true label $y$. At the same time, OC differs from regression in that the labels are discrete rather than continuous, and the underlying measurement scale is ordinal rather than cardinal. Thus, strictly speaking, there is no natural notion of distance. In spite of this, encoding the class labels by integers $1, \ldots , K$ and using distance-based losses such as $|\hat{y} - y|$ is common practice. 

Recent work in OC has largely focused on improving predictive performance, often by minimizing distance-based losses such as mean absolute error~\cite{DBLP:conf/ai/GaudetteJ09} or quadratic weighted kappa~\cite{cohen1968weighted}. Existing approaches can be broadly categorized into (i) \emph{unimodal soft-labeling methods}~\cite{DBLP:conf/cvpr/DiazM19,DBLP:journals/ijon/LiuFKDXLY20,DBLP:conf/pkdd/HaasH23,DBLP:journals/pr/VargasGH22,DBLP:journals/isci/VargasDGGH23,DBLP:journals/inffus/VargasGBH23}, (ii) \emph{ordinal loss functions}~\cite{DBLP:journals/corr/HouYS16,DBLP:journals/prl/TorrePV18,DBLP:conf/coling/CastagnosMD22,albuquerque2022quasi,DBLP:conf/aaai/NachmaniGSSG25,DBLP:journals/eswa/PolatCT25}, and (iii) \emph{explicit unimodality constraints}~\cite{DBLP:journals/nn/CostaAC08,DBLP:conf/icml/BeckhamP17,DBLP:conf/nips/DeyMK23,DBLP:journals/tai/CardosoCA25}.  

\paragraph*{Conformal Prediction} is a framework that can be applied on top of any base model to produce \emph{prediction sets} (or \emph{intervals} in regression) instead of point predictions~\cite{DBLP:journals/jmlr/ShaferV08,vovk2005algorithmic,DBLP:conf/icml/VovkGS99}. These sets are guaranteed to contain the true label with a user-specified \emph{marginal coverage probability}. 
Inductive conformal prediction~\cite{DBLP:conf/ecml/PapadopoulosPVG02,papadopoulos2008inductive}, also known as split conformal prediction, has become the standard approach in practice due to its computational efficiency. CP has been extensively studied for classification \cite{sadinle2019least, DBLP:conf/nips/RomanoSC20} and regression~\cite{DBLP:conf/nips/RomanoPC19}, where it provides finite-sample, distribution-free coverage guarantees.

More recently, conformal prediction for ordinal classification has attracted increasing attention. \citet{DBLP:conf/miccai/LuAP22} and \citet{zhang2025provably} construct contiguous prediction sets by expanding outward from the mode of the predictive distribution, performing a greedy search for a threshold that ensures marginal coverage while aiming to keep the sets as small as possible. \citet{DBLP:conf/uai/XuGW23} formulate ordinal conformal prediction within the conformal risk control framework~\cite{DBLP:journals/corr/abs-2208-02814}, pursuing essentially the same goals. A different approach is taken by~\citet{DBLP:conf/nips/DeyMK23}, who enforce unimodal predictive distributions, enabling the reuse of existing conformal methods such as least ambiguous set-valued classifiers (LAC)~\cite{sadinle2019least} and adaptive prediction sets (APS)~\cite{DBLP:conf/nips/RomanoSC20}, while guaranteeing contiguity. However, unimodality is a strong bias that is not always warranted and may negatively impact unbiased UQ in OC~\cite{DBLP:journals/ijar/HaasH25,DBLP:journals/corr/abs-2507-00733}.

In contrast to these approaches, we propose an efficient, model-agnostic conformal method that leverages the full predictive distribution through a principled proper scoring rule. This method guarantees median-centered, contiguous prediction sets without relying on greedy or search-based procedures, mode-centered constructions, or unimodality assumptions, while faithfully respecting the ordinal structure of the label space and minimizing ordinal risk under oracle conditional coverage.

\section{Method}

\subsection{Problem formulation}

Consider a dataset $\mathcal{D} = \{(X_i, Y_i)\}_{i=1}^n \subset \mathcal{X} \times \mathcal{Y}$,
drawn from an underlying distribution $\mathcal{P}$ over the joint input-output space
$\mathcal{X} \times \mathcal{Y}$.
We focus on the ordinal classification setting, where the output space
$\mathcal{Y} = \{y_{1}, \ldots, y_{K}\}$ consists of a finite set of class labels endowed with a natural (linear) order
$y_{1} \prec y_{2} \prec \cdots \prec y_{K}$.

Let $(X_{n+1}, Y_{n+1})$ denote a test instance such that the augmented sample
$\mathcal{D} \cup \{(X_{n+1}, Y_{n+1})\}$ is exchangeable.
Assuming that the test label $Y_{n+1}$ is unobserved, the goal of conformal prediction is to construct a prediction set
$\mathcal{C}_\alpha(X_{n+1}) \subseteq \mathcal{Y}$ satisfying a marginal coverage guarantee.

\begin{definition}[Marginal coverage]
\label{def:mc}
A conformal prediction procedure satisfies marginal coverage at level $1-\alpha$ if the prediction set $\mathcal{C}(X_{n+1})$ it outputs satisfies
\begin{equation}
\mathbb{P}\left( Y_{n+1} \in \mathcal{C}(X_{n+1}) \right) \geq 1 - \alpha ,
\end{equation}
where $\alpha \in (0,1)$ is a user-specified error rate and the probability is taken with respect to the joint
distribution $\mathcal{P}$ and any randomness in the construction of $\mathcal{C}$.
\end{definition}
Distribution-free CP methods cannot generally guarantee instance-wise conditional coverage \cite{DBLP:journals/ml/Vovk13}, a strictly stronger requirement than marginal coverage.
\begin{definition}[Conditional coverage]
\label{def:cc}
A conformal prediction procedure satisfies conditional coverage at level $1-\alpha$ if for all $X_{n+1}$, the set $\mathcal{C}(X_{n+1})$ it outputs satisfies
\begin{equation}
\mathbb{P}\left( Y_{n+1} \in \mathcal{C}(X_{n+1}) \mid X_{n+1} \right) \geq 1 - \alpha ,
\end{equation}
where $\alpha \in (0,1)$ is a user-specified error rate and the probability is taken with respect to
the conditional distribution induced by $\mathcal{P}$.
\end{definition}


Conformal prediction constructs prediction sets through a \emph{nonconformity score}, which quantifies how incompatible a candidate label is with a given input relative to a predictive model.
Formally, a nonconformity score is a function
$
s : \mathcal{X} \times \mathcal{Y} \rightarrow \mathbb{R},
$
where larger values indicate greater nonconformity between an input--label pair $(\vec{x}, y)$ and the model's predictive behavior.
In this work, we consider probabilistic predictors
$h : \mathcal{X} \rightarrow \mathbb{P}(\mathcal{Y})$,
which output a predictive probability vector
$\vec{p} = (p(y_1), \ldots, p(y_K)) = (p_1, \ldots, p_K) \in \mathbb{P}(\mathcal{Y})$,
where $p(y_k)$ denotes the predictive probability assigned to class $y_k$.
Nonconformity scores are then derived from this predictive distribution to quantify the incompatibility of candidate labels with the model's predictions.

In this work, we adopt the inductive (or split) conformal prediction framework~\cite{DBLP:conf/ecml/PapadopoulosPVG02,papadopoulos2008inductive}.
The dataset $\mathcal{D}$ is partitioned into a proper training set $\mathcal{D}_{\mathrm{train}}$, used to train a predictive model $h$, and a calibration set
$\mathcal{D}_{\mathrm{cal}} = \{(X_i, Y_i)\}_{i=1}^{n}$, used to compute nonconformity scores.
Given a test input $X_{n+1}$ and a candidate label $y_k \in \mathcal{Y}$, the conformal prediction set is defined as
\begin{equation}
\label{eq:quantile}
\mathcal{C}_\alpha(X_{n+1}) =
\left\{
y \in \mathcal{Y} \; : \;
s(X_{n+1}, y) \leq \hat q_{1-\alpha}
\right\},
\end{equation}
where $\hat q_{1-\alpha}$ denotes the empirical $(1-\alpha)$-quantile of the nonconformity scores computed on the calibration set.
This construction guarantees marginal coverage at level $1-\alpha$ under the exchangeability assumption on
$\mathcal{D}_{\mathrm{cal}} \cup \{(X_{n+1}, Y_{n+1})\}$~\cite{vovk2005algorithmic,DBLP:journals/jmlr/ShaferV08}.
Naturally, the choice of the nonconformity score plays a central role in conformal prediction, as it determines not only calibration but also how informative the resulting prediction sets are, particularly in structured output spaces such as ordinal classification.

\subsection{An Ordinal Oracle Method}
\label{subsec:oracle}
Assume the true data-generating distribution $\mathcal{P}$ is known, which would in principle allow construction of prediction sets satisfying conditional coverage (Definition~\ref{def:cc}).
In the ordinal classification setting, an additional common objective of conformal prediction is to construct the \emph{smallest contiguous} prediction set achieving conditional coverage, thereby balancing coverage with efficiency~\cite{DBLP:conf/miccai/LuAP22,DBLP:conf/uai/XuGW23,zhang2025provably}.

Formally, with contiguous intervals $[l,u] := \{ l, l+1, \ldots, u \}$ and interval probabilities $p_{l,u}^*(X)  :=  \sum_{j=l}^{u} p^*(y_j \mid X) $, the optimal contiguous prediction set for an instance $X$ can be written as
\begin{equation}\label{eq:oracle}
\mathcal{C}_\alpha^*(X) = \underset{[l,u] : 1 \le l \le u \le K}{\arg\min}
\bigl\{\, u - l  \;:\; p_{l,u}^*(X) \ge 1 - \alpha \bigr\}
\end{equation}
Minimizing the length of the prediction set balances efficiency with conditional coverage, producing the most efficient contiguous sets possible.
Contiguity respects the ordinal structure of $\mathcal{Y}$ and ensures that no gaps exist in the predicted set of labels.
This oracle construction is not available in practice, as the true conditional distribution $p^*(y \mid \vec{x})$ is unknown.
Notably, this oracle minimizes set length but does not explicitly account for ordinal risk, such as expected $l_1$
deviation from the true label.

\subsection{Conformal Ordinal Prediction}

To approximate the oracle construction \eqref{eq:oracle} in practice, existing approaches~
\cite{DBLP:conf/miccai/LuAP22,zhang2025provably}
aim to identify a threshold $\lambda$ through greedy search that satisfies the marginal coverage
guarantee (Definition~\ref{def:mc}) while producing efficient contiguous prediction sets.

\begin{equation}
\label{eq:mode}
\begin{aligned}
\mathcal{C}_{\lambda}(X)
&= \{ y_j \in \mathcal{Y} : l(X;\lambda) \le j \le u(X;\lambda) \}, \\[4pt]
\bigl(l(X;\lambda), u(X;\lambda)\bigr)
&= \underset{1 \le l \le u \le K}{\arg\min}
\Bigl\{\, u - l  \;:\;  p_{l,u}(X) \ge \lambda \Bigr\}.
\end{aligned}
\end{equation}

To ensure marginal coverage, the threshold $\lambda$ is selected using the
calibration set $\mathcal{D}_{\mathrm{cal}} = \{(X_i,Y_i)\}_{i=1}^n$ as the
smallest value satisfying
\begin{equation}
\begin{aligned}
\sum_{i=1}^n
\mathds{1}\!\left\{ Y_i \in \mathcal{C}_{\lambda}(X_i) \right\}
\;\ge\;
\left\lceil (1-\alpha)(n+1) \right\rceil .
\end{aligned}
\end{equation}

While the procedure guarantees marginal coverage, relying on a single global threshold $\lambda$ ignores the geometric structure of the ordinal predictive distribution and instead primarily controls set size, i.e., efficiency. As a consequence, the resulting prediction sets may be efficient in terms of cardinality, yet limited in faithfulness with respect to uncertainty quantification and ordinal risk.

Similarly, the approach of~\cite{DBLP:conf/nips/DeyMK23} constructs prediction sets by growing them outward from the mode of a unimodal predictive distribution.

\begin{definition}[Unimodality~\cite{keilson1971some}]
\label{def:umod}
A discrete probability distribution $\vec{p}$ is \emph{unimodal} if there exists at least one index $m$, called the mode, such that
$$
p_k \ge p_{k-1}, \quad \text{for all } k \le m,
$$
$$
p_{k+1} \le p_k, \quad \text{for all } k \ge m.
$$
\end{definition}

The assumption of unimodal predictive distributions is a common inductive bias in ordinal classification~\cite{DBLP:journals/nn/CostaAC08,DBLP:conf/icml/BeckhamP17}. 
Restricting the output distributions of a predictor to be unimodal allows one to obtain contiguous prediction sets using standard nominal nonconformity scores, such as LAC or APS~\cite{DBLP:conf/nips/DeyMK23}. 
Nonetheless, both methods are driven by probability magnitude rather than ordinal geometry, and thus may insufficiently reflect uncertainty in low-probability tails of unimodal predictive distributions. 



\subsection{The Ranked Probability Score (RPS)}

To address these limitations, we propose the use of the
\emph{Ranked Probability Score (RPS)}~\cite{epstein1969scoring}
as a nonconformity score for conformal prediction
in ordinal classification.
RPS is a \textit{proper scoring rule}~\cite{gneiting2007strictly, murphy1969ranked} defined as follows.

\begin{definition}[Proper scoring rule]
A scoring rule $s : \mathcal{Y} \times \mathbb{P}(\mathcal{Y}) \rightarrow \mathbb{R}$ is \emph{proper} if the expected score is minimized when the predicted distribution $\vec{p}$ equals the true distribution $\vec{p}^*$, that is,
$$
\mathbb{E}_{Y \sim \vec{p}^*} \big[ s(Y, \vec{p}^*) \big]
\;\le\;
\mathbb{E}_{Y \sim \vec{p}^*} \big[ s(Y, \vec{p}) \big]
\quad \forall\, \vec{p} \in \mathbb{P}(\mathcal{Y}).
$$
It is \emph{strictly proper} if equality holds if and only if $\vec{p} = \vec{p}^*$.
\end{definition}

The RPS can be viewed as the Brier score~\cite{glenn1950verification}
applied to cumulative predictive probabilities, making it sensitive to the ordinal distances between classes~\cite{DBLP:journals/mansci/JoseNW09}.
For $K=2$, it reduces exactly to the standard Brier score for binary outcomes.

Let
$
F_X(k) = \sum_{j=1}^{k} p(y_j \mid X)
$
denote the predicted cumulative distribution function (CDF), and let
$
\mathds{1}\{ Y \le y_k \}
$
be the corresponding cumulative indicator of the true label.
The RPS is then defined as
\begin{equation}
\label{eq:rps_cdf}
\mathrm{RPS}(X,Y)
= \frac{1}{K-1}
\sum_{k=1}^{K-1}
\left(
F_X(k) - \mathds{1}\{ Y \le y_k \}
\right)^2,
\end{equation}
where the factor $\frac{1}{K-1}$ normalizes the score to lie in the interval $[0,1]$, making it independent of the number of classes $K$.

Unlike previous ordinal conformal methods, RPS leverages the \emph{entire predictive distribution} rather than mode-centric summaries, yielding prediction sets that better adapt to input-dependent uncertainty.
By construction, it provides a natural measure of nonconformity for ordinal candidate labels with respect to the full predictive distribution: it attains its maximum when the predicted mass is concentrated at the opposite end of the ordinal scale from the true label, and its minimum when the predicted mass is concentrated on the true class. As a strictly proper scoring rule for ordinal outcomes~\cite{murphy1969ranked}, RPS is theoretically grounded and encourages calibrated probability forecasts.
Its properness has recently attracted renewed attention for evaluating probabilistic ordinal classifiers~\cite{DBLP:conf/miccai/Galdran23}, despite its long-standing popularity in forecasting~\cite{murphy1970ranked,Murphy1971ANO}.


\subsection{Validity of RPS for Ordinal CP}

In this section, we establish the validity of the RPS as a nonconformity measure within the conformal prediction framework.
Let $s_{\mathrm{RPS}}(X,Y) = \mathrm{RPS}(X,Y)$ denote the nonconformity score derived from the RPS.
\begin{algorithm}[t]
\caption{Prediction Sets via RPS}
\begin{algorithmic}[1]
\State \textbf{Input:} $\mathcal{D}_{\mathrm{cal}} = \{(X_i,Y_i)\}_{i=1}^n$, significance level $\alpha$, new instance $X_{n+1}$
\State \textbf{Output:} Prediction set $\mathcal{C}_\alpha(X_{n+1})$
\State $s_i \gets s_{\mathrm{RPS}}(X_i,Y_i), \quad i=1,\dots,n$
\State $\hat q_{1-\alpha} \gets \text{quantile}_{(1-\alpha)(1+\frac{1}{n})}\big(\{s_i\}_{i=1}^n\big)$
\State $\mathcal{C}_\alpha(X_{n+1}) \gets \{ y \in \mathcal{Y} : s_{\mathrm{RPS}}(X_{n+1}, y) \leq \hat q_{1-\alpha} \}$
\State \Return $\mathcal{C}_\alpha(X_{n+1})$
\end{algorithmic}
\end{algorithm}
A key requirement for using RPS in conformal prediction is that the resulting sets satisfy marginal coverage (Definition~\ref{def:mc}).

\begin{proposition}[Marginal coverage guarantee of RPS-based sets]
\label{prop:rps_cov}
Under the exchangeability assumption on $\mathcal{D}_{\mathrm{cal}} \cup \{(X_{n+1}, Y_{n+1})\}$, the RPS-based conformal procedure satisfies
$$
\mathbb{P}\bigl(Y_{n+1} \in \mathcal{C}(X_{n+1})\bigr) \ge 1-\alpha,
$$
regardless of the underlying predictive model.
\end{proposition}

\begin{proof}[Proof Sketch]
This result follows directly from the general theory of inductive conformal prediction~\cite{DBLP:journals/jmlr/ShaferV08,vovk2005algorithmic}.
Since $s_{\mathrm{RPS}}$ is a real-valued nonconformity score and the calibration scores are exchangeable with the test score, the standard quantile-based construction \eqref{eq:quantile} guarantees marginal coverage~\cite{vovk2005algorithmic,DBLP:journals/jmlr/ShaferV08}.
\end{proof}

Another desirable property for ordinal prediction sets is \emph{nesting} with respect to $\alpha$
\cite{DBLP:conf/miccai/LuAP22,DBLP:conf/uai/XuGW23,zhang2025provably},
e.g., $\mathcal{C}_{\alpha_2}(X)=\{\texttt{baby},\texttt{child}\}
\subseteq \mathcal{C}_{\alpha_1}(X)=\{\texttt{baby},\texttt{child},\texttt{teenager}\}$
for $0<\alpha_1\le\alpha_2<1$.
While RPS sets are always nested regardless of the predictive distribution, mode-based sets \eqref{eq:mode} are nested only when the predictive density is radially monotone, i.e., when it decreases monotonically from the mode along every direction~\cite{zhang2025provably}.

\begin{proposition}[Nestedness of RPS-based sets in $\alpha$]
\label{prop:nested}
For any input $X$ and any $0 < \alpha_1 \le \alpha_2 < 1$, the RPS-based conformal prediction sets satisfy
$
\mathcal{C}_{\alpha_2}(X) \subseteq \mathcal{C}_{\alpha_1}(X).
$
\end{proposition}

\begin{proof}[Proof]
If $\alpha_1 \le \alpha_2$, then $1-\alpha_1 \ge 1-\alpha_2$, which implies
$\hat q_{1-\alpha_1} \ge \hat q_{1-\alpha_2}$.
Hence, any label $y$ such that
$s_{\mathrm{RPS}}(X,y) \le \hat q_{1-\alpha_2}$
also satisfies
$s_{\mathrm{RPS}}(X,y) \le \hat q_{1-\alpha_1}$,
and therefore
$\mathcal{C}_{\alpha_2}(X) \subseteq \mathcal{C}_{\alpha_1}(X)$.
\end{proof}

Furthermore, ordinal prediction sets must be contiguous along the label ordering, consistent with the oracle definition in \eqref{eq:oracle}~\cite{DBLP:conf/uai/XuGW23,DBLP:conf/nips/DeyMK23}.

\begin{theorem}[Contiguity of RPS-based sets]
\label{theo:cont}
Let $\mathcal{Y} = \{y_1 \prec y_2 \prec \dots \prec y_K\}$ be a set of ordered labels, and let $s_{\mathrm{RPS}}$ denote the RPS-based nonconformity score.  
Then, for any input $X$ and any miscoverage level $\alpha \in (0,1)$, the conformal prediction set
$
\mathcal{C}^\mathrm{RPS}_\alpha(X) = \{ y \in \mathcal{Y} : s_{\mathrm{RPS}}(X,y) \le \hat q_{1-\alpha} \}
$
forms a contiguous interval of labels centered at the median $m$, i.e., there exist integers $1 \le l \le m \le u \le K$ such that
$
\mathcal{C}_\alpha(X) = \{y_l, \dots, y_m, \dots ,y_u\}.
$
\end{theorem}

Unlike existing nonconformity scores for nominal classification such as LAC~\cite{sadinle2019least}
or APS~\cite{DBLP:conf/nips/RomanoSC20}, RPS-based prediction sets are contiguous by design.
This contiguity follows directly from the ordinal structure and the monotonicity of the cumulative
distribution function, and does not require imposing unimodality assumptions on the predictive
probability mass function~\cite{DBLP:conf/nips/DeyMK23}, which may be unwarranted for ordinal targets.
A formal proof is provided in Appendix~\ref{sec:proof_cont}. Together with marginal validity and nesting in $\alpha$, these properties ensure that RPS-based conformal prediction produces prediction sets well suited for ordinal tasks.

Moreover, unlike \emph{mode}-centered approaches, RPS-based prediction sets are \emph{median}-centered (see Appendix \ref{sec:proof_cont}), thereby balancing the cumulative probability mass above and below the center. 
This property promotes robustness under skewed or heavy-tailed distributions; sublevel sets expand by minimizing the imbalance between the lower and upper predictive tails, rather than expanding outwards from a single high-probability mode. 
While this characteristic does not guarantee global risk optimality for the full conformal procedure, defined in terms of the set-based $l_1$ error \eqref{eq:ord_risk}, it does imply instance-level risk optimality at the level of the nonconformity scores. 
Furthermore, this instance-level optimality scales to full risk optimality within a conditional coverage oracle setting, where the true conditional distribution is known, causing pointwise risk minimization to aggregate directly into prediction sets that globally minimize the expected $l_1$ error.

\begin{theorem}[Ordinal risk optimality of RPS-based median-grown prediction sets under oracle conditional coverage]
\label{theo:ordrisk}
For a fixed input $X$, let $\mathcal{C}_\mathrm{RPS}(X) = \{y_l, \dots, y_m, \dots ,y_u\}$ denote the contiguous RPS-based prediction set, which is grown from a median index $m$ as in Theorem~\ref{theo:cont}.
Define the \emph{ordinal risk} of a set $\mathcal{C}(X)$ as the expected $l_1$-distance of true label from this set: 
\begin{equation}
\label{eq:ord_risk}
R(\mathcal{C}(X)) := \sum_{y=1}^K p(y \mid X)\, \min_{c \in \mathcal{C}(X)} |y - c|.
\end{equation}
Let $\mathcal{C}(X)$ be any other \emph{contiguous} conformal set of minimal cardinality satisfying the same coverage constraint as $\mathcal{C}_\mathrm{RPS}(X)$. Then
\begin{equation}
R(\mathcal{C}_\mathrm{RPS}(X)) \le R(\mathcal{C}(X)).
\end{equation}
\end{theorem}
A formal proof is provided in Appendix~\ref{sec:proof_ord_risk}. RPS-based sets, directly target ordinal risk reduction under oracle conditional coverage by sequentially adding adjacent labels that minimize risk, starting from the singleton risk-minimizing set containing only the median. This contrasts with mode-centered procedures, which prioritize set-size efficiency and may overlook the ordinal structure of the labels (Section~\ref{subsec:oracle}).

As illustrative examples, consider the unimodal distribution 
$\vec{p}_{\mathrm{umod}} = (0.06, 0.24, 0.32, 0.20, 0.18)$ 
and the multimodal distribution 
$\vec{p}_{\mathrm{multimod}} = (0.09, 0.12, 0.40, 0.04, 0.35)$, 
both with median $m=3$, coinciding with the mode. 
Tables~\ref{tab:median_mode_step_comparison_umod} and~\ref{tab:median_mode_step_comparison_multi} compare step-wise set expansions based on RPS with greedy mode-based expansion (e.g., min-CPS). 
The greedy strategy expands from the mode by iteratively adding the class with the largest remaining probability, thereby maximizing local probability mass for a given set size. 
In contrast, the RPS-based expansion accounts for cumulative distance-weighted deviations and directly minimizes ordinal risk \eqref{eq:ord_risk}. 
While the differences are moderate in the unimodal case, they become substantial for the multimodal distribution, where heavy tail mass causes the greedy procedure to underestimate ordinal risk.

\begin{table}[h!]
\scriptsize
\centering
\begin{tabular}{|c|c|c|c|c|}
\hline
\textbf{\#} & \textbf{RPS-based} & \textbf{Risk} & \textbf{Mode-grown} & \textbf{Risk} \\
\hline
1 & \texttt{\{3\}} & 0.92 & \texttt{\{3\}} & 0.92 \\
2 & \texttt{\{3, 4\}} & \textbf{0.54} & \texttt{\{2, 3\}} & 0.62 \\
3 & \texttt{\{2, 3, 4\}} & 0.24 & \texttt{\{2, 3, 4\}} & 0.24 \\
4 & \texttt{\{2, 3, 4, 5\}} & 0.06 & \texttt{\{2, 3, 4, 5\}} & 0.06 \\
5 & \texttt{\{1, 2, 3, 4, 5\}} & 0.00 & \texttt{\{1, 2, 3, 4, 5\}} & 0.00 \\
\hline
\end{tabular}
\caption{Step-wise comparison of ordinal risk for median RPS-based versus greedy mode-based set expansions for an exemplary unimodal distribution $\vec{p}_\mathrm{umod}=(0.06, 0.24, 0.32, 0.20, 0.18)$. Lower ordinal risk \eqref{eq:ord_risk} is highlighted.}
\label{tab:median_mode_step_comparison_umod}
\end{table}

\begin{table}[h!]
\scriptsize
\centering
\begin{tabular}{|c|c|c|c|c|}
\hline
\textbf{\#} & \textbf{RPS-based} & \textbf{Risk} & \textbf{Mode-grown} & \textbf{Risk} \\
\hline
1 & \texttt{\{3\}} & 1.04 & \texttt{\{3\}} & 1.04 \\
2 & \texttt{\{3, 4\}} & \textbf{0.65} & \texttt{\{2, 3\}} & 0.83 \\
3 & \texttt{\{3, 4, 5\}} & \textbf{0.30} & \texttt{\{1, 2, 3\}} & 0.74 \\
4 & \texttt{\{2, 3, 4, 5\}} & \textbf{0.09} & \texttt{\{1, 2, 3, 4\}} & 0.35 \\
5 & \texttt{\{1, 2, 3, 4, 5\}} & 0.00 & \texttt{\{1, 2, 3, 4, 5\}} & 0.00 \\
\hline
\end{tabular}
\caption{Step-wise comparison of ordinal risk for median RPS-based versus greedy mode-based set expansions for a multimodal distribution $\vec{p}_\mathrm{multimod} = (0.09, 0.12, 0.40, 0.04, 0.35)$. Lower ordinal risk \eqref{eq:ord_risk} is highlighted.}
\label{tab:median_mode_step_comparison_multi}
\end{table}

\subsection{Computational complexity of RPS}

During conformal \textbf{calibration}, we only need to compute the RPS score 
$s_{\mathrm{RPS}}(X_i, Y_i)$ for the true label $Y_i$ of each calibration point $X_i$. 
This requires a single RPS evaluation per point, each taking 
$\mathcal{O}(K)$ time to compute the cumulative distribution $F_X(k)$ and the score. 
Hence, the total cost for the calibration dataset $\mathcal{D}_{\mathrm{cal}}$ is $\mathcal{O}(n K)$.

During conformal \textbf{inference}, computing a prediction set for a new input requires evaluating
$s_{\mathrm{RPS}}(X, y_\ell)$ for all $K$ candidate labels $y_\ell$. Naively, this requires
$\mathcal{O}(K^2)$ time. However, after computing $s_{\mathrm{RPS}}(X, y_1)$ once, we can exploit the exact recurrence
$$
s_{\mathrm{RPS}}(X, y_{\ell+1})
= s_{\mathrm{RPS}}(X, y_\ell) + \frac{2F(\ell) - 1}{K-1},
$$
to compute all $K$ scores in linear time $\mathcal{O}(K)$ (see Appendix~\ref{sec:proof_cont}). 
Specifically, the initial score $s_{\mathrm{RPS}}(X, y_1)$ is computed in $\mathcal{O}(K)$ time, and each subsequent score is obtained with a constant-time update, i.e., $\mathcal{O}(1)$ per label.

Overall, RPS-based conformal prediction is highly efficient: linear in the number of calibration points $n$ and the number of labels $K$. 
By contrast, prior methods~\cite{DBLP:conf/miccai/LuAP22,DBLP:conf/uai/XuGW23, zhang2025provably} 
typically require an additional multiplicative cost of $\mathcal{O}(\log(1/\epsilon))$ to identify an $\epsilon$-optimal contiguous probability mass threshold $\lambda$ in~\eqref{eq:mode} via binary search to ensure marginal coverage.

\section{Experiments}

To evaluate the quality of RPS-based prediction sets, we conduct experiments on several ordinal image and tabular datasets, comparing against established ordinal conformal baselines. All experiments use neural networks as the underlying model class. 
The source code of the following experiments is made publicly available \footnote{\url{https://github.com/stefanahaas41/rps-ordinal-conformal-prediction}}.

\subsection{Baseline Nonconformity Scores}

As nominal baseline nonconformity measures, we include the Least ambiguous set-valued classifier (LAC) score~\cite{sadinle2019least}, as well as the adaptive prediction set (APS) score~\cite{DBLP:conf/nips/RomanoSC20} (see Appendix~\ref{sec:experiments_details} for details).
Both LAC and APS treat class labels as unstructured, providing natural baselines for assessing the benefits of incorporating ordinal structure. As ordinal conformal prediction baselines, we consider conformal prediction sets for ordinal classification (COPOC)~\cite{DBLP:conf/nips/DeyMK23} combined with LAC (COPOCL) and APS (COPOCA), as well as the min-CPS approach~\cite{zhang2025provably}. The latter is a greedy search–based algorithm that produces small contiguous mode-centered prediction sets and improves upon OrdinalAPS~\cite{DBLP:conf/miccai/LuAP22} in both computational efficiency and empirical performance.
As a naive ordinal baseline method, we also include the ordinal CDF (OCDF)~\cite{DBLP:conf/miccai/LuAP22}, which constructs prediction intervals from cumulative probabilities and is therefore an interesting baseline for RPS.

\subsection{Performance Metrics}

To evaluate the performance of the different conformal methods, we compute the empirical coverage (COV), as well as the average prediction set size (PS).
Additionally, we include the mean interval width (MW) (see Appendix~\ref{sec:experiments_details} for details). Another important aspect is the contiguity of the produced prediction sets, which we measure following~\cite{DBLP:conf/nips/DeyMK23} via the contiguity violation (CV) metric, where $0$ indicates no contiguity violations on $\mathcal{D}_\mathrm{test}$ and $1$ corresponds to maximal contiguity violation.

Since a central objective in ordinal classification is to minimize error distances, we prioritize ordinal-specific metrics. 
First, with $M = \{i : Y_i \notin \mathcal{C}(X_i)\}$ and
$$
d(Y_i,\mathcal{C}(X_i)) =
\begin{cases}
l_i - Y_i & \text{if } Y_i < l_i \\
Y_i - u_i & \text{if } Y_i > u_i
\end{cases}
$$
for a contiguous interval $\mathcal{C}(X_i) = [l_i,u_i]$, the mean absolute miscoverage magnitude (MAMM) is defined as
$$
\mathrm{MAMM} 
:= \frac{1}{|M|} \sum_{i \in M} d(Y_i,\mathcal{C}(X_i)) \, .
$$
The worst-case absolute miscoverage magnitude (WAMM) is
$$
\mathrm{WAMM} := \max_{i \in M} d(Y_i,\mathcal{C}(X_i)).
$$

These metrics quantify ordinal risk under miscoverage, measuring how far the true label lies from the prediction set rather than merely whether it is excluded.

Another highly relevant metric, which evaluates the trade-off between interval efficiency and distance-based error, is the average interval score loss (AISL)~\cite{gneiting2007strictly}, recently applied in conformal prediction for regression~\cite{DBLP:conf/uai/CabezasSRI25}:
\begin{align*}
\mathrm{AISL}
:= \frac{1}{|\mathcal{D}_\mathrm{test}|} \sum_{i \in \mathcal{D}_\mathrm{test}} 
\Bigg[
&(u_i - l_i)
+ \frac{2}{\alpha} (l_i - Y_i) \, \mathds{1}\{Y_i < l_i\}\\
&+ \frac{2}{\alpha} (Y_i - u_i) \, \mathds{1}\{Y_i > u_i\} 
\Bigg].
\end{align*}

AISL simultaneously accounts for interval width and miscoverage magnitude: the first term $(u_i - l_i)$ penalizes wide intervals, encouraging efficiency, while the second and third terms penalize distance-based errors outside the interval. The penalties are scaled by $2/\alpha$, so stricter coverage requirements (smaller $\alpha$) amplify the cost of miscoverage. By combining these factors, AISL provides an interpretable metric capturing both compactness and ordinal risk, making it particularly suitable for evaluating conformal prediction methods in ordinal settings.



\subsection{Experiments on Ordinal Datasets}

\begin{figure*}[!htbp]
    \centering
    \includegraphics[width=\linewidth]{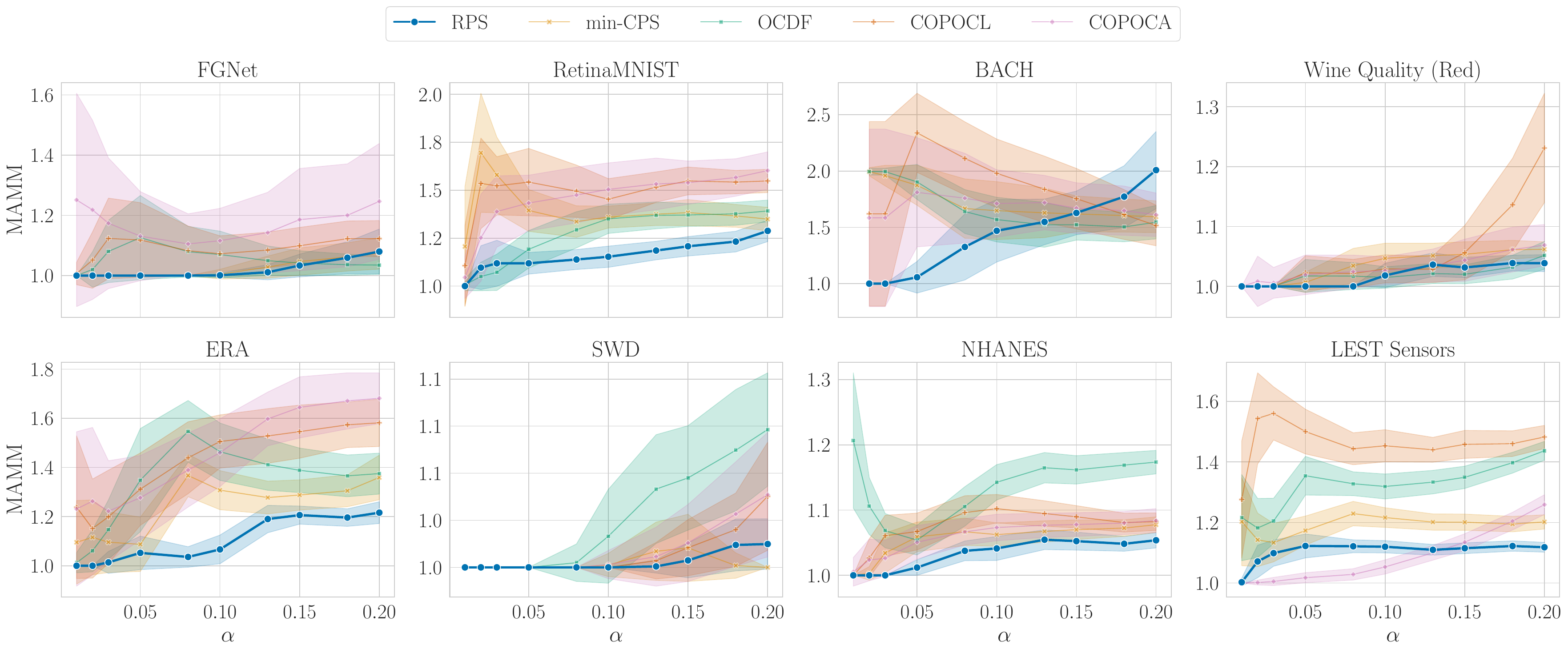}
    \caption{Comparison of prediction sets at $\alpha = \{0.01, 0.02, 0.03, 0.05, 0.08, 0.1, 0.13, 0.15, 0.18, 0.2\}$ across methods and datasets using the MAMM metric. Shaded regions indicate standard deviation over 50 trials.}
    \label{fig:mamm}
\end{figure*}

\begin{figure*}[!htbp]
    \centering
    \includegraphics[width=\linewidth]{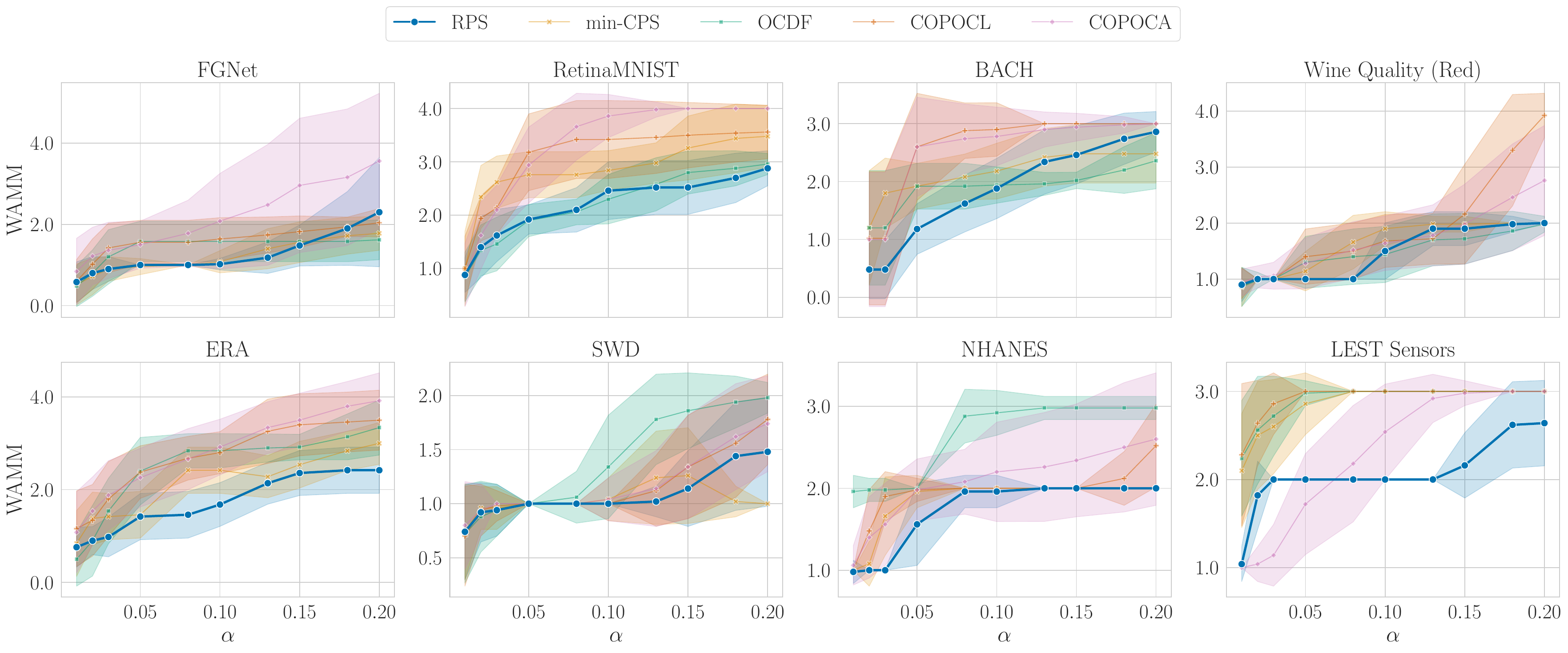}
    \caption{Comparison of prediction sets at $\alpha = \{0.01, 0.02, 0.03, 0.05, 0.08, 0.1, 0.13, 0.15, 0.18, 0.2\}$ across methods and datasets using the WAMM metric. Shaded regions indicate standard deviation over 50 trials.}
    \label{fig:wamm}
\end{figure*}
\begin{figure*}[!htbp]
    \centering
    \includegraphics[width=\linewidth]{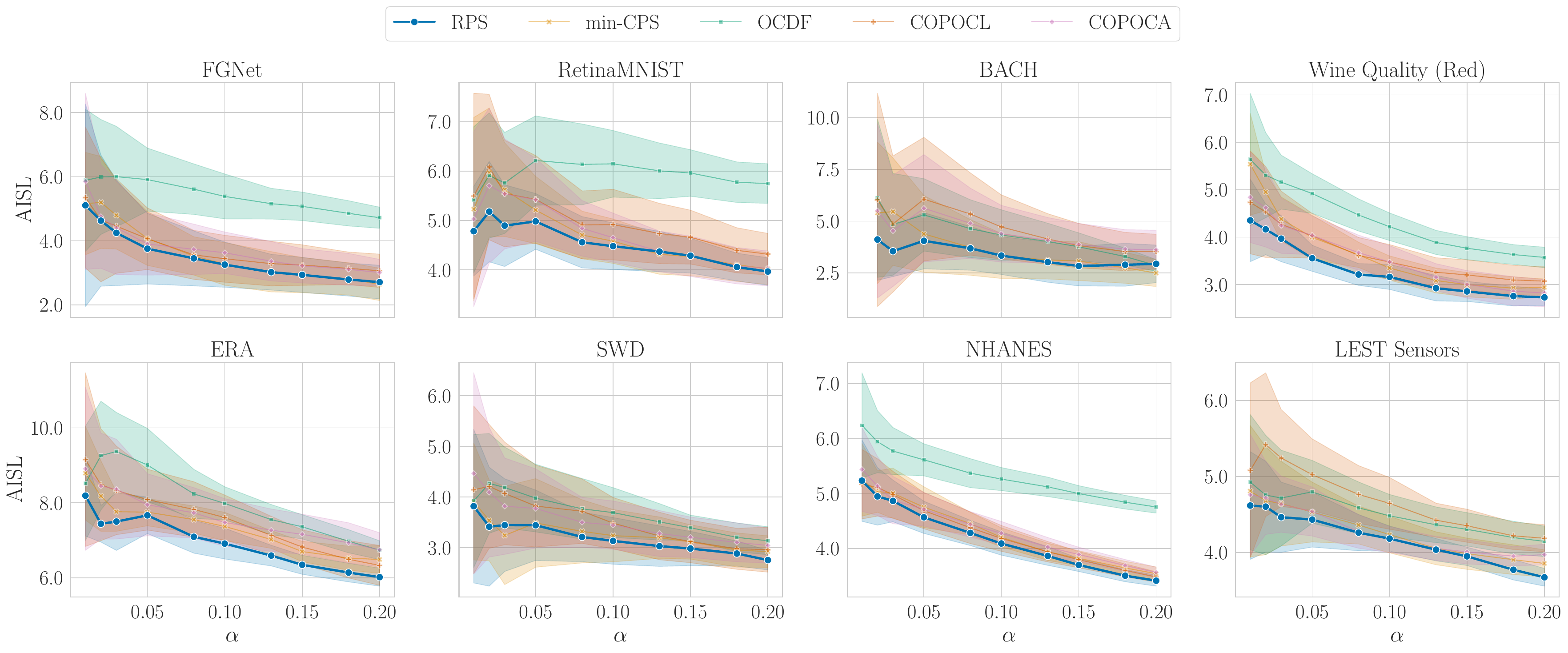}
    \caption{Comparison of prediction sets at $\alpha = \{0.01, 0.02, 0.03, 0.05, 0.08, 0.1, 0.13, 0.15, 0.18, 0.2\}$ across methods and datasets using the AISL metric. Shaded regions indicate standard deviation over 50 trials.}
    \label{fig:aisl}
\end{figure*}

We evaluate RPS-based prediction sets alongside several baseline methods on two medical image datasets (BACH~\cite{aresta2019bach} and RetinaMNIST~\cite{medmnistv2}) and an age-estimation dataset (FGNet~\cite{lanitis2002toward,panis2016overview}). Additionally, we include multiple ordinal tabular benchmark datasets~\cite{ayllon2025toc}.  

Due to space constraints, we focus on metrics that reflect ordinal miscoverage directly, consistent with our argument that evaluation in COP should account for the severity of missed predictions rather than treating all miscoverage events equally. Specifically, MAMM and WAMM quantify miscoverage magnitude, while AISL jointly captures prediction set size and ordinal miscoverage in a single score. 
LAC and APS violate the contiguity requirement for ordinal prediction sets and are therefore excluded from metrics that assume contiguous sets, such as MAMM, WAMM, and AISL.  
Additional experimental details and results over all datasets and metrics are provided in Appendix~\ref{sec:experiments_details} and~\ref{sec:experiments_ordinal_datasets}.

While RPS-based sets do not always achieve the highest prediction set efficiency, with min-CPS~\cite{zhang2025provably} showing particularly strong efficiency on this metric (see Appendix~\ref{sec:experiments_details} and~\ref{sec:experiments_ordinal_datasets}), they tend to achieve lower ordinal miscoverage, as reflected in MAMM (Figure~\ref{fig:mamm}) and WAMM (Figure~\ref{fig:wamm}). In contrast, mode-centered methods tend to underestimate ordinal risk relative to RPS-based sets. These results provide empirical support for our theoretical claim that RPS-based sets optimize for ordinal risk rather than set size (Theorem~\ref{theo:ordrisk}).

Furthermore, when considering the trade-off between efficiency and miscoverage magnitude via AISL (Figure~\ref{fig:aisl}), RPS-based sets achieve a favorable balance, demonstrating their practical advantage for real-world ordinal prediction tasks.

\section{Conclusion \& Discussion}

We have demonstrated that ranked probability score (RPS)-based conformal sets are, by construction, contiguous and median-centered, providing robust prediction sets for ordinal classification that minimize ordinal risk under oracle conditional
coverage, defined as the set-based $l_1$ error. These sets effectively balance efficiency and error reduction, a critical consideration in high-stakes applications.  
Specifically, RPS-based sets achieve a favorable trade-off between interval width and miscoverage magnitude while guaranteeing marginal coverage. They are highly competitive with existing ordinal conformal prediction methods, do not depend on specific model architectures, and can be applied to any underlying predictive model. Importantly, RPS-based sets produce meaningful contiguous intervals regardless of the data distribution, making them a practical and versatile solution for ordinal prediction tasks.

Building on this RPS-based framework for COP, a natural direction for future work is to incorporate a delineation of uncertainty into epistemic and aleatoric components~\cite{hullermeier2021aleatoric}. This line of research has recently attracted attention in both OC~\cite{DBLP:journals/corr/abs-2507-00733} and CP~\cite{sale2025aleatoric,javanmardi2025optimal}.

\begin{acknowledgements} 
Alireza Javanmardi gratefully acknowledges funding by the Klaus Tschira Stiftung (project 00.019.2024). 

\end{acknowledgements}


\bibliography{uai2026-template}

\newpage

\onecolumn



\appendix

\section{Proof of Theorem~\ref{theo:cont}}
\label{sec:proof_cont}
\begin{theorem*}[Contiguity of RPS-based sets]
Let $\mathcal{Y} = \{y_1 \prec y_2 \prec \dots \prec y_K\}$ be a set of ordered labels, and let $s_{\mathrm{RPS}}$ denote the RPS-based nonconformity score.  
Then, for any input $X$ and any miscoverage level $\alpha \in (0,1)$, the conformal prediction set
$
\mathcal{C}^\mathrm{RPS}_\alpha(X) = \{ y \in \mathcal{Y} : s_{\mathrm{RPS}}(X,y) \le \hat q_{1-\alpha} \}
$
forms a contiguous interval of labels centered at the median $m$, i.e., there exist integers $1 \le l \le m \le u \le K$ such that
$
\mathcal{C}_\alpha(X) = \{y_l, \dots, y_m, \dots y_u\}.
$
\end{theorem*}
\begin{proof}
Fix an input $X$. For a candidate label $y_\ell$, the RPS nonconformity score is
$$
s_{\mathrm{RPS}}(X, y_\ell) = \frac{1}{K-1} \sum_{k=1}^{K-1} \bigl(F_X(k) - \mathds{1}\{k \ge \ell\}\bigr)^2,
$$
where $k$ indexes the cumulative sums $F_X(k)$ and $\ell$ indexes the candidate labels.

\textbf{(1) Difference between consecutive labels.} Consider
$$
\Delta_\ell := s_{\mathrm{RPS}}(X, y_{\ell+1}) - s_{\mathrm{RPS}}(X, y_\ell).
$$
The step function $\mathds{1}\{k \ge \ell\}$ changes only at index $k = \ell$ when moving from $y_\ell$ to $y_{\ell+1}$; for all other $k$, the indicator remains the same. Therefore, the difference $\Delta_\ell$ between two adjacent labels ($y_\ell$ and $y_{\ell+1}$)  reduces to
$$
\Delta_\ell = \frac{1}{K-1}\left[(F(\ell) - 0)^2 - (F(\ell) - 1)^2\right] = \frac{2 F(\ell) - 1}{K-1} .
$$

\textbf{(2) Single minimum.}
Since the cumulative distribution $F(\ell)$ is non-decreasing in $\ell$ and satisfies $0 \le F(\ell) \le 1$,
the consecutive differences
$$
\Delta_\ell := s_{\mathrm{RPS}}(X,y_{\ell+1}) - s_{\mathrm{RPS}}(X,y_\ell)
= \frac{2 F(\ell) - 1}{K-1}
$$
form a non-decreasing sequence in $\ell$.
Hence, there exists an index
$$
m := \min\{\ell : F(\ell) \ge 1/2\},
$$
corresponding to a (discrete) median of the predictive distribution, such that
$$
\Delta_\ell \le 0 \quad \text{for } \ell < m,
\qquad
\Delta_\ell \ge 0 \quad \text{for } \ell \ge m.
$$

It follows that the sequence $\ell \mapsto s_{\mathrm{RPS}}(X,y_\ell)$, which satisfies the recurrence
$$
s_{\mathrm{RPS}}(X,y_{\ell+1}) = s_{\mathrm{RPS}}(X,y_\ell) + \frac{2 F(\ell) - 1}{K-1},
$$
is non-increasing for $\ell < m$ and non-decreasing for $\ell \ge m$.
Therefore, the RPS score attains a single minimum at the median index $m$.

Consequently, starting from the minimum at $m$, extending the candidate set to the right ($\ell \ge m$)
increases the RPS score by increments $(2F(\ell) - 1)/(K-1) \ge 0$,
while extending it to the left ($\ell < m$) increases the score by increments $(1 - 2F(\ell))/(K-1) \ge 0$.

Thus, the RPS score is V-shaped around the median of the predictive distribution, being non-increasing to the left of the median and non-decreasing to the right.


\textbf{(3) Contiguity of conformal sets.}
Since the mapping $\ell \mapsto s_{\mathrm{RPS}}(X,y_\ell)$ is V-shaped,
its sublevel set
$$
\mathcal{C}^\mathrm{RPS}_\alpha(X) = \{y_\ell : s_{\mathrm{RPS}}(X,y_\ell) \le \hat q_{1-\alpha}\}
$$
forms a contiguous interval along the ordinal axis, expanding from the median index $m$ toward the tails.
Hence, the RPS-based conformal prediction set is contiguous.

\end{proof}

\section{Proof of Theorem~\ref{theo:ordrisk}}
\label{sec:proof_ord_risk}

\begin{theorem*}[Ordinal risk optimality of RPS-based median-grown prediction sets under oracle conditional coverage]
For a fixed input $X$, let $\mathcal{C}_\mathrm{RPS}(X) = \{y_l, \dots, y_m, \dots ,y_u\}$ denote the contiguous RPS-based prediction set, which is grown from a median index $m$ as in Theorem~\ref{theo:cont}.
Define the \emph{ordinal risk} of a set $\mathcal{C}(X)$ as the expected $l_1$-distance of true label from this set: 
\begin{equation}
\label{eq:ord_risk_app}
R(\mathcal{C}(X)) := \sum_{y=1}^K p(y \mid X)\, \min_{c \in \mathcal{C}(X)} |y - c|.
\end{equation}
Let $\mathcal{C}(X)$ be any other \emph{contiguous} conformal set of minimal cardinality satisfying the same coverage constraint as $\mathcal{C}_\mathrm{RPS}(X)$. Then
\begin{equation}
R(\mathcal{C}_\mathrm{RPS}(X)) \le R(\mathcal{C}(X)).
\end{equation}
\end{theorem*}
\begin{proof}
\textbf{(1) Singleton case.}  
For a singleton set $\{c\}$, the ordinal risk is
$$
R(\{c\}) = \sum_{y=1}^K p(y\mid X)\,|y-c|.
$$
It is well known that this expected absolute deviation is minimized at a (discrete) median $m$ of $p(\cdot\mid X)$. Hence, starting with $\{m\}$, as done by RPS-based sets (see Theorem~\ref{theo:cont}), is optimal among all sets of size $1$. Any discrete median suffices.

\textbf{(2) Adding adjacent labels.}  
We initialize the contiguous set $\mathcal C$ at the singleton median of the predictive distribution, 
$$
m := \min\{\ell : F(\ell) \ge 1/2\}, 
\qquad \mathcal C = \{m\},
$$ 
and then expand it by adding adjacent labels toward the side with larger tail probability.

Let $\mathcal C = \{l, \dots, u\}$ be a contiguous set of labels, and let
$$
F(y) := \sum_{j \le y} p(y_j \mid X)
$$
be the cumulative probability function.

\paragraph{Ordinal risk reduction.}  
For any label $y \le l-1$, we have
$$
\min_{c \in \mathcal C} |y - c| = l - y, 
\qquad
\min_{c \in \mathcal C \cup \{l-1\}} |y - c| = (l-1) - y,
$$
so the distance decreases by $1$, while for $y \ge l$ the distance is unchanged. Hence, the reduction in ordinal risk from adding $l-1$ is
$$
R(\mathcal C) - R(\mathcal C \cup \{l-1\}) = \sum_{y \le l-1} p(y \mid X) = F(l-1).
$$

Similarly, if $u < K$, adding $u+1$ reduces the ordinal risk by
$$
R(\mathcal C) - R(\mathcal C \cup \{u+1\}) = \sum_{y \ge u+1} p(y \mid X) = 1 - F(u).
$$

Thus, among the two possible single-step contiguous extensions of $\mathcal C$, the one yielding the largest reduction in ordinal risk is toward the side with larger tail probability outside $\mathcal C$:
\begin{equation}
\label{eq:ord_risk_minimization}
\text{add left if } F(l-1) \ge 1 - F(u), \quad \text{otherwise add right}.
\end{equation}

\paragraph{Connection to RPS sublevel sets.}  

By Theorem~\ref{theo:cont}, extending the candidate set to the right $(u+1)$
increases the RPS score by increments $(2F(u) - 1)/(K-1) \ge 0$,
while extending it to the left $(l-1)$ increases the score by increments $(1 - 2F(l-1))/(K-1) \ge 0$.

The conformal prediction procedure selects the next label that yields the \emph{smallest increase in the RPS score} \eqref{eq:quantile}, so the greedy expansion chooses
$$
\text{add left if } \frac{1 - 2F(l-1)}{K-1} \le \frac{2F(u) - 1}{K-1}, \quad \text{otherwise add right},
$$
which is algebraically equivalent to \eqref{eq:ord_risk_minimization}:

\begin{align*}
1 - 2F(l-1) \le 2F(u) - 1
&\iff 2 \le 2F(u) + 2F(l-1) \\
&\iff 1 \le F(u) + F(l-1) \\
&\iff F(l-1) \ge 1 - F(u).
\end{align*}

\textbf{(3) Risk optimality among contiguous intervals of fixed length.}  
Fix a length $s \ge 1$ and consider all contiguous sets
$\mathcal C = \{l,\dots,l+s-1\}$ of size $s$. For such sets, the ordinal risk
$$
R(\mathcal C(X)) = \sum_{y=1}^K p(y\mid X) \min_{c\in\mathcal C} |y-c|
$$
is the expected $l_1$ distance from $Y$ to the set $\mathcal C$. For fixed $s$, this risk is minimized when the interval is (in a discrete sense) centered at a median index $m$ of $p(\cdot\mid X)$: moving the interval one step away from $m$ increases the expected absolute distance. Consequently, among all contiguous sets of size $s$, those whose center is closest to $m$ have minimal ordinal risk.

By Theorem~\ref{theo:cont}, for each cardinality $s$, the RPS sublevel set expands around $m$ in a way that keeps the tail imbalance $|F(\ell) - 1/2|$ as small as possible. Consequently, for each $s$, the RPS-based set is (discretely) centered at the median, which minimizes the ordinal risk among all contiguous sets of size $s$.

\textbf{(4) Risk dominance over minimal-cardinality coverage sets.}
Let $\mathcal C(X)$ be any contiguous minimal-cardinality set satisfying coverage
$\sum_{y\in \mathcal C} p(y\mid X) \ge 1-\alpha$, with size $s^\star$. Since
$\mathcal C_{\textrm{RPS}}(X)$ also satisfies coverage (Proposition~\ref{prop:rps_cov}),
$|\mathcal C_{\textrm{RPS}}(X)| \ge s^\star$, so the size-$s^\star$ sublevel set
$\mathcal C^{s^\star}_{\textrm{RPS}}(X) \subseteq \mathcal C_{\textrm{RPS}}(X)$ exists
(Proposition~\ref{prop:nested}). By (3) it minimizes ordinal risk among contiguous size-$s^\star$
sets, and by monotonicity of $R$ under inclusion,
$$
R(\mathcal{C}_\textrm{RPS}(X)) \le R(\mathcal C^{s^\star}_{\textrm{RPS}}(X)) \le R(\mathcal{C}(X)).
$$
\end{proof}

\section{Additional Details for Experiments on ordinal image datasets}
\label{sec:experiments_details}

This section provides additional details for the experiments on ordinal image datasets.

\paragraph{Baseline Nonconformity Scores.}
Least ambiguous set-valued classifier (LAC) score~\cite{sadinle2019least},
$$
s_{\mathrm{LAC}}(X,y) := 1 - p(y \mid X).
$$
Adaptive prediction set (APS) score~\cite{DBLP:conf/nips/RomanoSC20},
$$
s_{\mathrm{APS}}(X,y) := \sum_{y' :\, p(y' \mid X) \ge p(y \mid X)} p(y' \mid X).
$$

\paragraph{Performance Metrics.}
Empirical coverage (COV),
$$
\mathrm{COV} := \frac{1}{|\mathcal{D}_\mathrm{test}|} \sum_{i \in \mathcal{D}_\mathrm{test}} \mathds{1} \{Y_i \in \mathcal{C}(X_i)\}.
$$
Average prediction set size (PS),
$$
\mathrm{PS} :=  \frac{1}{|\mathcal{D}_\mathrm{test}|} \sum_{i \in \mathcal{D}_\mathrm{test}} |\mathcal{C}(X_i)|.
$$
Mean interval width (MW),
$$
\mathrm{MW} := \frac{1}{|\mathcal{D}_\mathrm{test}|} \sum_{i \in \mathcal{D}_\mathrm{test}} (u_i - l_i),
$$
where $\mathcal{C}(X_i) = [l_i,u_i]$ denotes the contiguous prediction interval for the $i$-th test point.

To measure ordinal error distance over the entire test set $\mathcal{D}_\mathrm{test}$, we also report the mean absolute interval error (MAIE),
$$
\mathrm{MAIE} := \frac{1}{|\mathcal{D}_\mathrm{test}|} 
\sum_{i \in \mathcal{D}_\mathrm{test}}
\begin{cases}
l_i - Y_i & \text{if } Y_i < l_i\\[1mm]
Y_i - u_i & \text{if } Y_i > u_i\\
0 & \text{otherwise}
\end{cases} \, .
$$

\paragraph{Model implementation.}
All models are implemented using \texttt{skorch}~\cite{skorch}, a \texttt{scikit-learn}~\cite{scikit-learn}-compatible wrapper for \texttt{PyTorch}~\cite{paszke2019pytorch}, and are conformalized with \texttt{MAPIE}~\cite{Cordier_Flexible_and_Systematic_2023}. 
The \texttt{dlordinal}~\cite{DBLP:journals/ijon/BerchezMorenoAYGHFG25} package is used to implement ordinal-specific methodologies, such as COPOC~\cite{DBLP:conf/nips/DeyMK23}.

For image datasets, we employ a computationally efficient \texttt{ResNet-18}~\cite{DBLP:conf/cvpr/HeZRS16} model pretrained on ImageNet~\cite{deng2009imagenet}, as our primary objective is not to maximize predictive performance but to evaluate the proposed methodology. Nonetheless, ResNet-18 is widely used as a backbone in image-based ordinal classification research, where it achieves reasonable performance~\cite{DBLP:journals/ijon/BerchezMorenoAYGHFG25}.

\paragraph{Loss functions.}
We train all models using the standard cross-entropy (CE) loss, which is a proper scoring rule and has been shown to yield unbiased predictive probability distributions, including in ordinal classification~\cite{DBLP:journals/ijar/HaasH25,DBLP:journals/corr/abs-2507-00733}.
$$
l_{\mathrm{CE}}(\vec{p}, y) = - \sum_{k=1}^K \mathds{1}\{y = y_k\} \, \log(p_k),
$$
where $\vec{p} = (p_1, \dots, p_K)$ is the predicted probability distribution over the $K$ classes, and $y$ is the true label.
In addition, we consider the non-parametric conformal prediction sets for ordinal classification (COPOC) approach proposed by~\cite{DBLP:conf/nips/DeyMK23}, which enforces unimodality in the predictive probability distribution. This also serves as an exemplary ordinal-specific loss, encouraging predictions to respect the natural ordering of the labels.

\paragraph{BACH Dataset}

The BACH (BreAst Cancer Histology) dataset~\cite{aresta2019bach} is a benchmark for breast cancer histopathological image classification, originally introduced as part of the ICIAR 2018 Grand Challenge on Breast Cancer Histology Images. It consists of hematoxylin and eosin (H\&E) stained microscopy images of breast tissue, annotated by expert pathologists. The dataset contains 400 high-resolution images (2048$\times$1536 pixels) with 100 images per class, representing four ordinal classes that follow the natural progression of breast cancer:
\begin{enumerate}
    \item \textbf{Normal} -- healthy breast tissue
    \item \textbf{Benign} -- non-cancerous abnormal tissue
    \item \textbf{In situ carcinoma} -- cancerous cells that have not invaded surrounding tissue
    \item \textbf{Invasive carcinoma} -- cancerous cells that have spread to surrounding tissue
\end{enumerate}

The ordinal structure of these classes reflects increasing disease severity, making BACH particularly well-suited for evaluating ordinal classification methods in medical imaging. The dataset is perfectly balanced with an imbalance ratio (IR) of 1.0.
For our experiments, we resize all images to 224$\times$224 pixels and apply standard normalization with mean and standard deviation of 0.5 across all channels. During training, we augment the data with random rotations (up to 10 degrees) and random horizontal flips to improve model generalization. We use a ResNet-18 model pretrained on ImageNet and fine-tune it for our task.
We use two thirds of the data to fine-tune a ResNet-18 model and split the remaining one third equally into calibration and test sets. This random split is repeated over 50 trials.
BACH is widely used in medical imaging research and provides an important testbed for uncertainty quantification methods, as reliable predictions with well-calibrated confidence are critical in clinical decision-making for cancer diagnosis.

\begin{figure*}[htb!]
 \centering
 \begin{minipage}[b]{0.20\textwidth}
        \centering
\includegraphics[width=\linewidth]{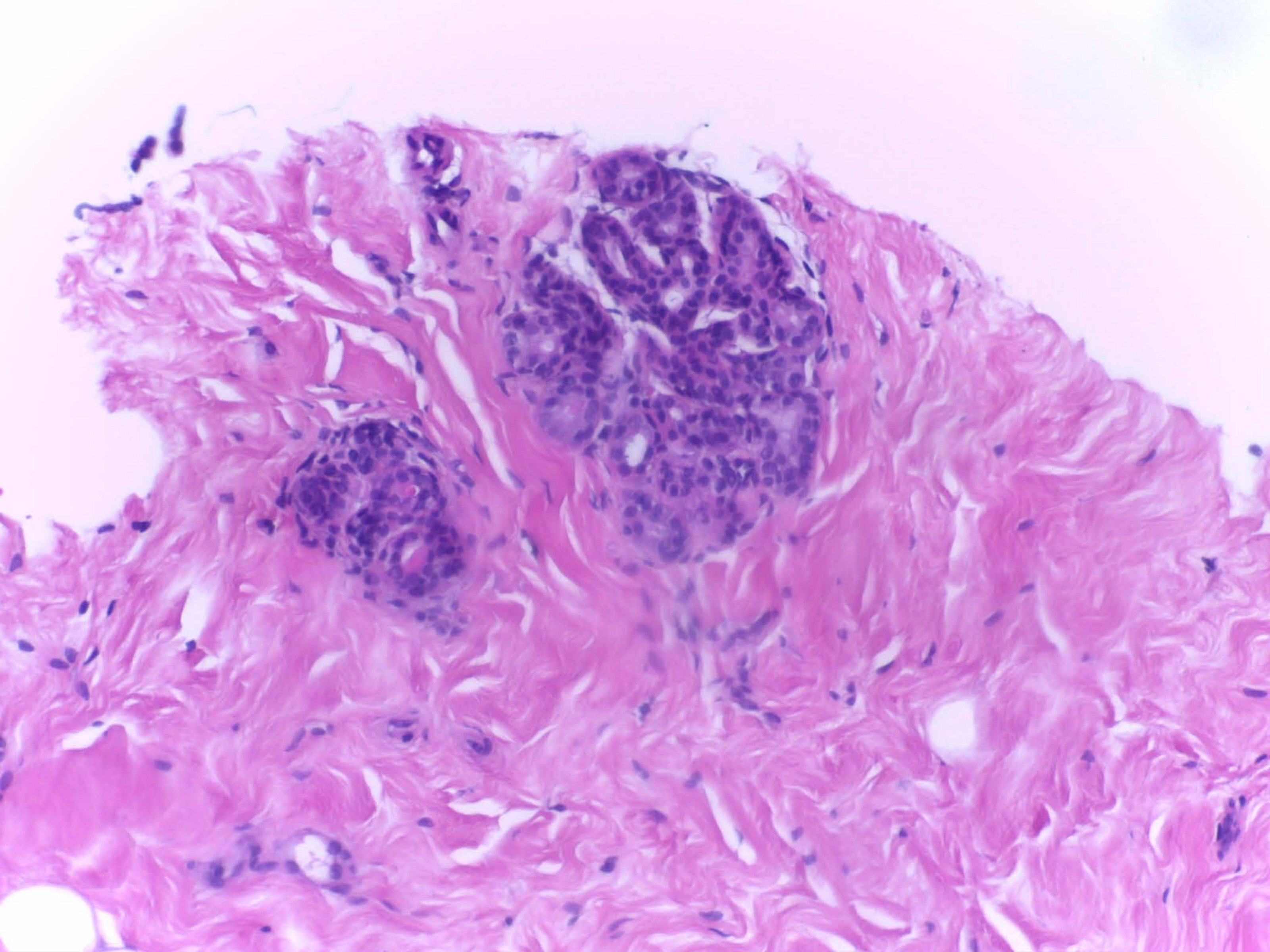}
 \end{minipage}
 \begin{minipage}[b]{0.20\textwidth}
        \centering
\includegraphics[width=\linewidth]{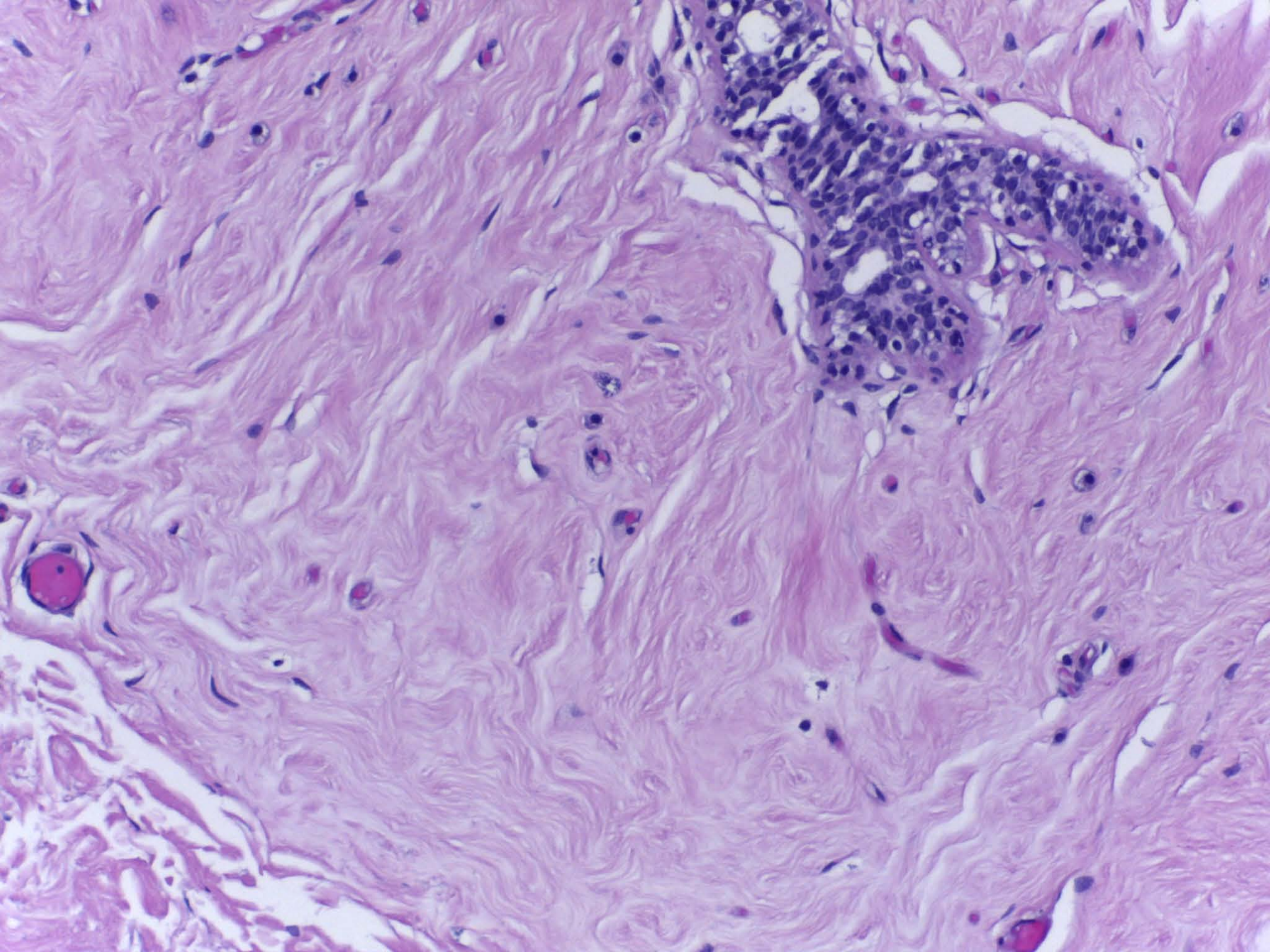}
 \end{minipage}
  \begin{minipage}[b]{0.20\textwidth}
        \centering
\includegraphics[width=\linewidth]{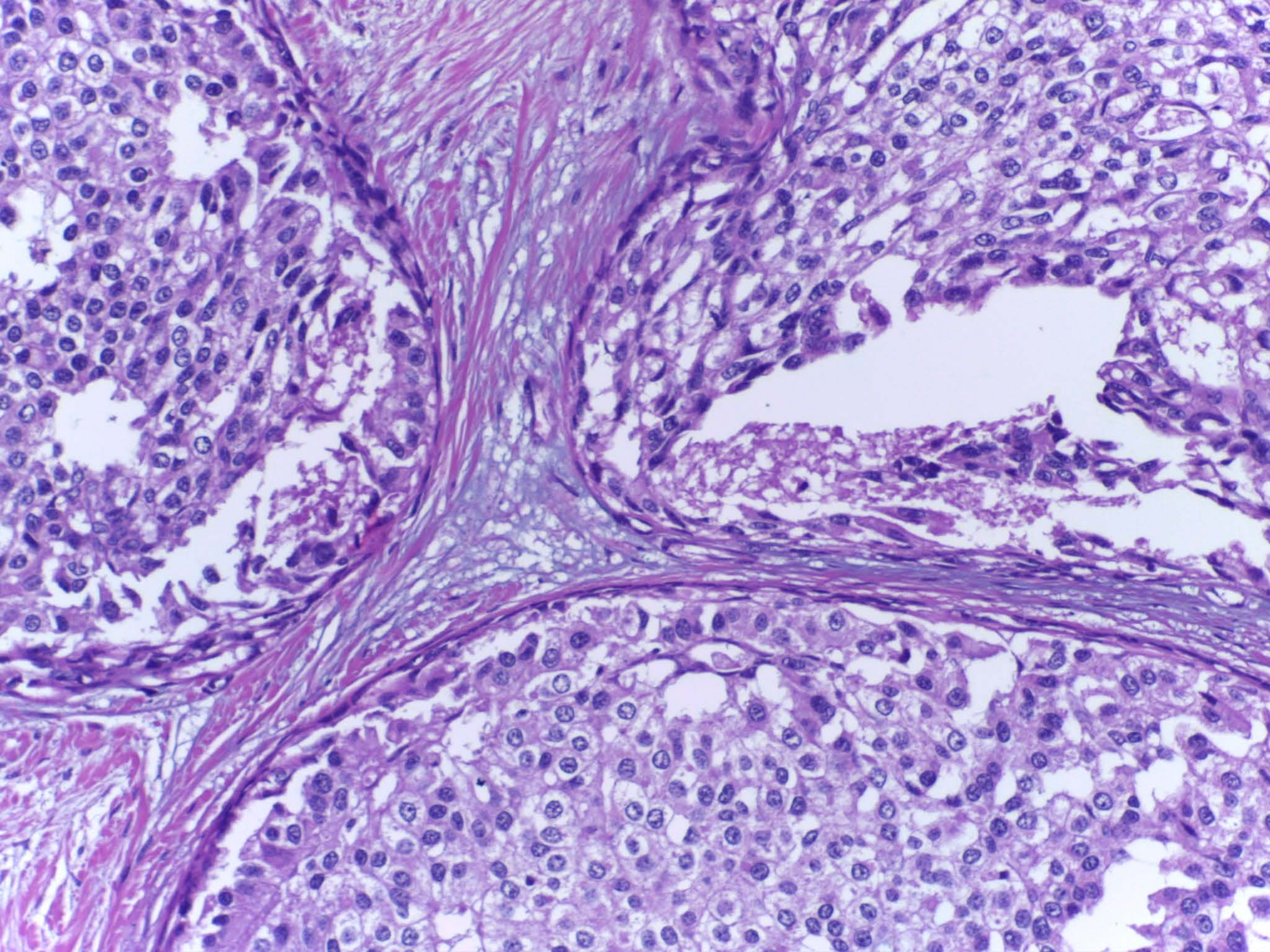}
 \end{minipage}
  \begin{minipage}[b]{0.20\textwidth}
        \centering
\includegraphics[width=\linewidth]{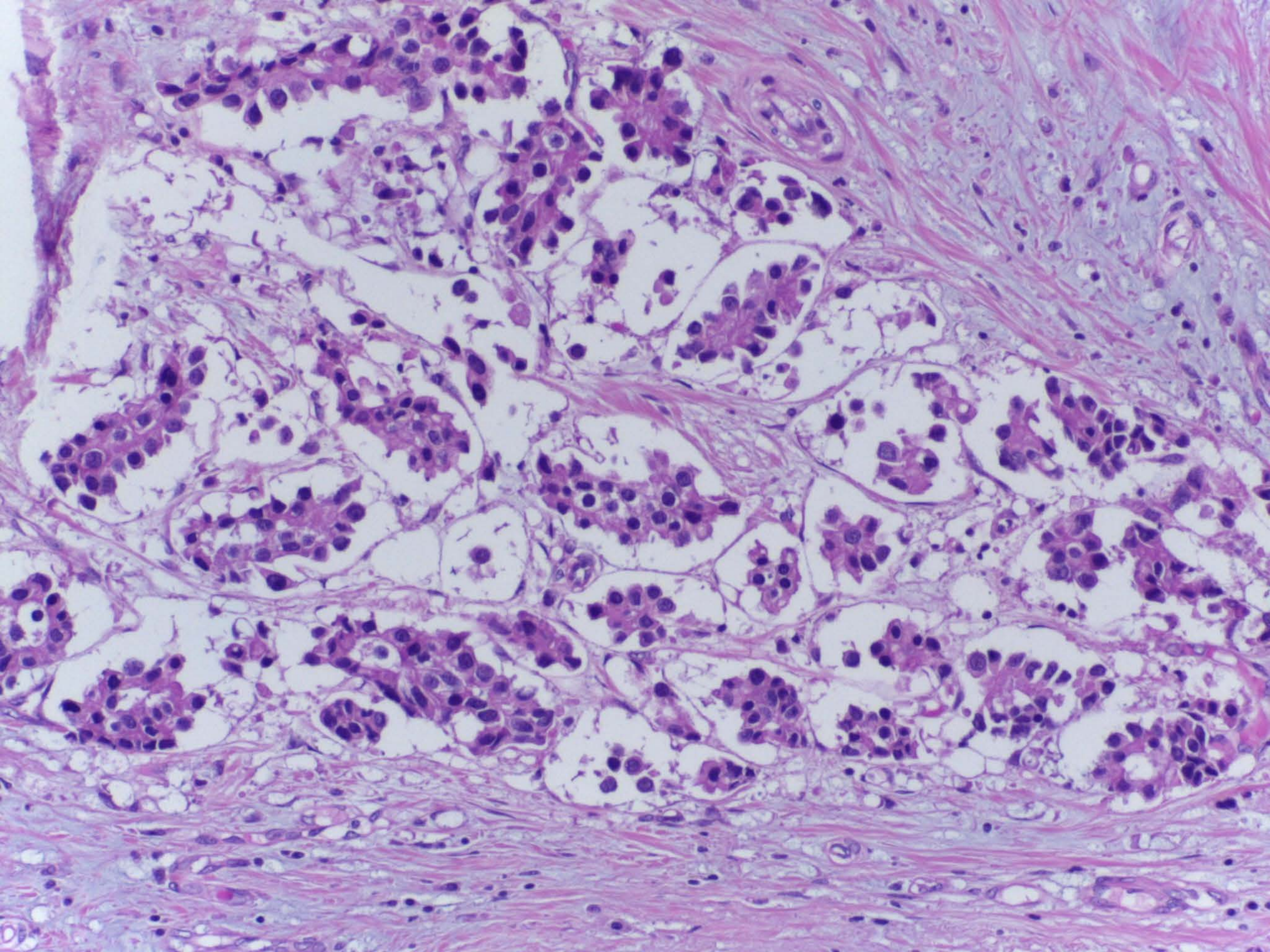}
 \end{minipage}
     \caption{Example images from the BACH dataset~\cite{aresta2019bach}.}
\end{figure*}

\paragraph{RetinaMNIST Dataset}
RetinaMNIST is a benchmark dataset of retinal fundus images from the MedMNIST~\cite{medmnistv2} collection, consisting of 1{,}600 $28\times 28$ grayscale images labeled with diabetic retinopathy severity levels. The labels form an ordered set of five classes (\emph{No DR}, \emph{Mild}, \emph{Moderate}, \emph{Severe}, \emph{Proliferative DR}) reflecting increasing disease severity, which makes RetinaMNIST a common choice in image-based ordinal classification research, e.g.,~\cite{DBLP:conf/nips/DeyMK23}. This ordered structure makes RetinaMNIST well-suited for evaluating ordinal classification and uncertainty estimation methods in medical image analysis.
In our experiments, we use a ResNet-18 backbone pretrained on ImageNet and adapted to the RetinaMNIST resolution. ResNet-18 provides a good trade-off between performance and computational efficiency for this task.
We use the dedicated training set to train the model, while the original validation and test sets are merged and then randomly split equally into calibration and test sets, again repeated over 50 trials.

\begin{figure*}[htb!]
 \centering
 \begin{minipage}[b]{0.10\textwidth}
        \centering
\includegraphics[width=\linewidth]{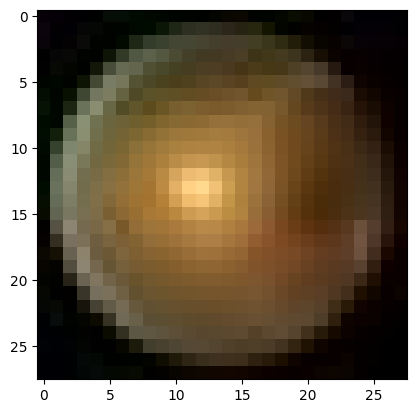}
 \end{minipage}
 \begin{minipage}[b]{0.10\textwidth}
        \centering
\includegraphics[width=\linewidth]{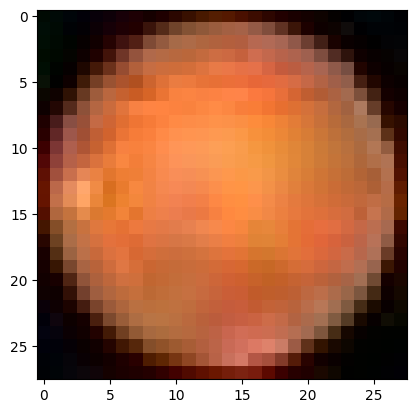}
 \end{minipage}
  \begin{minipage}[b]{0.10\textwidth}
        \centering
\includegraphics[width=\linewidth]{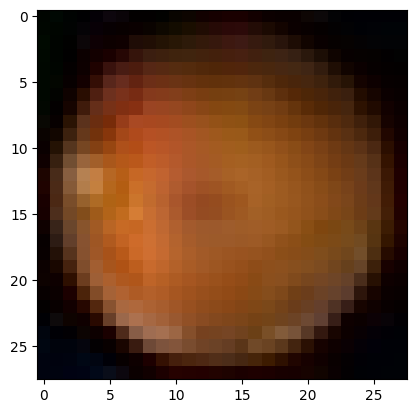}
 \end{minipage}
  \begin{minipage}[b]{0.10\textwidth}
        \centering
\includegraphics[width=\linewidth]{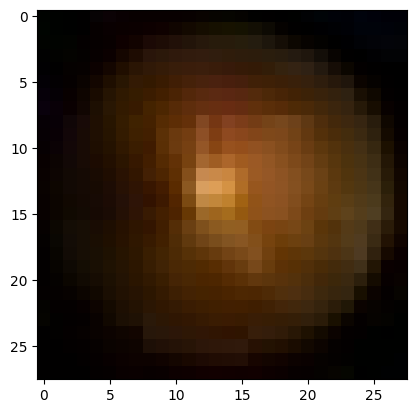}
 \end{minipage}
   \begin{minipage}[b]{0.10\textwidth}
        \centering
\includegraphics[width=\linewidth]{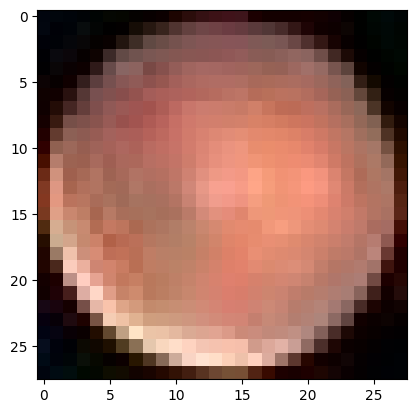}
 \end{minipage}
   \begin{minipage}[b]{0.10\textwidth}
        \centering
\includegraphics[width=\linewidth]{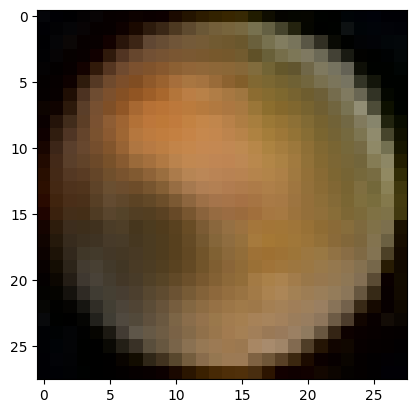}
 \end{minipage}
   \begin{minipage}[b]{0.10\textwidth}
        \centering
\includegraphics[width=\linewidth]{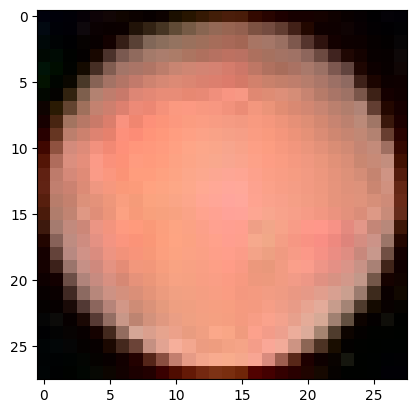}
 \end{minipage}
    \begin{minipage}[b]{0.10\textwidth}
        \centering
\includegraphics[width=\linewidth]{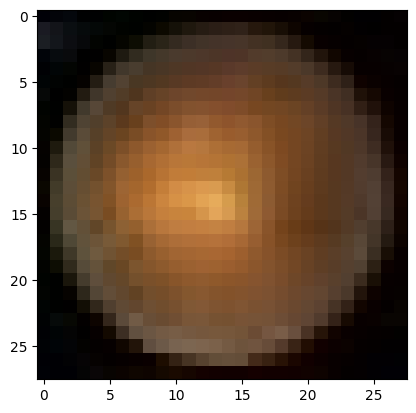}
 \end{minipage}
 \caption{Example images from the RetinaMNIST dataset~\cite{medmnistv2}}
\end{figure*}

\paragraph{FGNet Dataset}
FGNet~\cite{lanitis2002toward} is a widely used benchmark for age estimation from facial images, containing 1,002 images of 82 subjects spanning ages 0–69 years. For our experiments, we group ages into six ordinal classes to evaluate coarse age prediction performance, which aligns with the natural ordering of ages and allows ordinal classification evaluation.  
We preprocess images by resizing them to 256 pixels on the smaller side, followed by a random resized crop to $224\times224$ (scale 0.85–1.0). Data augmentation includes random horizontal flips (50\% probability), color jitter (brightness, contrast, saturation, and hue variations), and random rotations up to 10 degrees. Images are converted to tensors and then normalized using the ImageNet channel-wise mean [0.485, 0.456, 0.406] and standard deviation [0.229, 0.224, 0.225].
We use a ResNet-18 backbone pretrained on ImageNet and fine-tune it on FGNet for the six-class age classification task, leveraging the ordinal structure of the labels to evaluate our ordinal classification methods.
We use the dedicated training set to train the model, and split the test set into calibration and test sets, also repeated over 50 trials.

\begin{figure*}[htb!]
 \centering
 \begin{minipage}[b]{0.12\textwidth}
        \centering
\includegraphics[width=\linewidth]{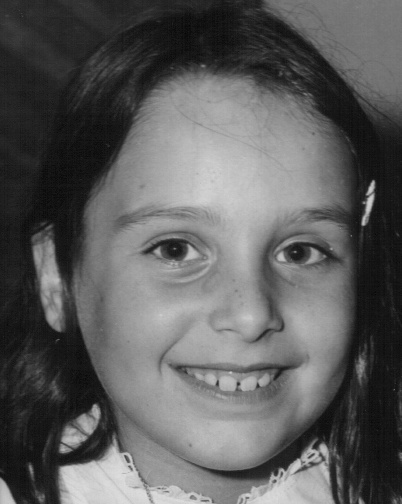}
 \end{minipage}
 \begin{minipage}[b]{0.12\textwidth}
        \centering
\includegraphics[width=\linewidth]{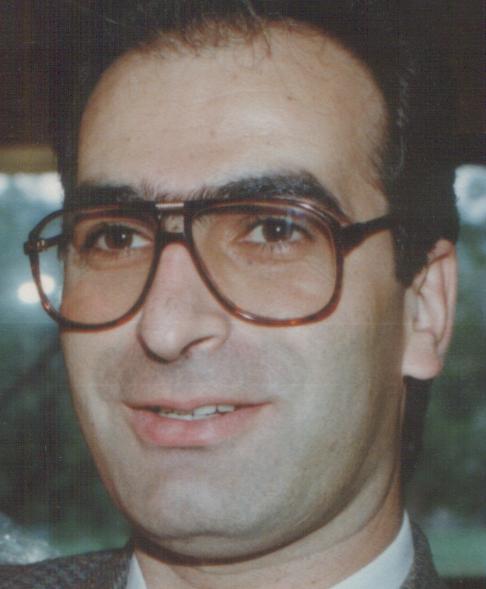}
 \end{minipage}
  \begin{minipage}[b]{0.12\textwidth}
        \centering
\includegraphics[width=\linewidth]{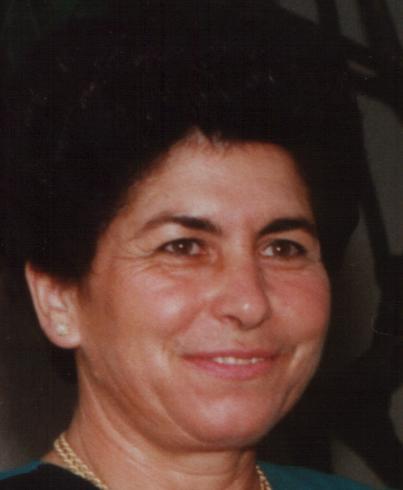}
 \end{minipage}
  \begin{minipage}[b]{0.12\textwidth}
        \centering
\includegraphics[width=\linewidth]{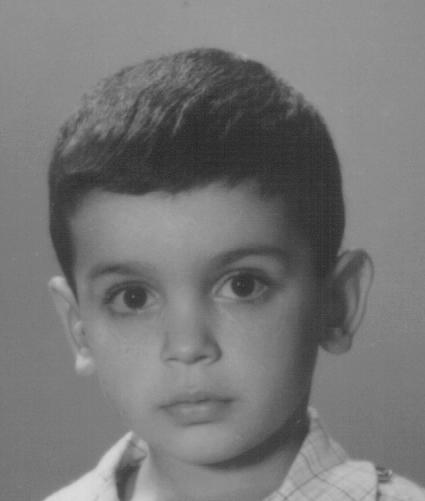}
 \end{minipage}
 \caption{Example images from the FGNet dataset~\cite{lanitis2002toward}}
\end{figure*}

\paragraph{Predictive Performance.}
Table~\ref{tab:results_ord_img_tab} depicts the predictive performance obtained using COPOC and cross-entropy (CE) loss across the different image datasets. Overall, CE achieves the best performance on FGNet and BACH, consistently improving ACC, MAE, MSE, and QWK. On RetinaMNIST, CE attains higher ACC and lower error metrics (MAE and MSE), while COPOC slightly outperforms in terms of 1-OFF accuracy and QWK. These results indicate that although CE generally provides stronger overall classification accuracy, COPOC remains competitive, particularly with respect to ordinal consistency metrics.
Furthermore, COPOC ensures unimodality across all datasets, as indicated by its degree of unimodality (UMOD) being equal to one.

\begin{table}[ht]
\centering
\small
\caption{Predictive performance on the image datasets. Results are reported as mean $\pm$ stddev. Metrics: ACC (Accuracy), 1-OFF (1-Off Accuracy), MAE (Mean Absolute Error), MSE (Mean Squared Error), QWK (Quadratic Weighted Kappa), UMOD (Degree of Unimodality). For ACC, 1-OFF, and QWK, higher is better; for MAE and MSE, lower is better. Bold values indicate the best performance for each metric within each dataset.}
\label{tab:results_ord_img_tab}
\begin{tabular}{@{}lccccccc@{}}
\toprule
\textbf{Dataset} & \textbf{Loss} & \textbf{ACC} ($\uparrow$) & \textbf{1-OFF} ($\uparrow$) & \textbf{MAE} ($\downarrow$) & \textbf{MSE} ($\downarrow$) & \textbf{QWK} ($\uparrow$) & \textbf{UMOD} \\
\midrule
\multirow{2}{*}{BACH} 
& COPOC & 0.662 $\pm$ 0.041 & 0.839 $\pm$ 0.030 & 0.549 $\pm$ 0.070 & 1.073 $\pm$ 0.166 & 0.549 $\pm$ 0.072 & 1.000 $\pm$ 0.000 \\
& CE 
& \textbf{0.839} $\pm$ 0.033 
& \textbf{0.912} $\pm$ 0.026 
& \textbf{0.256} $\pm$ 0.057 
& \textbf{0.462} $\pm$ 0.118 
& \textbf{0.806} $\pm$ 0.051 & 0.325 $\pm$ 0.031 \\
\midrule
\multirow{2}{*}{RetinaMNIST} & COPOC & 0.481 $\pm$ 0.018 & \textbf{0.759} $\pm$ 0.016 & 0.867 $\pm$ 0.040 & 1.803 $\pm$ 0.123 & \textbf{0.504} $\pm$ 0.033 & 1.000 $\pm$ 0.000\\
 & CE & \textbf{0.519} $\pm$ 0.019 & 0.757 $\pm$ 0.017 & \textbf{0.826} $\pm$ 0.038 & \textbf{1.741} $\pm$ 0.108 & 0.460 $\pm$ 0.032 & 0.208 $\pm$ 0.018 \\
\midrule
\multirow{2}{*}{FGNet}  & COPOC & 0.502 $\pm$ 0.032 & 0.945 $\pm$ 0.013 & 0.563 $\pm$ 0.040 & 0.717 $\pm$ 0.080 & 0.821 $\pm$ 0.021 & 1.000 $\pm$ 0.000  \\
 & CE & \textbf{0.579} $\pm$ 0.035 & \textbf{0.952} $\pm$ 0.013 & \textbf{0.474} $\pm$ 0.043 & \textbf{0.591} $\pm$ 0.072 & \textbf{0.860} $\pm$ 0.018 & 0.131 $\pm$ 0.023 \\
\bottomrule
\end{tabular}
\end{table}

\paragraph{Detailed Experimental Results.}

Results over 50 trials for $\alpha = \{0.01, 0.02, 0.03, 0.05, 0.08, 0.1, 0.13, 0.15, 0.18, 0.2\}$ on the BACH, RetinaMNIST, and FGNet datasets, covering the considered CP methods and metrics, are shown in Figures~\ref{fig:bach}, \ref{fig:retina}, and \ref{fig:fgnet}, respectively, and summarized in Table~\ref{tab:image_results} for $\alpha = \{0.02, 0.05, 0.1\}$.
In all cases, as expected, LAC and APS violate the contiguity of prediction sets (CV) and are therefore excluded from measures that require contiguous sets, i.e., MW, MAMM, WAMM, AISL, and MAIE.
All ordinal methods (min-CPS, COPOCL, COPOCA, OCDF, and RPS), however, output purely contiguous prediction sets.
Considering the trade-off between efficiency, as indicated by MW, and ordinal errors, as indicated by MAMM, WAMM, and MAIE, RPS-based sets strike a favorable balance, which is also reflected by the AISL metric that combines both aspects into a single score.
Compared to the other ordinal methods (min-CPS, COPOCL, COPOCA, and OCDF), RPS-based sets are reliable in the sense that risk is neither underestimated nor overestimated.
In contrast, min-CPS, COPOCL, and COPOCA tend to underestimate risk, as indicated by higher MAMM, WAMM, and MAIE values induced by overly small PS and MW, whereas OCDF tends to produce very large PS and MW values, particularly for the RetinaMNIST and FGNet datasets (see Figure~\ref{fig:retina}).
The claim that RPS-based sets strike a favorable balance between ordinal error awareness and efficient, small intervals is further supported by the tabular experiments reported in Appendix~\ref{sec:experiments_ordinal_datasets}.
These experiments also support our earlier claim that mode-centered set construction does not accurately capture uncertainty, whereas RPS-based sets faithfully account for the full ordinal structure and the associated risk.

\begin{table*}[!h]
\tiny
\centering
\caption{Comparison of the different non-conformity measures over the BACH, RetinaMNIST and FGNet datasets at $\alpha=0.02$, $\alpha=0.05$ and $\alpha=0.1$.}
\label{tab:image_results}
\begin{tabular}{lllcccccccc}
\toprule
Dataset & $\alpha$ & Method & {COV} & {PS} ($\downarrow$) & {CV} ($\downarrow$) & {MW} ($\downarrow$) & {MAMM} ($\downarrow$) & {WAMM} ($\downarrow$) & {MAIE} ($\downarrow$) & {AISL} ($\downarrow$) \\
\midrule
 \multirow{21}{*}{BACH}& \multirow{7}{*}{0.02}
                      & APS     & 0.982 $ \pm $ 0.021 & 2.475 $ \pm $ 0.603 & 0.240 $ \pm $ 0.060 & - & - & - & - & - \\
                      &         & COPOCA  & 0.985 $ \pm $ 0.021 & 3.593 $ \pm $ 0.380 & 0.000 $ \pm $ 0.000 & 2.593 $ \pm $ 0.380 & 1.584 $ \pm $ 0.788 & 1.000 $ \pm $ 1.161 & 0.029 $ \pm $ 0.046 & 5.502 $ \pm $ 4.206 \\
                      &         & COPOCL  & 0.983 $ \pm $ 0.022 & 3.539 $ \pm $ 0.459 & 0.000 $ \pm $ 0.000 & 2.539 $ \pm $ 0.459 & 1.620 $ \pm $ 0.820 & 1.020 $ \pm $ 1.152 & 0.035 $ \pm $ 0.056 & 6.024 $ \pm $ 5.145 \\
                      &         & LAC     & 0.981 $ \pm $ 0.020 & \textbf{2.417} $ \pm $ 0.658 & 0.210 $ \pm $ 0.059 & - & - & - & - & - \\
                      &         & OCDF    & 0.981 $ \pm $ 0.021 & 3.319 $ \pm $ 0.320 & 0.000 $ \pm $ 0.000 & 2.319 $ \pm $ 0.320 & 1.994 $ \pm $ 0.030 & 1.200 $ \pm $ 0.990 & 0.038 $ \pm $ 0.041 & 6.107 $ \pm $ 3.818 \\
                      &         & RPS     & 0.983 $ \pm $ 0.022 & 3.419 $ \pm $ 0.286 & 0.000 $ \pm $ 0.000 & 2.419 $ \pm $ 0.286 & \textbf{1.000} $ \pm $ 0.000 & \textbf{0.480} $ \pm $ 0.505 & \textbf{0.017} $ \pm $ 0.022 & \textbf{4.116} $ \pm $ 1.971 \\
                      &         & min-CPS & 0.981 $ \pm $ 0.020 & 2.671 $ \pm $ 0.594 & 0.000 $ \pm $ 0.000 & \textbf{1.671} $ \pm $ 0.594 & 1.993 $ \pm $ 0.037 & 1.200 $ \pm $ 0.990 & 0.037 $ \pm $ 0.039 & 5.398 $ \pm $ 3.416 \\
\cmidrule(lr){2-11}
 & \multirow{7}{*}{0.05} & APS     & 0.949 $ \pm $ 0.030 & 1.773 $ \pm $ 0.188 & 0.235 $ \pm $ 0.052 & - & - & - & - & - \\
                      & $-$& COPOCA  & 0.943 $ \pm $ 0.041 & 2.612 $ \pm $ 0.310 & 0.000 $ \pm $ 0.000 & 1.612 $ \pm $ 0.310 & 1.811 $ \pm $ 0.485 & 2.600 $ \pm $ 0.857 & 0.100 $ \pm $ 0.071 & 5.624 $ \pm $ 2.586 \\
                      & $-$& COPOCL  & 0.949 $ \pm $ 0.038 & 2.505 $ \pm $ 0.340 & 0.000 $ \pm $ 0.000 & 1.505 $ \pm $ 0.340 & 2.339 $ \pm $ 0.351 & 2.600 $ \pm $ 0.926 & 0.114 $ \pm $ 0.083 & 6.062 $ \pm $ 2.983 \\
                      & $-$& LAC     & 0.955 $ \pm $ 0.029 & \textbf{1.767} $ \pm $ 0.148 & 0.220 $ \pm $ 0.049 & - & - & - & - & - \\
                      & $-$& OCDF    & 0.954 $ \pm $ 0.029 & 2.910 $ \pm $ 0.164 & 0.000 $ \pm $ 0.000 & 1.910 $ \pm $ 0.164 & 1.903 $ \pm $ 0.155 & 1.920 $ \pm $ 0.396 & 0.085 $ \pm $ 0.047 & 5.304 $ \pm $ 1.748 \\
                      & $-$& RPS     & 0.951 $ \pm $ 0.029 & 2.860 $ \pm $ 0.320 & 0.000 $ \pm $ 0.000 & 1.860 $ \pm $ 0.320 & \textbf{1.057} $ \pm $ 0.140 & \textbf{1.180} $ \pm $ 0.438 & \textbf{0.055} $ \pm $ 0.041 & \textbf{4.054} $ \pm $ 1.356 \\
                      & $-$& min-CPS & 0.952 $ \pm $ 0.033 & 1.979 $ \pm $ 0.198 & 0.000 $ \pm $ 0.000 & \textbf{0.979} $ \pm $ 0.198 & 1.873 $ \pm $ 0.177 & 1.920 $ \pm $ 0.396 & 0.085 $ \pm $ 0.050 & 4.385 $ \pm $ 1.854 \\
\cmidrule(lr){2-11}
 & \multirow{7}{*}{0.1} & APS     & 0.913 $ \pm $ 0.043 & 1.468 $ \pm $ 0.205 & 0.165 $ \pm $ 0.056 & - & - & - & - & - \\
                      & $-$& COPOCA  & 0.905 $ \pm $ 0.043 & 2.158 $ \pm $ 0.230 & 0.000 $ \pm $ 0.000 & 1.158 $ \pm $ 0.230 & 1.712 $ \pm $ 0.299 & 2.780 $ \pm $ 0.507 & 0.162 $ \pm $ 0.078 & 4.394 $ \pm $ 1.369 \\
                      & $-$& COPOCL  & 0.900 $ \pm $ 0.056 & 1.927 $ \pm $ 0.205 & 0.000 $ \pm $ 0.000 & 0.927 $ \pm $ 0.205 & 1.980 $ \pm $ 0.304 & 2.900 $ \pm $ 0.463 & 0.190 $ \pm $ 0.086 & 4.721 $ \pm $ 1.550 \\
                      & $-$& LAC     & 0.916 $ \pm $ 0.040 & \textbf{1.451} $ \pm $ 0.223 & 0.153 $ \pm $ 0.059 & - & - & - & - & - \\
                      & $-$& OCDF    & 0.912 $ \pm $ 0.049 & 2.592 $ \pm $ 0.226 & 0.000 $ \pm $ 0.000 & 1.592 $ \pm $ 0.226 & 1.568 $ \pm $ 0.198 & 1.940 $ \pm $ 0.314 & 0.137 $ \pm $ 0.071 & 4.331 $ \pm $ 1.220 \\
                      & $-$& RPS     & 0.917 $ \pm $ 0.042 & 1.788 $ \pm $ 0.614 & 0.000 $ \pm $ 0.000 & 0.788 $ \pm $ 0.614 & \textbf{1.469} $ \pm $ 0.277 & \textbf{1.880} $ \pm $ 0.521 & \textbf{0.128} $ \pm $ 0.071 & 3.340 $ \pm $ 0.911 \\
                      & $-$& min-CPS & 0.917 $ \pm $ 0.040 & 1.613 $ \pm $ 0.293 & 0.000 $ \pm $ 0.000 & \textbf{0.613} $ \pm $ 0.293 & 1.648 $ \pm $ 0.257 & 2.180 $ \pm $ 0.482 & 0.135 $ \pm $ 0.067 & \textbf{3.322} $ \pm $ 1.088 \\
                      \midrule
\multirow{21}{*}{RetinMNIST} & \multirow{7}{*}{0.02}
                            & APS     & 0.984 $ \pm $ 0.010 & 4.080 $ \pm $ 0.156 & 0.057 $ \pm $ 0.020 & - & - & - & - & - \\
                            &         & COPOCA  & 0.982 $ \pm $ 0.011 & 4.266 $ \pm $ 0.190 & 0.000 $ \pm $ 0.000 & 3.266 $ \pm $ 0.190 & 1.252 $ \pm $ 0.227 & 1.620 $ \pm $ 0.635 & 0.024 $ \pm $ 0.017 & 5.705 $ \pm $ 1.570 \\
                            &         & COPOCL  & 0.981 $ \pm $ 0.011 & 4.222 $ \pm $ 0.156 & 0.000 $ \pm $ 0.000 & 3.222 $ \pm $ 0.156 & 1.535 $ \pm $ 0.236 & 1.940 $ \pm $ 0.424 & 0.029 $ \pm $ 0.016 & 6.083 $ \pm $ 1.473 \\
                            &         & LAC     & 0.982 $ \pm $ 0.010 & \textbf{4.025} $ \pm $ 0.086 & 0.046 $ \pm $ 0.013 & - & - & - & - & - \\
                            &         & OCDF    & 0.981 $ \pm $ 0.012 & 4.848 $ \pm $ 0.047 & 0.000 $ \pm $ 0.000 & 3.848 $ \pm $ 0.047 & \textbf{1.050} $ \pm $ 0.074 & \textbf{1.340} $ \pm $ 0.479 & 0.021 $ \pm $ 0.013 & 5.910 $ \pm $ 1.267 \\
                            &         & RPS     & 0.982 $ \pm $ 0.010 & 4.172 $ \pm $ 0.100 & 0.000 $ \pm $ 0.000 & 3.172 $ \pm $ 0.100 & 1.096 $ \pm $ 0.114 & 1.400 $ \pm $ 0.571 & \textbf{0.020} $ \pm $ 0.011 & \textbf{5.180} $ \pm $ 1.016 \\
                            &         & min-CPS & 0.983 $ \pm $ 0.008 & 4.032 $ \pm $ 0.090 & 0.000 $ \pm $ 0.000 & \textbf{3.032} $ \pm $ 0.090 & 1.695 $ \pm $ 0.310 & 2.340 $ \pm $ 0.593 & 0.029 $ \pm $ 0.014 & 5.971 $ \pm $ 1.306 \\
\cmidrule(lr){2-11}               
 & \multirow{7}{*}{0.05} & APS     & 0.951 $ \pm $ 0.014 & 3.600 $ \pm $ 0.146 & 0.121 $ \pm $ 0.025 & - & - & - & - & - \\
                            & $-$& COPOCA  & 0.950 $ \pm $ 0.016 & 3.620 $ \pm $ 0.149 & 0.000 $ \pm $ 0.000 & 2.620 $ \pm $ 0.149 & 1.433 $ \pm $ 0.145 & 2.940 $ \pm $ 0.712 & 0.072 $ \pm $ 0.026 & 5.420 $ \pm $ 0.840 \\
                            & $-$& COPOCL  & 0.953 $ \pm $ 0.016 & 3.527 $ \pm $ 0.201 & 0.000 $ \pm $ 0.000 & 2.527 $ \pm $ 0.201 & 1.542 $ \pm $ 0.176 & 3.180 $ \pm $ 0.720 & 0.073 $ \pm $ 0.027 & 5.428 $ \pm $ 0.887 \\
                            & $-$& LAC     & 0.952 $ \pm $ 0.015 & 3.639 $ \pm $ 0.216 & 0.095 $ \pm $ 0.030 & - & - & - & - & - \\
                            & $-$& OCDF    & 0.947 $ \pm $ 0.019 & 4.660 $ \pm $ 0.075 & 0.000 $ \pm $ 0.000 & 3.660 $ \pm $ 0.075 & 1.192 $ \pm $ 0.098 & \textbf{1.900} $ \pm $ 0.303 & 0.064 $ \pm $ 0.024 & 6.214 $ \pm $ 0.903 \\
                            & $-$& RPS     & 0.951 $ \pm $ 0.015 & 3.804 $ \pm $ 0.137 & 0.000 $ \pm $ 0.000 & 2.804 $ \pm $ 0.137 & \textbf{1.119} $ \pm $ 0.056 & 1.920 $ \pm $ 0.274 & \textbf{0.054} $ \pm $ 0.017 & \textbf{4.983} $ \pm $ 0.564 \\
                            & $-$& min-CPS & 0.950 $ \pm $ 0.016 & \textbf{3.476} $ \pm $ 0.143 & 0.000 $ \pm $ 0.000 & \textbf{2.476} $ \pm $ 0.143 & 1.394 $ \pm $ 0.109 & 2.760 $ \pm $ 0.431 & 0.068 $ \pm $ 0.020 & 5.212 $ \pm $ 0.681 \\
\cmidrule(lr){2-11}
 & \multirow{7}{*}{0.1} & APS     & 0.901 $ \pm $ 0.024 & 2.958 $ \pm $ 0.144 & 0.249 $ \pm $ 0.034 & - & - & - & - & - \\
                            & $-$& COPOCA  & 0.901 $ \pm $ 0.025 & 3.135 $ \pm $ 0.137 & 0.000 $ \pm $ 0.000 & 2.141 $ \pm $ 0.135 & 1.504 $ \pm $ 0.138 & 3.860 $ \pm $ 0.405 & 0.149 $ \pm $ 0.040 & 4.653 $ \pm $ 0.497 \\
                            & $-$& COPOCL  & 0.899 $ \pm $ 0.027 & 2.990 $ \pm $ 0.134 & 0.000 $ \pm $ 0.000 & 1.990 $ \pm $ 0.134 & 1.453 $ \pm $ 0.107 & 3.420 $ \pm $ 0.731 & 0.146 $ \pm $ 0.041 & 4.917 $ \pm $ 0.715 \\
                            & $-$& LAC     & 0.902 $ \pm $ 0.018 & \textbf{2.783} $ \pm $ 0.127 & 0.203 $ \pm $ 0.024 & - & - & - & - & - \\
                            & $-$& OCDF    & 0.899 $ \pm $ 0.026 & 4.413 $ \pm $ 0.100 & 0.000 $ \pm $ 0.000 & 3.413 $ \pm $ 0.100 & 1.351 $ \pm $ 0.076 & \textbf{2.300} $ \pm $ 0.463 & 0.137 $ \pm $ 0.038 & 6.147 $ \pm $ 0.674 \\
                            & $-$& RPS     & 0.898 $ \pm $ 0.025 & 3.128 $ \pm $ 0.097 & 0.000 $ \pm $ 0.000 & 2.128 $ \pm $ 0.097 & \textbf{1.153} $ \pm $ 0.055 & 2.460 $ \pm $ 0.542 & \textbf{0.118} $ \pm $ 0.028 & \textbf{4.480} $ \pm $ 0.472 \\
                            & $-$& min-CPS & 0.900 $ \pm $ 0.020 & 2.849 $ \pm $ 0.119 & 0.000 $ \pm $ 0.000 & \textbf{1.849} $ \pm $ 0.119 & 1.360 $ \pm $ 0.057 & 2.840 $ \pm $ 0.370 & 0.137 $ \pm $ 0.028 & 4.582 $ \pm $ 0.449 \\
\midrule
\multirow{21}{*}{FGNet} & \multirow{7}{*}{0.02}
                       & APS     & 0.984 $ \pm $ 0.019 & 5.393 $ \pm $ 0.450 & 0.315 $ \pm $ 0.137 & - & - & - & - & - \\
                       &        & COPOCA  & 0.980 $ \pm $ 0.016 & 3.417 $ \pm $ 0.306 & 0.000 $ \pm $ 0.000 & 2.417 $ \pm $ 0.306 & 1.218 $ \pm $ 0.297 & 1.220 $ \pm $ 0.708 & 0.023 $ \pm $ 0.018 & 4.757 $ \pm $ 1.617 \\
                       &        & COPOCL  & 0.978 $ \pm $ 0.018 & 3.227 $ \pm $ 0.269 & 0.000 $ \pm $ 0.000 & 2.227 $ \pm $ 0.269 & 1.051 $ \pm $ 0.093 & 1.020 $ \pm $ 0.622 & 0.024 $ \pm $ 0.021 & 4.647 $ \pm $ 1.930 \\
                       &        & LAC     & 0.984 $ \pm $ 0.017 & 5.175 $ \pm $ 0.321 & 0.214 $ \pm $ 0.041 & - & - & - & - & - \\
                       &        & OCDF    & 0.981 $ \pm $ 0.017 & 5.050 $ \pm $ 0.143 & 0.000 $ \pm $ 0.000 & 4.050 $ \pm $ 0.143 & 1.020 $ \pm $ 0.059 & 0.800 $ \pm $ 0.571 & \textbf{0.019} $ \pm $ 0.019 & 5.990 $ \pm $ 1.792 \\
                       &        & RPS     & 0.975 $ \pm $ 0.023 & \textbf{3.148} $ \pm $ 0.274 & 0.000 $ \pm $ 0.000 & \textbf{2.148} $ \pm $ 0.274 & \textbf{1.000} $ \pm $ 0.000 & 0.800 $ \pm $ 0.404 & 0.025 $ \pm $ 0.023 & \textbf{4.628} $ \pm $ 2.046 \\
                       &        & min-CPS & 0.980 $ \pm $ 0.018 & 4.236 $ \pm $ 0.400 & 0.000 $ \pm $ 0.000 & 3.236 $ \pm $ 0.400 &  \textbf{1.000} $ \pm $ 0.000 & \textbf{0.720} $ \pm $ 0.454 & 0.020 $ \pm $ 0.018 & 5.196 $ \pm $ 1.434 \\
\cmidrule(lr){2-11}
 & \multirow{7}{*}{0.05} & APS     & 0.950 $ \pm $ 0.030 & 3.875 $ \pm $ 0.661 & 0.446 $ \pm $ 0.094 & - & - & - & - & - \\
                       & $-$& COPOCA  & 0.957 $ \pm $ 0.026 & 2.964 $ \pm $ 0.216 & 0.000 $ \pm $ 0.000 & 1.964 $ \pm $ 0.216 & 1.131 $ \pm $ 0.148 & 1.500 $ \pm $ 0.580 & 0.048 $ \pm $ 0.029 & 3.900 $ \pm $ 0.972 \\
                       & $-$& COPOCL  & 0.948 $ \pm $ 0.027 & 2.759 $ \pm $ 0.205 & 0.000 $ \pm $ 0.000 & 1.759 $ \pm $ 0.205 & 1.118 $ \pm $ 0.124 & 1.560 $ \pm $ 0.541 & 0.058 $ \pm $ 0.029 & 4.063 $ \pm $ 0.962 \\
                       & $-$& LAC     & 0.952 $ \pm $ 0.032 & 3.914 $ \pm $ 0.601 & 0.247 $ \pm $ 0.045 & - & - & - & - & - \\
                       & $-$& OCDF    & 0.951 $ \pm $ 0.026 & 4.732 $ \pm $ 0.186 & 0.000 $ \pm $ 0.000 & 3.732 $ \pm $ 0.186 & 1.126 $ \pm $ 0.140 & 1.580 $ \pm $ 0.499 & 0.054 $ \pm $ 0.029 & 5.908 $ \pm $ 0.985 \\
                       & $-$& RPS     & 0.946 $ \pm $ 0.032 & \textbf{2.584} $ \pm $ 0.198 & 0.000 $ \pm $ 0.000 & \textbf{1.584} $ \pm $ 0.198 & \textbf{1.000} $ \pm $ 0.000 & 1.000 $ \pm $ 0.000 & 0.054 $ \pm $ 0.032 & \textbf{3.752} $ \pm $ 1.102 \\
                       & $-$& min-CPS & 0.953 $ \pm $ 0.030 & 3.190 $ \pm $ 0.421 & 0.000 $ \pm $ 0.000 & 2.190 $ \pm $ 0.421 & \textbf{1.000} $ \pm $ 0.000 & \textbf{0.960} $ \pm $ 0.198 & \textbf{0.047} $ \pm $ 0.030 & 4.070 $ \pm $ 0.839 \\
\cmidrule(lr){2-11}
 & \multirow{7}{*}{0.1} & APS     & 0.902 $ \pm $ 0.040 & 2.562 $ \pm $ 0.343 & 0.230 $ \pm $ 0.066 & - & - & - & - & - \\
                       & $-$& COPOCA  & 0.898 $ \pm $ 0.037 & 2.428 $ \pm $ 0.182 & 0.000 $ \pm $ 0.000 & 1.429 $ \pm $ 0.181 & 1.116 $ \pm $ 0.107 & 2.080 $ \pm $ 1.175 & 0.115 $ \pm $ 0.047 & 3.625 $ \pm $ 0.647 \\
                       & $-$& COPOCL  & 0.900 $ \pm $ 0.040 & 2.294 $ \pm $ 0.147 & 0.000 $ \pm $ 0.000 & 1.294 $ \pm $ 0.147 & 1.072 $ \pm $ 0.060 & 1.640 $ \pm $ 0.525 & 0.107 $ \pm $ 0.043 & 3.438 $ \pm $ 0.733 \\
                       & $-$& LAC     & 0.908 $ \pm $ 0.035 & 2.558 $ \pm $ 0.332 & 0.170 $ \pm $ 0.039 & - & - & - & - & - \\
                       & $-$& OCDF    & 0.905 $ \pm $ 0.042 & 4.359 $ \pm $ 0.181 & 0.000 $ \pm $ 0.000 & 3.359 $ \pm $ 0.181 & 1.070 $ \pm $ 0.078 & 1.580 $ \pm $ 0.499 & 0.101 $ \pm $ 0.044 & 5.383 $ \pm $ 0.701 \\
                       & $-$& RPS     & 0.897 $ \pm $ 0.043 & \textbf{2.182} $ \pm $ 0.172 & 0.000 $ \pm $ 0.000 & \textbf{1.182} $ \pm $ 0.172 & \textbf{1.001} $ \pm $ 0.007 & \textbf{1.020} $ \pm $ 0.141 & 0.104 $ \pm $ 0.043 & \textbf{3.254} $ \pm $ 0.705 \\
                       & $-$& min-CPS & 0.902 $ \pm $ 0.044 & 2.303 $ \pm $ 0.293 & 0.000 $ \pm $ 0.000 & 1.303 $ \pm $ 0.293 & 1.005 $ \pm $ 0.016 & 1.080 $ \pm $ 0.274 & \textbf{0.099} $ \pm $ 0.046 & 3.275 $ \pm $ 0.665 \\
\bottomrule
\end{tabular}
\end{table*}

\begin{figure*}[!htb]
    \centering
    \includegraphics[width=\linewidth]{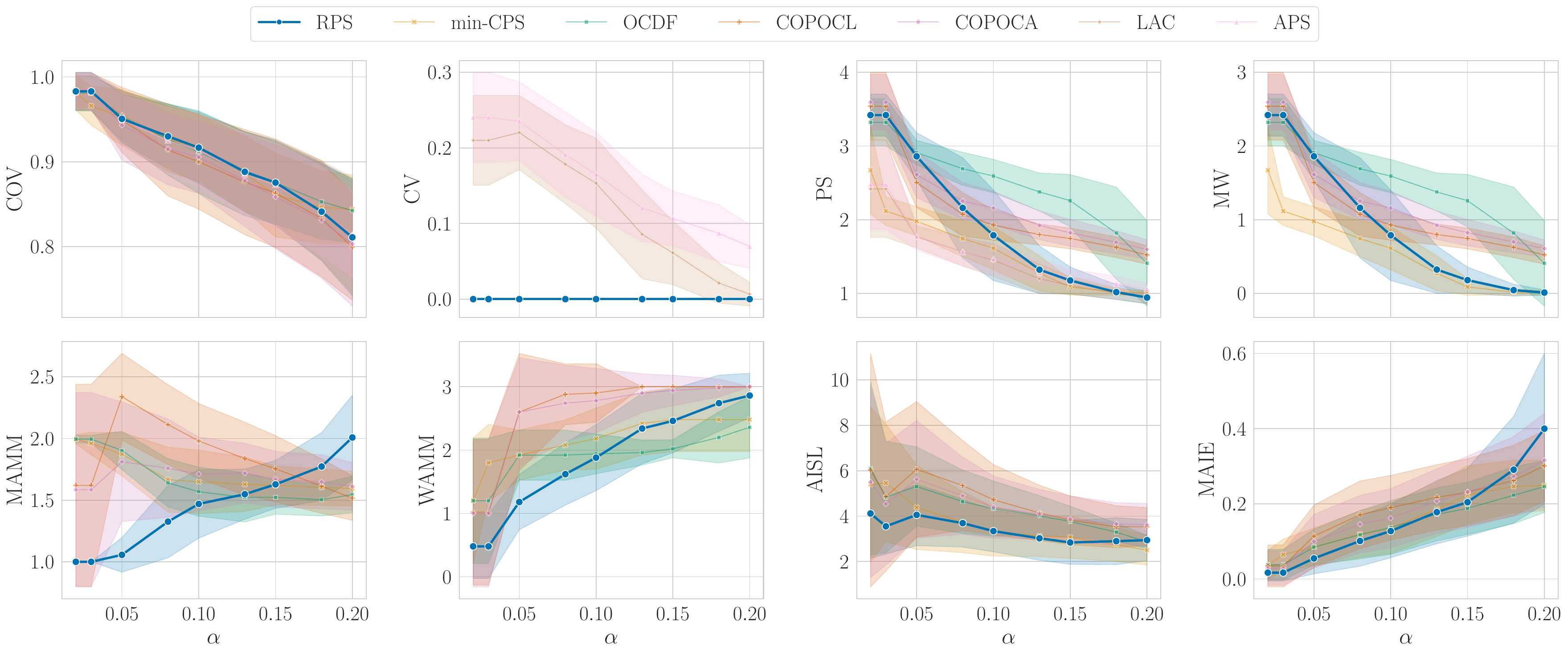}
    \caption{Comparison of prediction sets at $\alpha = \{0.01, 0.02, 0.03, 0.05, 0.08, 0.1, 0.13, 0.15, 0.18, 0.2\}$ across methods on the BACH dataset~\cite{aresta2019bach}. Shaded regions indicate standard deviation over 50 trials.}
    \label{fig:bach}
\end{figure*}

  \begin{figure*}[!htb]
     \centering
    \includegraphics[width=\linewidth]{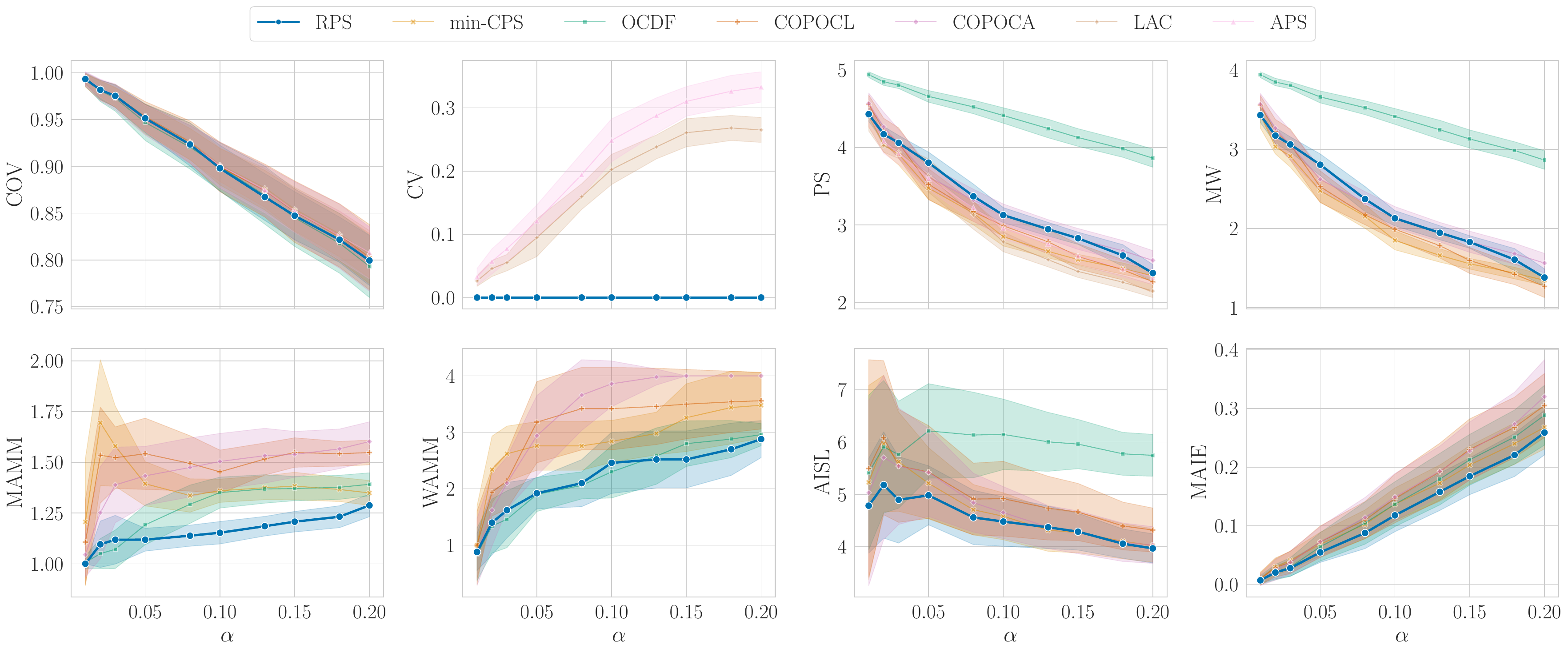}
       \caption{Comparison of prediction sets across methods on the RetinaMNIST dataset~\cite{medmnistv2}. Shaded regions indicate standard deviation.}
       \label{fig:retina}
 \end{figure*}

   \begin{figure*}[!htb]
     \centering
    \includegraphics[width=\linewidth]{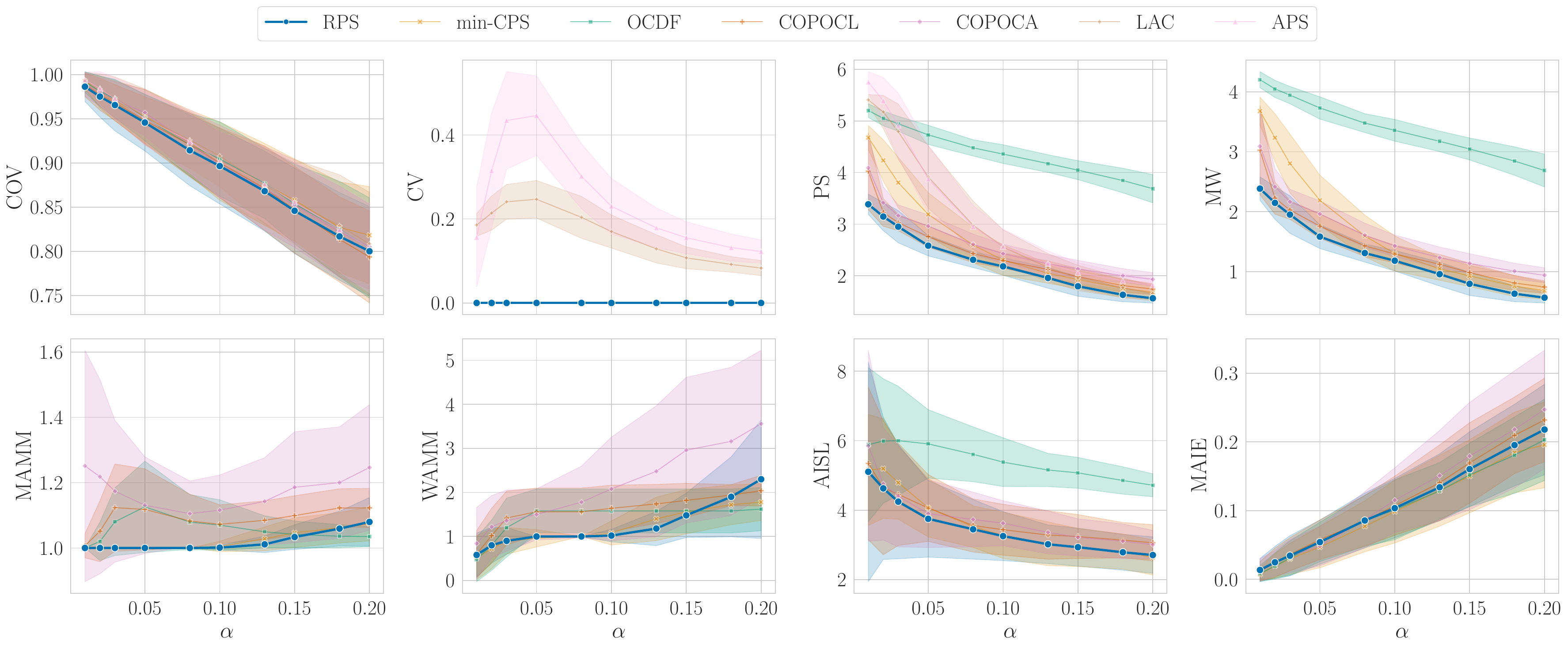}
       \caption{Comparison of prediction sets across methods on the FGNet dataset~\cite{lanitis2002toward}. Shaded regions indicate standard deviation.}
       \label{fig:fgnet}
 \end{figure*}

\section{Additional Experiments on Tabular Ordinal Datasets}
\label{sec:experiments_ordinal_datasets}

\paragraph{Model implementation.}
We conduct additional experiments on tabular datasets using multilayer perceptrons (MLPs), which allow straightforward integration of the COPOC architecture~\cite{DBLP:conf/nips/DeyMK23}.
We use a simple MLP with a single hidden layer of 64 units to maintain consistency across datasets and isolate the effect of the conformal prediction method. Our focus is on evaluating conformal prediction performance rather than optimizing base model accuracy; accordingly, we do not perform hyperparameter tuning. Nonetheless, the model achieves reasonably close-to-standard predictive performance on the datasets and successfully captures the ordinal structure of the labels.
See Table~\ref{tab:results_ord_tab} for the predictive performance of COPOC and CE loss on the considered datasets. 

\paragraph{Tabular Datasets.}
The tabular ordinal datasets are obtained from the TOC-UCO repository~\cite{ayllon2025toc} (see Table~\ref{tab:datasets} for details). 
All features are already numeric and are then standardized using standard scaling.
For evaluation, approximately 60\% of each dataset is used for training, with the remaining 40\% split evenly between calibration and test sets. The results are averaged over 50 random splits between calibration and test sets, with the training data held constant. We focus on several larger datasets exhibiting diverse class distributions, which induce predictive distributions ranging from bimodal to unimodal, as reflected by the mean predicted probabilities (MP) (see Figure~\ref{fig:mps_tabular_datasets}).

\paragraph{Experimental Results.}
Again, we exclude LAC and APS, which may produce non-contiguous prediction sets, from metrics other than COV, CV, and PS.
See the figures below for CP results across several $\alpha$ values ($\alpha=\{0.01,0.02,0.03,0.05,0.08,0.1,0.13,0.15,0.18,0.2\}$) and Table~\ref{tab:cp_performance_tab_data} for detailed results at $\alpha=0.1$.
Consistent with previous findings, RPS-based conformal prediction tends to reduce ordinal error magnitudes while often striking a favorable balance between interval width and miscoverage magnitude, as indicated by AISL. COPOCA remains a strong competitor, particularly on the LESTSensors and LEVXSensors datasets, which exhibit bimodal predictive distributions; however, this advantage comes at the cost of larger prediction sets and wider intervals compared to RPS.
Because COPOCA enforces unimodal predictive distributions, bimodal distributions are effectively centered between the extreme classes, thereby reducing ordinal risk in a manner similar to RPS and improving over min-CPS.
In general, LAC and its unimodal COPOCL variant tend to produce the most efficient sets, as indicated by PS and MW. However, this often comes at the cost of underestimating the ordinal risk indicated by MAMM, WAMM, and MAIE.

\begin{table}[ht]
\small
\centering
\caption{Summary of tabular ordinal datasets used in our experiments~\cite{ayllon2025toc}. }
\label{tab:datasets}
\begin{tabular}{lcccccc}
\toprule
Dataset & \#Samples & \#Features & \#Classes & Class distribution & IR \\
\midrule
LESTSensors    & 5,112 & 6  & 4 & (0.16, 0.14, 0.23, 0.46) & 0.98 \\
LEVXSensors    & 5,112 & 6  & 4 & (0.33, 0.08, 0.10, 0.48) & 1.45 \\
nhanes         & 5,223 & 30 & 5 & (0.11, 0.28, 0.40, 0.18, 0.04) & 1.72 \\
swd            & 1,000 & 10 & 4 & (0.03, 0.35, 0.40, 0.22) & 2.33 \\
winequalityRed & 1,599 & 11 & 5 & (0.04, 0.43, 0.40, 0.12, 0.01) & 4.88 \\
insurance         & 1,338  & 9  & 9  & (0.28, 0.19, 0.16, 0.12, 0.06, 0.05, 0.03, 0.03, 0.08) & 1.55 \\
melbourneAirbnb   & 20,036 & 48 & 10 & (0.08, 0.09, 0.10, 0.09, 0.05, 0.12, 0.11, 0.09, 0.09, 0.18) & 1.01 \\
cancerDeathRate   & 3,047  & 29 & 10 & (0.10, 0.09, 0.10, 0.11, 0.12, 0.13, 0.11, 0.08, 0.06, 0.10) & 0.95 \\
era             & 1,000 & 4  & 9  & (0.09, 0.14, 0.18, 0.17, 0.16, 0.12, 0.09, 0.03, 0.02) & 1.66 \\
lev             & 1,000 & 4  & 5  & (0.09, 0.28, 0.40, 0.20, 0.03)             & 2.16 \\
\bottomrule
\end{tabular}
\end{table}

\begin{table}[ht]
\scriptsize
\centering
\caption{Predictive performance on the datasets. Results are reported as mean $\pm$ stddev. Metrics: ACC (Accuracy), 1-OFF (1-Off Accuracy), MAE (Mean Absolute Error), MSE (Mean Squared Error), QWK (Quadratic Weighted Kappa), UMOD (Degree of Unimodality). For ACC, 1-OFF, and QWK, higher is better; for MAE and MSE, lower is better. Bold values indicate the best performance for each metric within each dataset.}
\label{tab:results_ord_tab}
\begin{tabular}{@{}lccccccc@{}}
\toprule
\textbf{Dataset} & \textbf{Loss} & \textbf{ACC} ($\uparrow$) & \textbf{1-OFF} ($\uparrow$) & \textbf{MAE} ($\downarrow$) & \textbf{MSE} ($\downarrow$) & \textbf{QWK} ($\uparrow$) & \textbf{UMOD} \\
\midrule
\multirow{2}{*}{LESTSensors} 
  & COPOC & $0.4851 \pm 0.0077$ & $\mathbf{0.7589} \pm 0.0043$ & $\mathbf{0.8590} \pm 0.0104$ & $\mathbf{1.7536} \pm 0.0265$ & $0.1119 \pm 0.0113$ & 1.000 $\pm$ 0.000\\
  & CE    & $\mathbf{0.4928} \pm 0.0091$ & $0.7552 \pm 0.0082$ & $0.8597 \pm 0.0191$ & $1.7799 \pm 0.0534$ & $\mathbf{0.2579} \pm 0.0207$ & 0.471 $\pm$ 0.013\\
\midrule
\multirow{2}{*}{LEVXSensors} 
  & COPOC & $0.5220 \pm 0.0078$ & $0.6740 \pm 0.0085$ & $1.0011 \pm 0.0215$ & $2.4413 \pm 0.0656$ & $0.3253 \pm 0.0187$ & 1.000 $\pm$ 0.000\\
  & CE    & $\mathbf{0.6091} \pm 0.0091$ & $\mathbf{0.7106} \pm 0.0101$ & $\mathbf{0.8894} \pm 0.0278$ & $\mathbf{2.3046} \pm 0.0835$ & $\mathbf{0.4368} \pm 0.0204$ & 0.0278 $\pm$ 0.004 \\
\midrule
\multirow{2}{*}{nhanes} 
  & COPOC & $0.4098 \pm 0.0093$ & $0.8884 \pm 0.0056$ & $0.7073 \pm 0.0121$ & $0.9526 \pm 0.0224$ & $\mathbf{0.3198} \pm 0.0168$ & 1.000 $\pm$ 0.000 \\
  & CE    & $\mathbf{0.4225} \pm 0.0078$ & $\mathbf{0.8944} \pm 0.0054$ & $\mathbf{0.6870} \pm 0.0097$ & $\mathbf{0.9140} \pm 0.0206$ & $0.2944 \pm 0.0179$ & 0.941 $\pm$ 0.005 \\
\midrule
\multirow{2}{*}{SWD} 
  & COPOC & $0.4730 \pm 0.0267$ & $0.8827 \pm 0.0141$ & $0.6443 \pm 0.0341$ & $0.8789 \pm 0.0571$ & $0.4249 \pm 0.0364$ & 1.000 $\pm$ 0.000\\
  & CE    & $\mathbf{0.5596} \pm 0.0208$ & $\mathbf{0.9648} \pm 0.0091$ & $\mathbf{0.4756} \pm 0.0236$ & $\mathbf{0.5460} \pm 0.0361$ & $\mathbf{0.5031} \pm 0.0327$ & 1.000 $\pm$ 0.000 \\
\midrule
\multirow{2}{*}{winequalityRed} 
  & COPOC & $0.5298 \pm 0.0235$ & $0.9672 \pm 0.0060$ & $0.5030 \pm 0.0248$ & $0.5686 \pm 0.0309$ & $\mathbf{0.4927} \pm 0.0304$ & 1.000 $\pm$ 0.000\\
  & CE    & $\mathbf{0.5693} \pm 0.0156$ & $\mathbf{0.9724} \pm 0.0057$ & $\mathbf{0.4583} \pm 0.0152$ & $\mathbf{0.5136} \pm 0.0201$ & $0.4379 \pm 0.0204$ & 1.000 $\pm$ 0.000 \\
  \midrule
\multirow{2}{*}{insurance} 
  & COPOC & $\mathbf{0.5605} \pm 0.0163$ & $\mathbf{0.7562} \pm 0.0144$ & $\mathbf{1.1058} \pm 0.0507$ & $\mathbf{4.2700} \pm 0.3290$ & $\mathbf{0.6236} \pm 0.0320$ & 1.000 $\pm$ 0.000\\
  & CE    & $0.4943 \pm 0.0151$ & $0.7527 \pm 0.0105$ & $1.2496 \pm 0.0541$ & $5.0466 \pm 0.3768$ & $0.6021 \pm 0.0345$ & 0.000 $\pm$ 0.000 \\
  \midrule
  \multirow{2}{*}{melbourneAirbnb} 
  & COPOC & $0.3447 \pm 0.0054$ & $\mathbf{0.6779} \pm 0.0058$ & $\mathbf{1.2867} \pm 0.0168$ & $\mathbf{3.7262} \pm 0.0873$ & $\mathbf{0.8007} \pm 0.0046$ & 1.000 $\pm$ 0.000\\
  & CE    & $\mathbf{0.3493} \pm 0.0046$ & $0.6483 \pm 0.0056$ & $1.3835 \pm 0.0184$ & $4.4412 \pm 0.1079$ & $0.7740 \pm 0.0054$ & 0.197 $\pm$ 0.005 \\
  \midrule
\multirow{2}{*}{cancerDeathRate} 
  & COPOC & $0.2675 \pm 0.0127$ & $\mathbf{0.6067} \pm 0.0146$ & $\mathbf{1.5839} \pm 0.0532$ & $\mathbf{5.0903} \pm 0.3254$ & $\mathbf{0.6973} \pm 0.0192$ & 1.000 $\pm$ 0.000\\
  & CE    & $\mathbf{0.2685} \pm 0.0123$ & $0.5339 \pm 0.0146$ & $1.8201 \pm 0.0569$ & $6.6766 \pm 0.3864$ & $0.6432 \pm 0.0202$ & 0.083 $\pm$ 0.008 \\
\midrule
\multirow{2}{*}{era} 
  & COPOC & $0.2171 \pm 0.0223$ & $\mathbf{0.6362} \pm 0.0258$ & $\mathbf{1.3670} \pm 0.0615$ & $3.1862 \pm 0.2774$ & $\mathbf{0.5692} \pm 0.0374$ & 1.000 $\pm$ 0.000\\
  & CE    & $\mathbf{0.2312} \pm 0.0184$ & $0.5846 \pm 0.0214$ & $1.3874 \pm 0.0389$ & $\mathbf{3.1610} \pm 0.1495$ & $0.3379 \pm 0.0306$ & 0.166 $\pm$ 0.015\\
  \midrule
\multirow{2}{*}{lev} 
  & COPOC & $\mathbf{0.5737} \pm 0.0223$ & $\mathbf{0.9573} \pm 0.0095$ & $\mathbf{0.4690} \pm 0.0249$ & $\mathbf{0.5544} \pm 0.0377$ & $\mathbf{0.6668} \pm 0.0245$ & 1.000 $\pm$ 0.000\\
  & CE    & $0.4633 \pm 0.0186$ & $0.9500 \pm 0.0082$ & $0.5867 \pm 0.0193$ & $0.6867 \pm 0.0289$ & $0.4009 \pm 0.0272$ & 0.777 $\pm$ 0.012 \\
\bottomrule
\end{tabular}
\end{table}

\begin{table}[ht]
\tiny
\centering
\caption{Performance comparison of the different CP methods on the tabular datasets at $\alpha=0.1$.}
\label{tab:cp_performance_tab_data}
\begin{tabular}{llcccccccc}
\toprule
Dataset & Method & COV & PS ($\downarrow$) & CV ($\downarrow$) & MW ($\downarrow$) & MAMM ($\downarrow$) & WAMM ($\downarrow$) & MAIE ($\downarrow$) & AISL ($\downarrow$) \\
\midrule
 \multirow{7}{*}{LESTSensors} 
 & APS      & $0.904 \pm 0.012$ & $2.932 \pm 0.063$ & $0.343 \pm 0.015$ & $-$ & $-$ & $-$ & $-$ & $-$ \\
 & COPOCA   & $0.901 \pm 0.009$ & $3.08 \pm 0.049$  & $0.0 \pm 0.0$     & $2.081 \pm 0.049$ & $\mathbf{1.053} \pm 0.023$ & $2.54 \pm 0.542$ & $\mathbf{0.104} \pm 0.01$  & $\mathbf{4.17} \pm 0.171$  \\
 & COPOCL   & $0.902 \pm 0.012$ & $2.803 \pm 0.065$ & $0.0 \pm 0.0$     & $\mathbf{1.803} \pm 0.065$ & $1.453 \pm 0.053$ & $3.0 \pm 0.0$    & $0.142 \pm 0.02$  & $4.646 \pm 0.341$ \\
 & LAC      & $0.901 \pm 0.01$  & $\mathbf{2.755} \pm 0.049$ & $0.232 \pm 0.016$ & $-$ & $-$ & $-$ & $-$ & $-$ \\
 & OCDF     & $0.901 \pm 0.012$ & $2.871 \pm 0.048$ & $0.0 \pm 0.0$     & $1.871 \pm 0.048$ & $1.319 \pm 0.041$ & $3.0 \pm 0.0$    & $0.131 \pm 0.016$ & $4.482 \pm 0.281$ \\
 & RPS      & $0.901 \pm 0.008$ & $2.966 \pm 0.028$ & $0.0 \pm 0.0$     & $1.966 \pm 0.028$ & $1.12 \pm 0.02$   & $\mathbf{2.0} \pm 0.0$    & $0.111 \pm 0.009$ & $4.182 \pm 0.166$ \\
 & min-CPS  & $0.901 \pm 0.01$  & $2.823 \pm 0.036$ & $0.0 \pm 0.0$     & $1.823 \pm 0.036$ & $1.216 \pm 0.032$ & $3.0 \pm 0.0$    & $0.121 \pm 0.014$ & $4.235 \pm 0.241$ \\
\midrule
 \multirow{7}{*}{LEVXSensors} 
 & APS      & $0.897 \pm 0.009$ & $2.487 \pm 0.051$ & $0.78 \pm 0.01$  & $-$ & $-$ & $-$ & $-$ & $-$ \\
 & COPOCA   & $0.897 \pm 0.012$ & $3.414 \pm 0.051$ & $0.0 \pm 0.0$    & $2.414 \pm 0.051$ & $\mathbf{1.027} \pm 0.016$ & $2.06 \pm 0.424$ & $\mathbf{0.106} \pm 0.012$ & $\mathbf{4.525} \pm 0.208$ \\
 & COPOCL   & $0.899 \pm 0.014$ & $2.927 \pm 0.053$ & $0.0 \pm 0.0$    & $1.927 \pm 0.053$ & $1.843 \pm 0.066$ & $3.0 \pm 0.0$    & $0.186 \pm 0.026$ & $5.638 \pm 0.474$ \\
 & LAC      & $0.899 \pm 0.011$ & $\mathbf{2.405} \pm 0.055$ & $0.626 \pm 0.016$ & $-$ & $-$ & $-$ & $-$ & $-$ \\
 & OCDF     & $0.898 \pm 0.012$ & $2.808 \pm 0.057$ & $0.0 \pm 0.0$    & $\mathbf{1.808} \pm 0.057$ & $1.876 \pm 0.059$ & $3.0 \pm 0.0$    & $0.192 \pm 0.027$ & $5.65 \pm 0.483$  \\
 & RPS      & $0.896 \pm 0.014$ & $3.104 \pm 0.042$ & $0.0 \pm 0.0$    & $2.104 \pm 0.042$ & $1.174 \pm 0.029$ & $\mathbf{2.0} \pm 0.0$    & $0.122 \pm 0.017$ & $4.538 \pm 0.295$ \\
 & min-CPS  & $0.896 \pm 0.015$ & $2.874 \pm 0.052$ & $0.0 \pm 0.0$    & $1.874 \pm 0.052$ & $1.458 \pm 0.059$ & $3.0 \pm 0.0$    & $0.152 \pm 0.022$ & $4.911 \pm 0.394$ \\
 \midrule
  \multirow{7}{*}{nhanes} 
 & APS      & $0.899 \pm 0.013$ & $3.129 \pm 0.066$ & $0.03 \pm 0.005$ & $-$ & $-$ & $-$ & $-$ & $-$ \\
 & COPOCA   & $0.899 \pm 0.012$ & $3.107 \pm 0.062$ & $0.0 \pm 0.0$   & $2.107 \pm 0.062$ & $1.073 \pm 0.02$  & $2.2 \pm 0.606$  & $0.109 \pm 0.013$ & $4.279 \pm 0.214$ \\
 & COPOCL   & $0.9 \pm 0.013$   & $2.974 \pm 0.055$ & $0.0 \pm 0.0$   & $\mathbf{1.974} \pm 0.055$ & $1.102 \pm 0.022$ & $2.0 \pm 0.0$    & $0.111 \pm 0.014$ & $4.185 \pm 0.235$ \\
 & LAC      & $0.899 \pm 0.012$ & $\mathbf{2.971} \pm 0.052$ & $0.011 \pm 0.002$ & $-$ & $-$ & $-$ & $-$ & $-$ \\
 & OCDF     & $0.902 \pm 0.01$  & $4.018 \pm 0.045$ & $0.0 \pm 0.0$   & $3.018 \pm 0.045$ & $1.143 \pm 0.027$ & $2.92 \pm 0.274$ & $0.112 \pm 0.013$ & $5.264 \pm 0.209$ \\
 & RPS      & $0.9 \pm 0.012$   & $3.005 \pm 0.052$ & $0.0 \pm 0.0$   & $2.005 \pm 0.052$ & $\mathbf{1.041} \pm 0.018$ & $\mathbf{1.96} \pm 0.198$ & $\mathbf{0.104} \pm 0.013$ & $\mathbf{4.092} \pm 0.212$ \\
 & min-CPS  & $0.899 \pm 0.01$  & $3.064 \pm 0.046$ & $0.0 \pm 0.0$   & $2.064 \pm 0.046$ & $1.062 \pm 0.018$ & $2.0 \pm 0.0$    & $0.107 \pm 0.011$ & $4.203 \pm 0.173$ \\
 \midrule
  \multirow{7}{*}{swd} 
 & APS      & $0.906 \pm 0.025$ & $2.221 \pm 0.086$ & $0.0 \pm 0.0$   & $-$ & $-$ & $-$ & $-$ & $-$ \\
 & COPOCA   & $0.901 \pm 0.027$ & $2.459 \pm 0.098$ & $0.0 \pm 0.0$   & $1.459 \pm 0.098$ & $1.002 \pm 0.007$ & $1.04 \pm 0.198$  & $0.099 \pm 0.028$ & $3.447 \pm 0.476$ \\
 & COPOCL   & $0.899 \pm 0.033$ & $2.464 \pm 0.168$ & $0.0 \pm 0.0$   & $1.464 \pm 0.168$ & $\mathbf{1.0} \pm 0.0$     & $\mathbf{1.0} \pm 0.0$     & $0.101 \pm 0.033$ & $3.486 \pm 0.506$ \\
 & LAC      & $0.903 \pm 0.026$ & $\mathbf{2.131} \pm 0.074$ & $0.0 \pm 0.0$   & $-$ & $-$ & $-$ & $-$ & $-$ \\
 & OCDF     & $0.902 \pm 0.027$ & $2.686 \pm 0.102$ & $0.0 \pm 0.0$   & $1.686 \pm 0.102$ & $1.017 \pm 0.025$ & $1.34 \pm 0.479$  & $0.1 \pm 0.029$   & $3.694 \pm 0.492$ \\
 & RPS      & $0.901 \pm 0.025$ & $2.15 \pm 0.051$  & $0.0 \pm 0.0$   & $\mathbf{1.15} \pm 0.051$  & $\mathbf{1.0} \pm 0.0$     & $\mathbf{1.0} \pm 0.0$     & $0.099 \pm 0.025$ & $\mathbf{3.136} \pm 0.46$  \\
 & min-CPS  & $0.906 \pm 0.028$ & $2.356 \pm 0.07$  & $0.0 \pm 0.0$   & $1.356 \pm 0.07$  & $1.001 \pm 0.006$ & $1.04 \pm 0.198$  & $\mathbf{0.094} \pm 0.029$ & $3.24 \pm 0.516$  \\
 \midrule
  \multirow{7}{*}{winequalityRed} 
 & APS      & $0.905 \pm 0.014$ & $2.383 \pm 0.101$ & $0.0 \pm 0.0$   & $-$ & $-$ & $-$ & $-$ & $-$ \\
 & COPOCA   & $0.903 \pm 0.023$ & $2.481 \pm 0.132$ & $0.0 \pm 0.0$   & $1.481 \pm 0.132$ & $1.026 \pm 0.023$ & $1.64 \pm 0.485$ & $0.1 \pm 0.023$   & $3.478 \pm 0.344$ \\
 & COPOCL   & $0.903 \pm 0.021$ & $2.471 \pm 0.084$ & $0.0 \pm 0.0$   & $1.471 \pm 0.084$ & $1.029 \pm 0.023$ & $1.68 \pm 0.471$ & $0.1 \pm 0.022$   & $3.476 \pm 0.369$ \\
 & LAC      & $0.901 \pm 0.015$ & $2.138 \pm 0.126$ & $0.0 \pm 0.0$   & $-$ & $-$ & $-$ & $-$ & $-$ \\
 & OCDF     & $0.903 \pm 0.021$ & $3.25 \pm 0.14$   & $0.0 \pm 0.0$   & $2.25 \pm 0.14$   & $\mathbf{1.015} \pm 0.018$ & $\mathbf{1.44} \pm 0.501$ & $\mathbf{0.099} \pm 0.021$ & $4.221 \pm 0.291$ \\
 & RPS      & $0.901 \pm 0.016$ & $\mathbf{2.135} \pm 0.092$ & $0.0 \pm 0.0$   & $\mathbf{1.135} \pm 0.092$ & $1.018 \pm 0.02$  & $1.5 \pm 0.505$  & $0.101 \pm 0.017$ & $\mathbf{3.163} \pm 0.264$ \\
 & min-CPS  & $0.902 \pm 0.017$ & $2.296 \pm 0.103$ & $0.0 \pm 0.0$   & $1.296 \pm 0.103$ & $1.047 \pm 0.025$ & $1.9 \pm 0.303$  & $0.103 \pm 0.017$ & $3.35 \pm 0.255$  \\
 \midrule
 \multirow{7}{*}{insurance} 
 & APS      & $0.9 \pm 0.018$   & $4.284 \pm 0.277$ & $0.495 \pm 0.072$ & $-$ & $-$ & $-$ & $-$ & $-$ \\
 & COPOCA   & $0.9 \pm 0.019$   & $4.599 \pm 0.38$  & $0.0 \pm 0.0$     & $3.599 \pm 0.38$  & $1.914 \pm 0.247$ & $4.22 \pm 0.79$  & $0.196 \pm 0.061$ & $7.51 \pm 0.852$  \\
 & COPOCL   & $0.901 \pm 0.02$  & $4.464 \pm 0.241$ & $0.0 \pm 0.0$     & $\mathbf{3.464} \pm 0.241$ & $2.163 \pm 0.177$ & $4.48 \pm 0.646$ & $0.216 \pm 0.056$ & $7.793 \pm 0.891$ \\
 & LAC      & $0.899 \pm 0.018$ & $\mathbf{3.922} \pm 0.241$ & $0.342 \pm 0.043$ & $-$ & $-$ & $-$ & $-$ & $-$ \\
 & OCDF     & $0.901 \pm 0.021$ & $5.631 \pm 0.257$ & $0.0 \pm 0.0$     & $4.631 \pm 0.257$ & $2.259 \pm 0.199$ & $4.9 \pm 0.303$  & $0.223 \pm 0.052$ & $9.088 \pm 0.818$ \\
 & RPS      & $0.901 \pm 0.017$ & $5.954 \pm 0.138$ & $0.0 \pm 0.0$     & $4.954 \pm 0.138$ & $\mathbf{1.231} \pm 0.083$ & $\mathbf{2.58} \pm 0.538$ & $\mathbf{0.123} \pm 0.027$ & $\mathbf{7.416} \pm 0.404$ \\
 & min-CPS  & $0.898 \pm 0.021$ & $4.75 \pm 0.364$  & $0.0 \pm 0.0$     & $3.75 \pm 0.364$  & $1.963 \pm 0.229$ & $4.04 \pm 0.402$ & $0.204 \pm 0.063$ & $7.829 \pm 0.911$ \\
 \midrule
  \multirow{7}{*}{melbourneAirbnb} 
 & APS      & $0.899 \pm 0.008$ & $4.851 \pm 0.069$ & $0.233 \pm 0.005$ & $-$ & $-$ & $-$ & $-$ & $-$ \\
 & COPOCA   & $0.899 \pm 0.007$ & $4.905 \pm 0.066$ & $0.0 \pm 0.0$     & $3.907 \pm 0.066$ & $1.879 \pm 0.06$  & $8.98 \pm 0.141$ & $0.191 \pm 0.015$ & $7.716 \pm 0.248$ \\
 & COPOCL   & $0.898 \pm 0.007$ & $4.741 \pm 0.052$ & $0.0 \pm 0.0$     & $\mathbf{3.741} \pm 0.052$ & $1.924 \pm 0.044$ & $7.82 \pm 0.438$ & $0.196 \pm 0.014$ & $7.653 \pm 0.229$ \\
 & LAC      & $0.899 \pm 0.007$ & $\mathbf{4.725} \pm 0.055$ & $0.163 \pm 0.004$ & $-$ & $-$ & $-$ & $-$ & $-$ \\
 & OCDF     & $0.9 \pm 0.006$   & $6.76 \pm 0.056$  & $0.0 \pm 0.0$     & $5.76 \pm 0.056$  & $1.942 \pm 0.04$  & $7.78 \pm 0.418$ & $0.194 \pm 0.012$ & $9.632 \pm 0.187$ \\
 & RPS      & $0.899 \pm 0.007$ & $4.915 \pm 0.052$ & $0.0 \pm 0.0$     & $3.915 \pm 0.052$ & $\mathbf{1.768} \pm 0.041$ & $\mathbf{6.56} \pm 0.501$ & $\mathbf{0.178} \pm 0.013$ & $\mathbf{7.47} \pm 0.204$  \\
 & min-CPS  & $0.899 \pm 0.009$ & $4.813 \pm 0.061$ & $0.0 \pm 0.0$     & $3.813 \pm 0.061$ & $1.822 \pm 0.039$ & $7.76 \pm 0.431$ & $0.185 \pm 0.016$ & $7.51 \pm 0.255$  \\
 \midrule
  \multirow{7}{*}{cancerDeathRate} 
 & APS      & $0.899 \pm 0.018$ & $6.025 \pm 0.186$ & $0.179 \pm 0.018$ & $-$ & $-$ & $-$ & $-$ & $-$ \\
 & COPOCA   & $0.902 \pm 0.021$ & $5.952 \pm 0.218$ & $0.0 \pm 0.0$     & $\mathbf{4.952} \pm 0.218$ & $1.91 \pm 0.128$  & $5.72 \pm 0.784$ & $0.187 \pm 0.045$ & $8.695 \pm 0.686$ \\
 & COPOCL   & $0.904 \pm 0.021$ & $6.002 \pm 0.221$ & $0.0 \pm 0.0$     & $5.002 \pm 0.221$ & $1.92 \pm 0.106$  & $5.2 \pm 0.404$  & $0.185 \pm 0.043$ & $8.693 \pm 0.659$ \\
 & LAC      & $0.899 \pm 0.016$ & $\mathbf{6.04} \pm 0.168$  & $0.112 \pm 0.016$ & $-$ & $-$ & $-$ & $-$ & $-$ \\
 & OCDF     & $0.902 \pm 0.017$ & $7.213 \pm 0.159$ & $0.0 \pm 0.0$     & $6.213 \pm 0.159$ & $2.278 \pm 0.143$ & $6.52 \pm 0.735$ & $0.223 \pm 0.037$ & $10.677 \pm 0.618$ \\
 & RPS      & $0.9 \pm 0.019$   & $6.093 \pm 0.184$ & $0.0 \pm 0.0$     & $5.093 \pm 0.184$ & $\mathbf{1.638} \pm 0.102$ & $\mathbf{4.88} \pm 0.328$ & $\mathbf{0.163} \pm 0.033$ & $\mathbf{8.354} \pm 0.483$ \\
 & min-CPS  & $0.899 \pm 0.017$ & $5.987 \pm 0.155$ & $0.0 \pm 0.0$     & $4.987 \pm 0.155$ & $1.882 \pm 0.135$ & $5.58 \pm 0.499$ & $0.191 \pm 0.038$ & $8.812 \pm 0.62$  \\
 \midrule
  \multirow{7}{*}{era} 
 & APS      & $0.9 \pm 0.024$   & $5.706 \pm 0.178$ & $0.19 \pm 0.035$ & $-$ & $-$ & $-$ & $-$ & $-$ \\
 & COPOCA   & $0.905 \pm 0.027$ & $5.679 \pm 0.276$ & $0.0 \pm 0.0$   & $4.679 \pm 0.276$ & $1.46 \pm 0.138$  & $2.92 \pm 0.601$ & $0.14 \pm 0.045$  & $7.477 \pm 0.634$ \\
 & COPOCL   & $0.902 \pm 0.024$ & $\mathbf{5.639} \pm 0.218$ & $0.0 \pm 0.0$   & $\mathbf{4.639} \pm 0.218$ & $1.506 \pm 0.108$ & $2.8 \pm 0.452$  & $0.148 \pm 0.04$  & $7.609 \pm 0.598$ \\
 & LAC      & $0.903 \pm 0.022$ & $5.842 \pm 0.234$ & $0.116 \pm 0.044$ & $-$ & $-$ & $-$ & $-$ & $-$ \\
 & OCDF     & $0.908 \pm 0.018$ & $6.289 \pm 0.144$ & $0.0 \pm 0.0$   & $5.289 \pm 0.144$ & $1.464 \pm 0.117$ & $2.84 \pm 0.37$  & $0.135 \pm 0.028$ & $7.985 \pm 0.443$ \\
 & RPS      & $0.905 \pm 0.026$ & $5.877 \pm 0.244$ & $0.0 \pm 0.0$   & $4.877 \pm 0.244$ & $\mathbf{1.066} \pm 0.058$ & $\mathbf{1.68} \pm 0.471$ & $\mathbf{0.102} \pm 0.032$ & $\mathbf{6.913} \pm 0.41$  \\
 & min-CPS  & $0.909 \pm 0.019$ & $6.013 \pm 0.036$ & $0.0 \pm 0.0$   & $5.013 \pm 0.036$ & $1.308 \pm 0.079$ & $2.42 \pm 0.499$ & $0.118 \pm 0.02$  & $7.371 \pm 0.376$ \\
 \midrule
  \multirow{7}{*}{lev} 
 & APS      & $0.902 \pm 0.021$ & $2.824 \pm 0.111$ & $0.014 \pm 0.01$ & $-$ & $-$ & $-$ & $-$ & $-$ \\
 & COPOCA   & $0.902 \pm 0.025$ & $2.675 \pm 0.148$ & $0.0 \pm 0.0$   & $1.675 \pm 0.148$ & $1.066 \pm 0.041$ & $1.88 \pm 0.328$ & $0.105 \pm 0.027$ & $3.771 \pm 0.429$ \\
 & COPOCL   & $0.905 \pm 0.024$ & $\mathbf{2.45} \pm 0.146$  & $0.0 \pm 0.0$   & $\mathbf{1.45} \pm 0.146$  & $1.059 \pm 0.038$ & $1.82 \pm 0.388$ & $0.101 \pm 0.026$ & $\mathbf{3.462} \pm 0.386$ \\
 & LAC      & $0.897 \pm 0.022$ & $2.628 \pm 0.115$ & $0.0 \pm 0.0$   & $-$ & $-$ & $-$ & $-$ & $-$ \\
 & OCDF     & $0.901 \pm 0.026$ & $3.515 \pm 0.106$ & $0.0 \pm 0.0$   & $2.515 \pm 0.106$ & $1.106 \pm 0.052$ & $1.94 \pm 0.24$  & $0.11 \pm 0.029$  & $4.709 \pm 0.49$  \\
 & RPS      & $0.904 \pm 0.014$ & $2.857 \pm 0.073$ & $0.0 \pm 0.0$   & $1.857 \pm 0.073$ & $\mathbf{1.0} \pm 0.0$     & $\mathbf{1.0} \pm 0.0$    & $\mathbf{0.096} \pm 0.014$ & $3.781 \pm 0.216$ \\
 & min-CPS  & $0.902 \pm 0.026$ & $2.812 \pm 0.106$ & $0.0 \pm 0.0$   & $1.812 \pm 0.106$ & $1.057 \pm 0.045$ & $1.78 \pm 0.418$ & $0.102 \pm 0.026$ & $3.862 \pm 0.426$ \\
\bottomrule
\end{tabular}
\end{table}

\begin{figure*}[htb!]
 \centering
 \begin{minipage}[b]{\textwidth}
        \centering
\includegraphics[width=\linewidth]{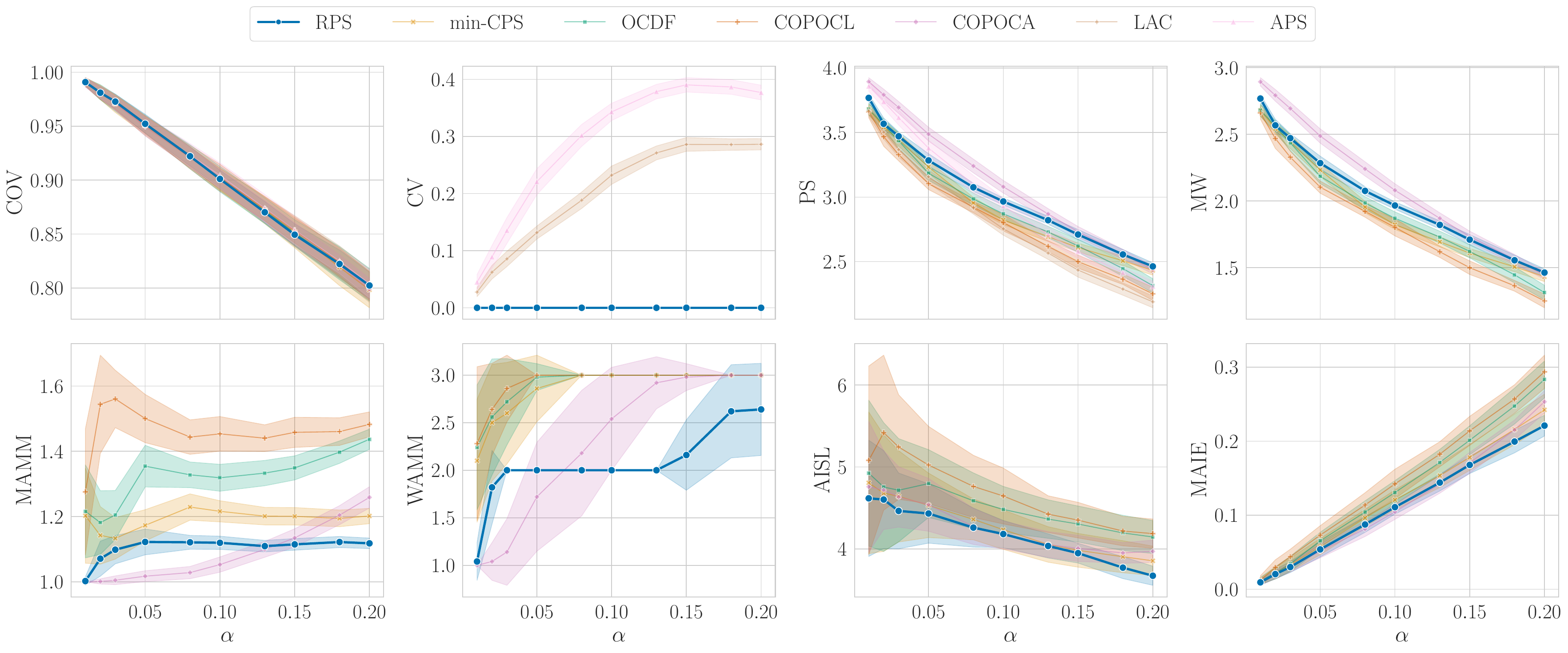}
    \caption{Comparison of prediction sets across methods on LESTSensors dataset. Shaded regions indicate standard deviation.}
 \end{minipage}
 \begin{minipage}[b]{\textwidth}
        \centering
\includegraphics[width=\linewidth]{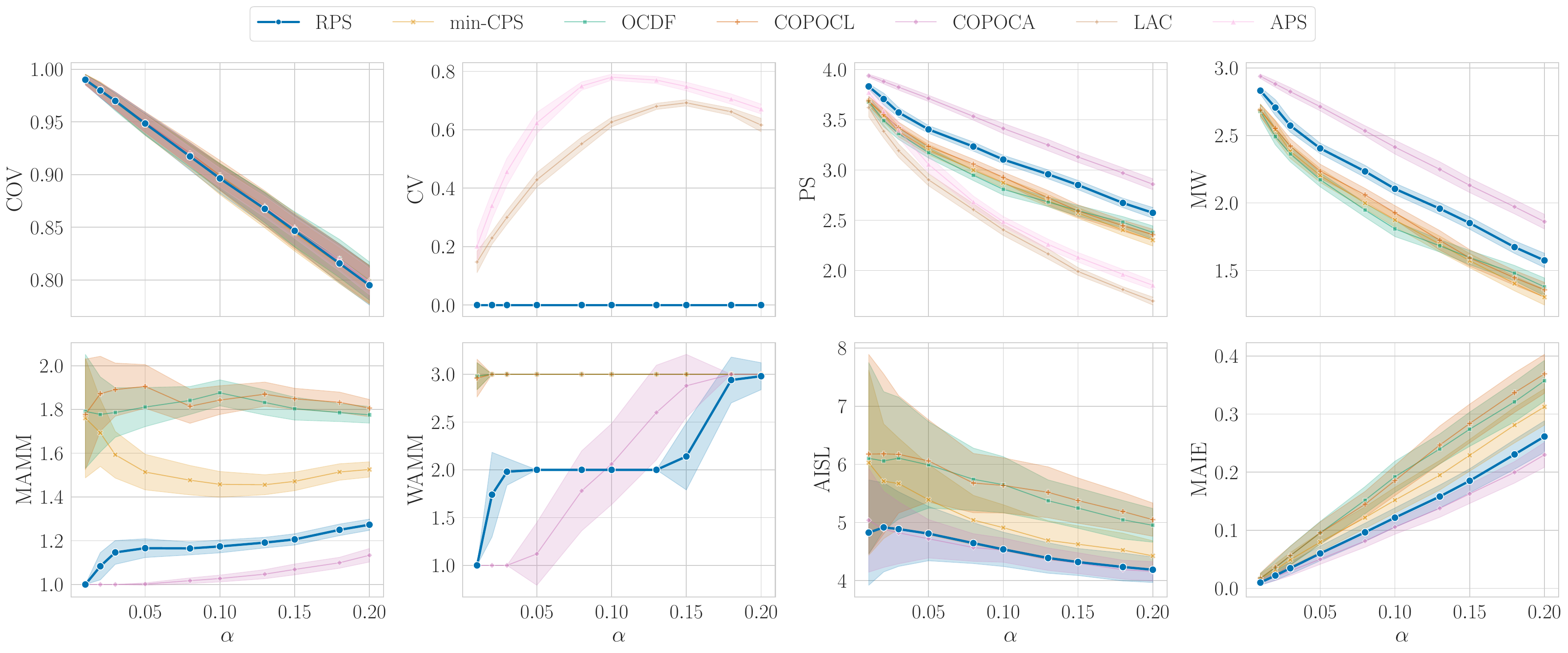}
    \caption{Comparison of prediction sets across methods on LEVXSensors dataset. Shaded regions indicate standard deviation.}
 \end{minipage}
    \begin{minipage}[b]{\textwidth}
        \centering
\includegraphics[width=\linewidth]{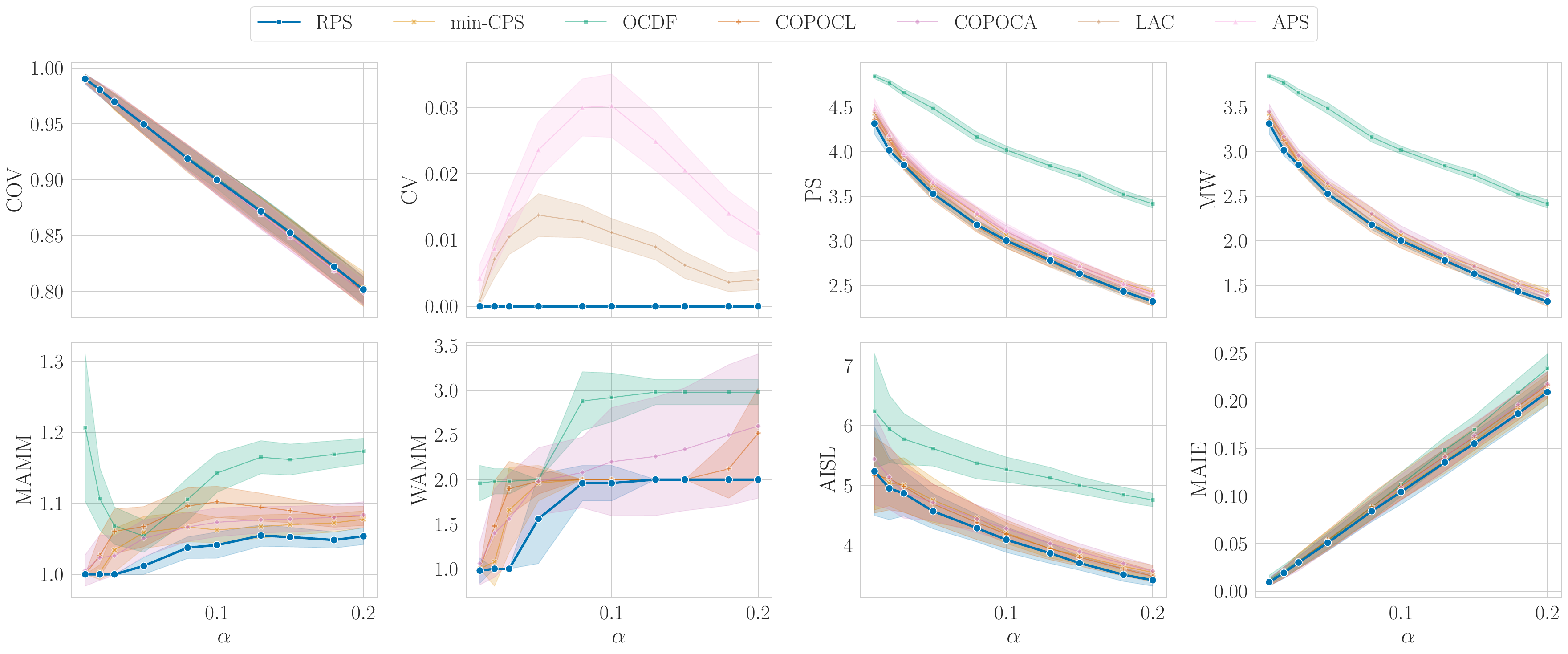}
    \caption{Comparison of prediction sets across methods on nhanes dataset. Shaded regions indicate standard deviation.}
 \end{minipage}
\end{figure*}

\begin{figure*}[htb!]
 \centering
       \begin{minipage}[b]{\textwidth}
        \centering
\includegraphics[width=\linewidth]{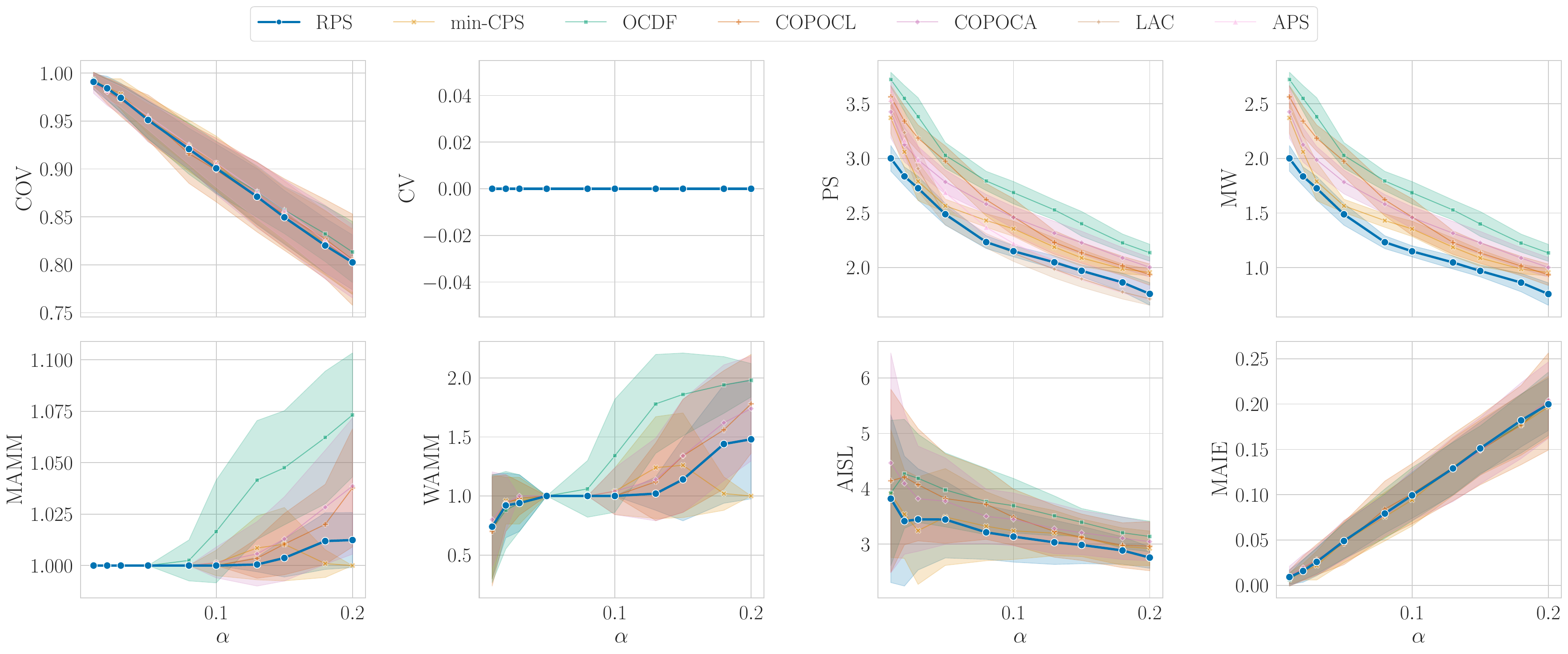}
    \caption{Comparison of prediction sets across methods on swd dataset. Shaded regions indicate standard deviation.}
 \end{minipage}
     \begin{minipage}[b]{\textwidth}
        \centering
\includegraphics[width=\linewidth]{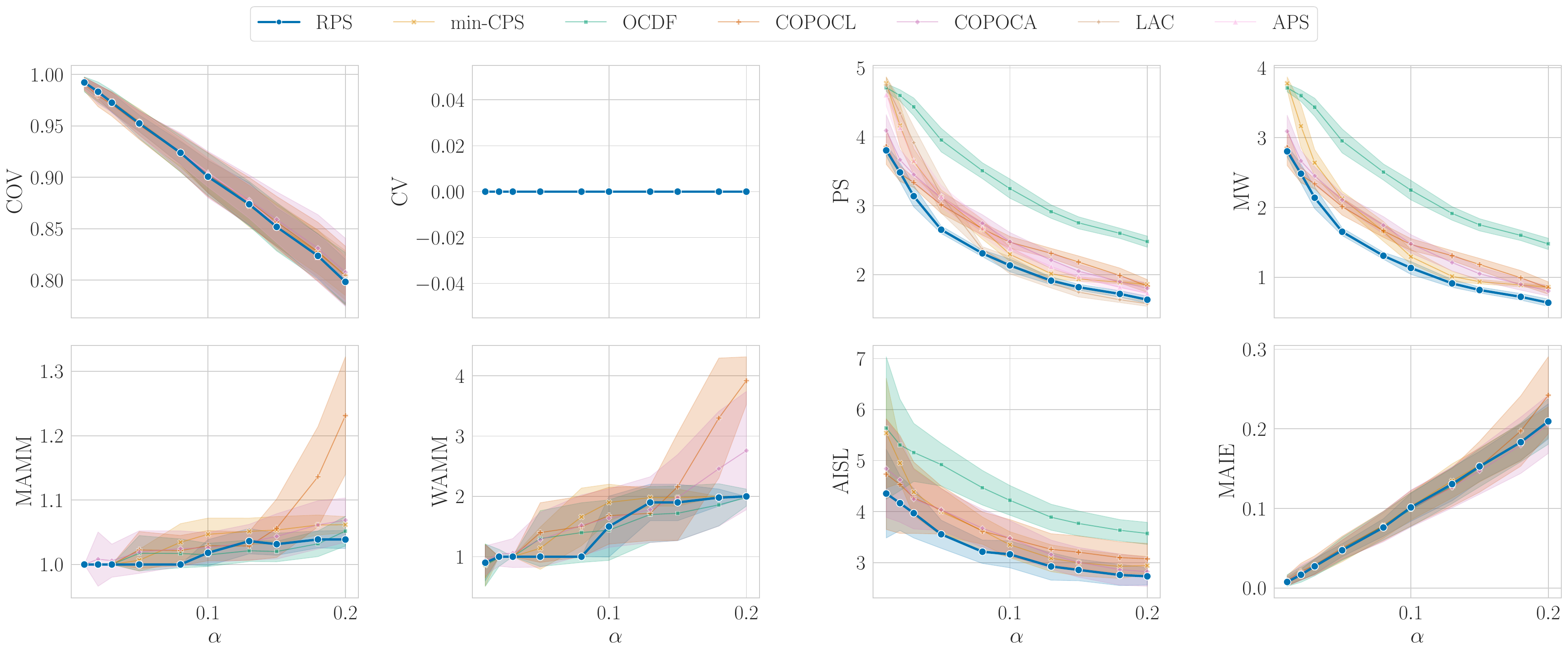}
    \caption{Comparison of prediction sets across methods on winequalityRed dataset. Shaded regions indicate standard deviation.}
    \label{fig:tab_dataset_exps}
 \end{minipage}
  \begin{minipage}[b]{\textwidth}
        \centering
\includegraphics[width=\linewidth]{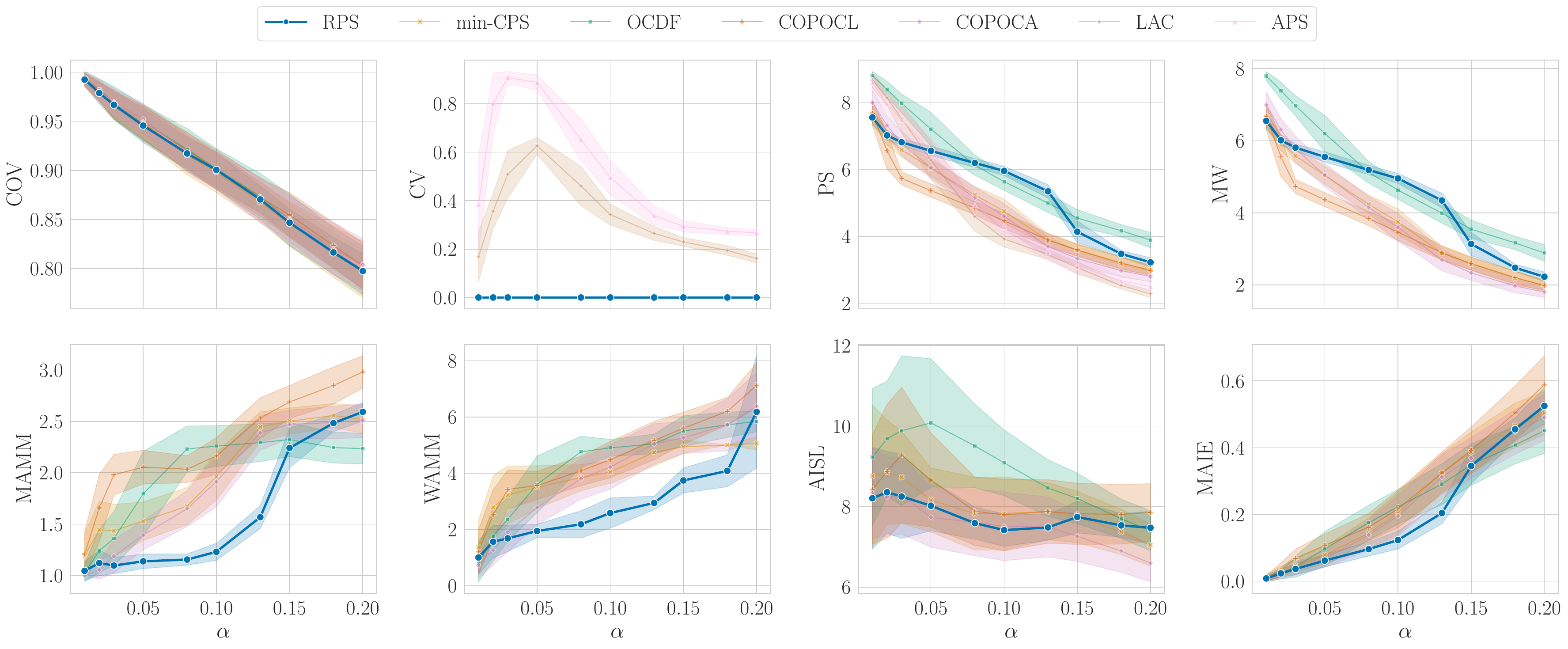}
    \caption{Comparison of prediction sets across methods on insurance dataset. Shaded regions indicate standard deviation.}
 \end{minipage}
\end{figure*}

\begin{figure*}[htb!]
 \centering
 \begin{minipage}[b]{\textwidth}
        \centering
\includegraphics[width=\linewidth]{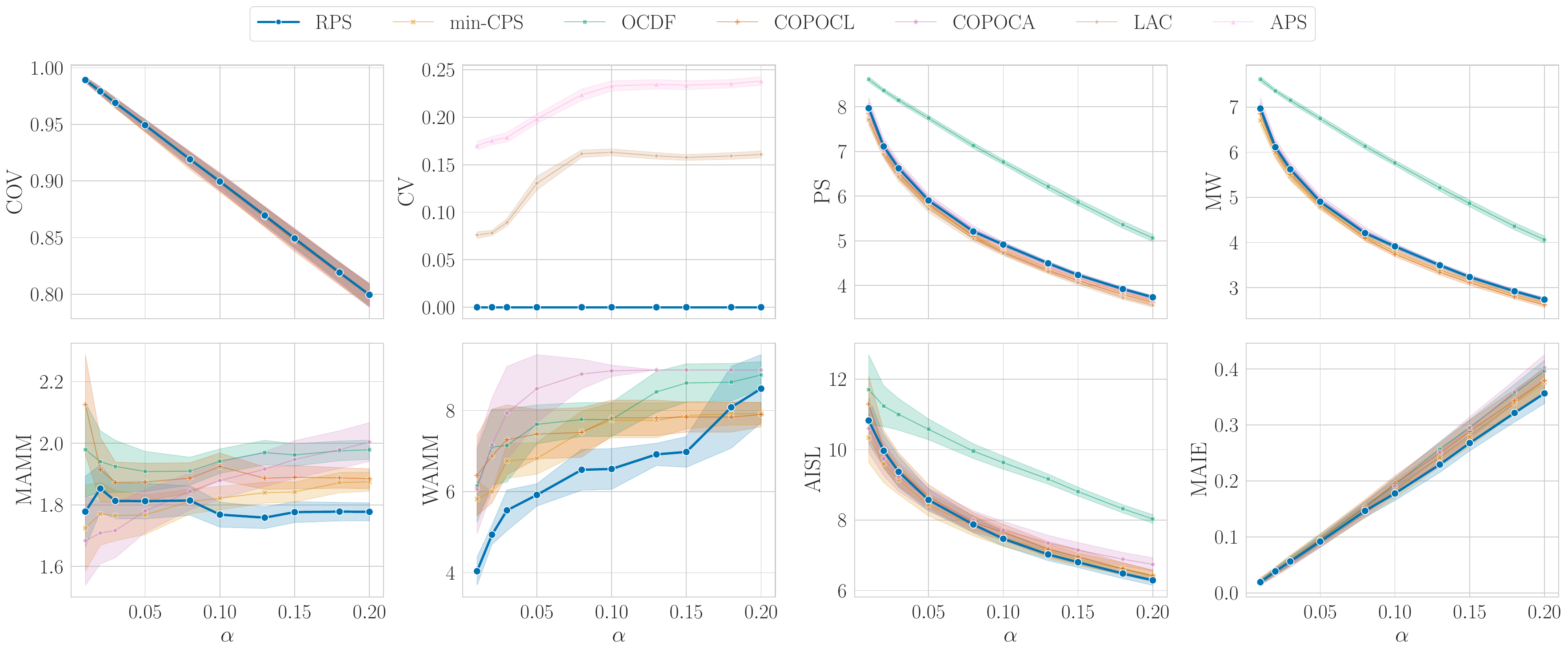}
    \caption{Comparison of prediction sets across methods on melbourneAirbnb dataset. Shaded regions indicate standard deviation.}
 \end{minipage}
    \begin{minipage}[b]{\textwidth}
        \centering
\includegraphics[width=\linewidth]{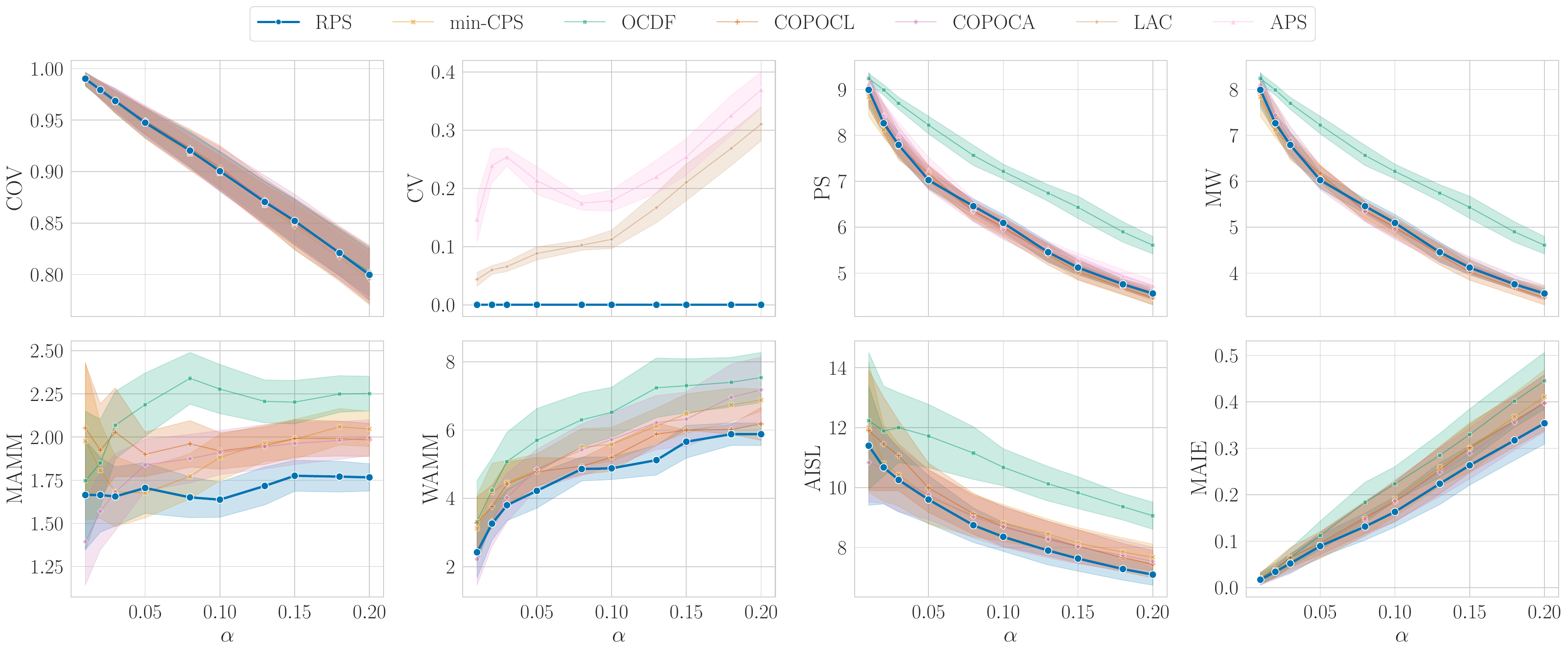}
    \caption{Comparison of prediction sets across methods on cancerDeathRate dataset. Shaded regions indicate standard deviation.}
 \end{minipage}
     \begin{minipage}[b]{\textwidth}
        \centering
\includegraphics[width=\linewidth]{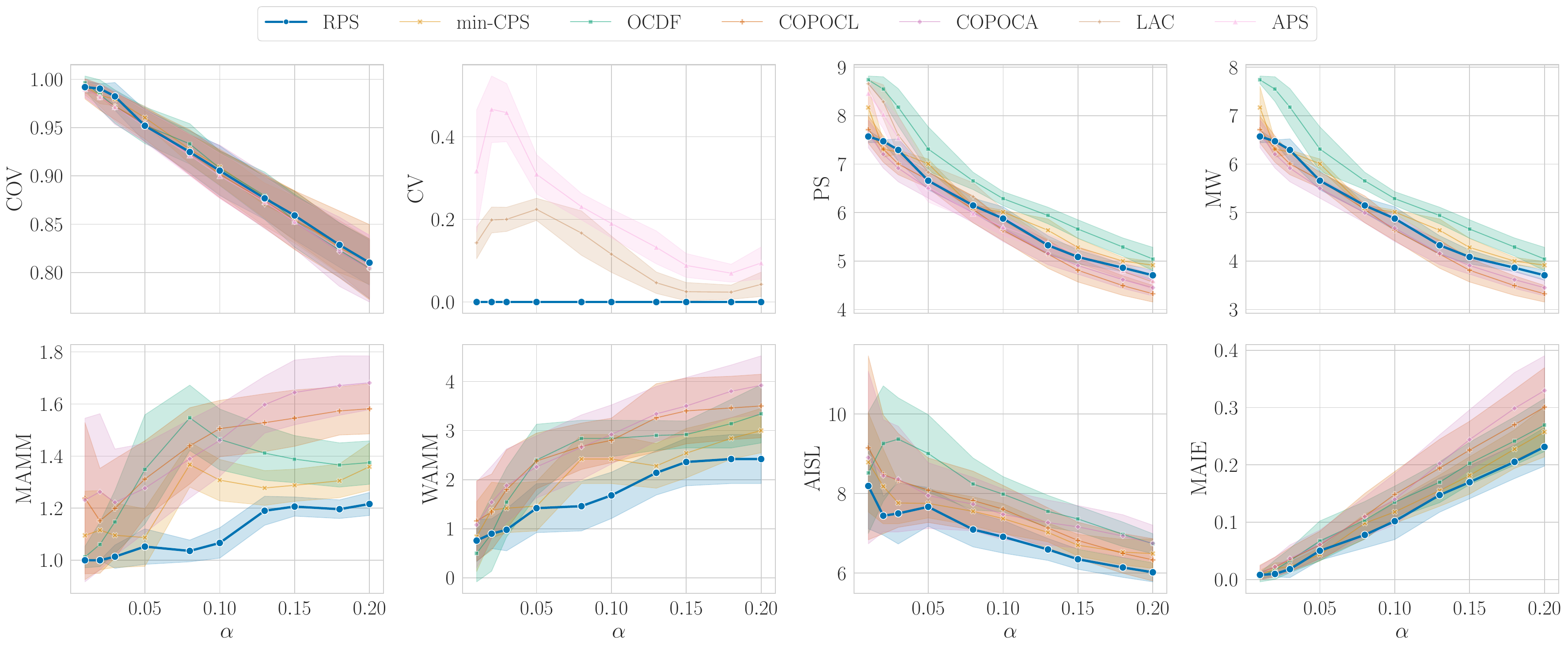}
    \caption{Comparison of prediction sets across methods on era dataset. Shaded regions indicate standard deviation.}
 \end{minipage}
\end{figure*}

\begin{figure*}
     \begin{minipage}[b]{\textwidth}
        \centering
\includegraphics[width=\linewidth]{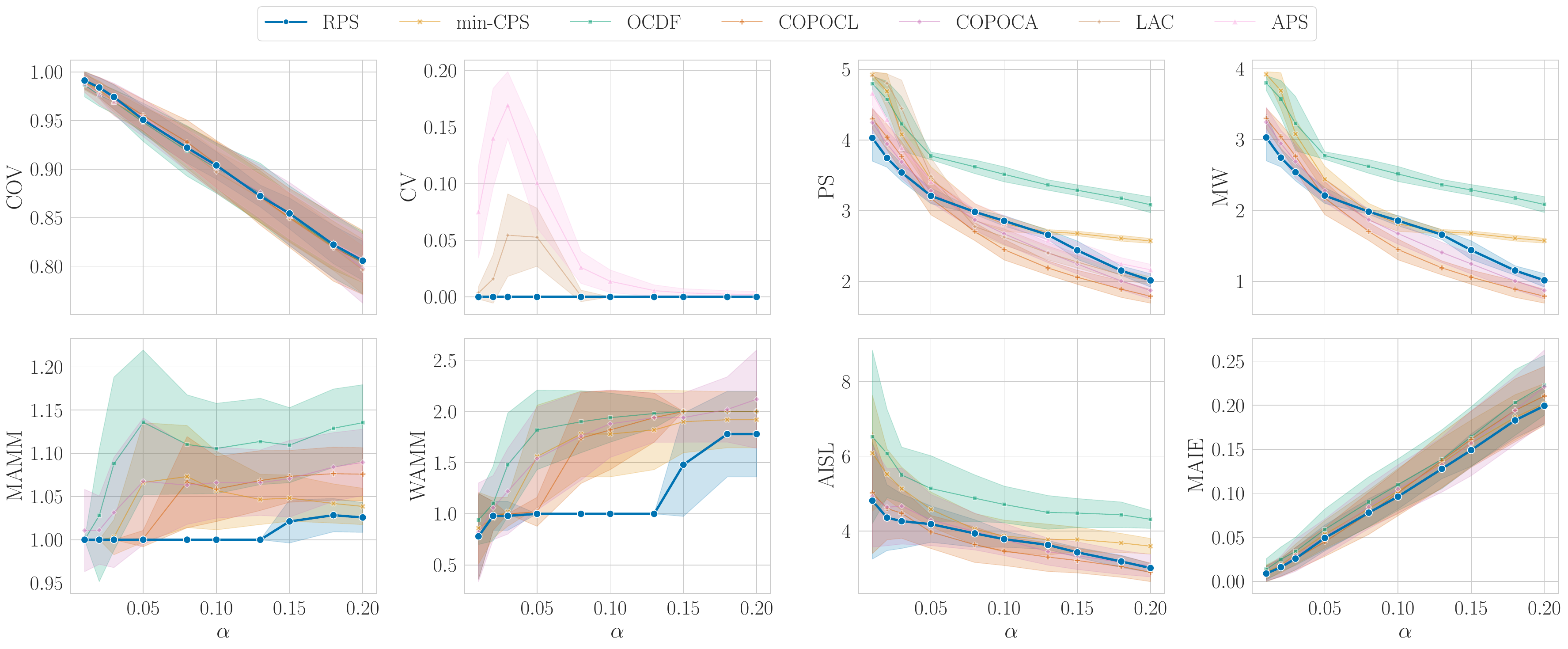}
    \caption{Comparison of prediction sets across methods on lev dataset. Shaded regions indicate standard deviation.}
 \end{minipage}
    \begin{minipage}[b]{0.19\textwidth}
        \centering
        \includegraphics[width=\linewidth]{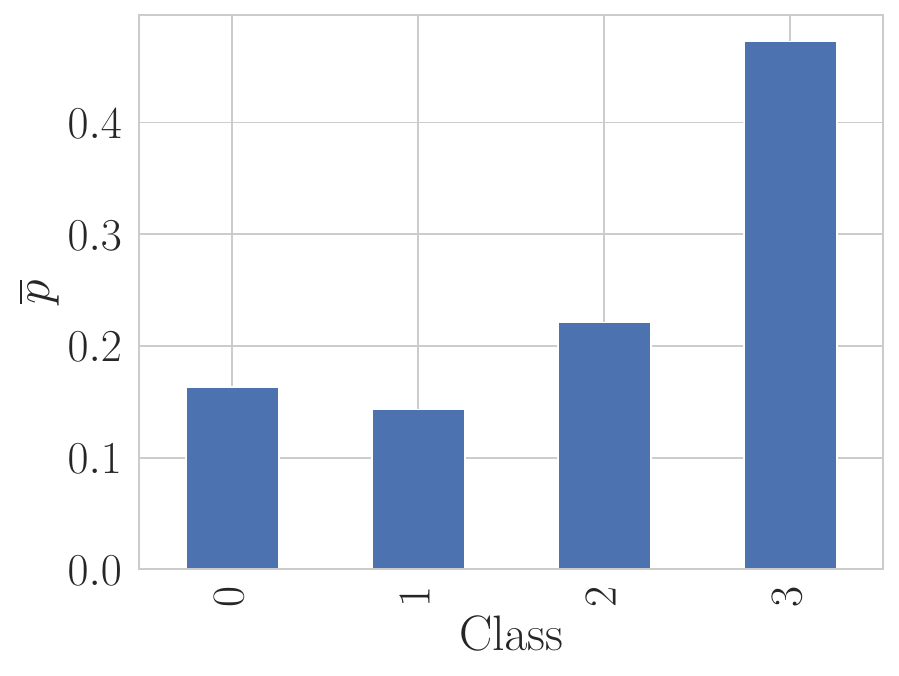}
        \caption*{(a) LESTSensors}
    \end{minipage}
     \begin{minipage}[b]{0.19\textwidth}
        \centering
        \includegraphics[width=\linewidth]{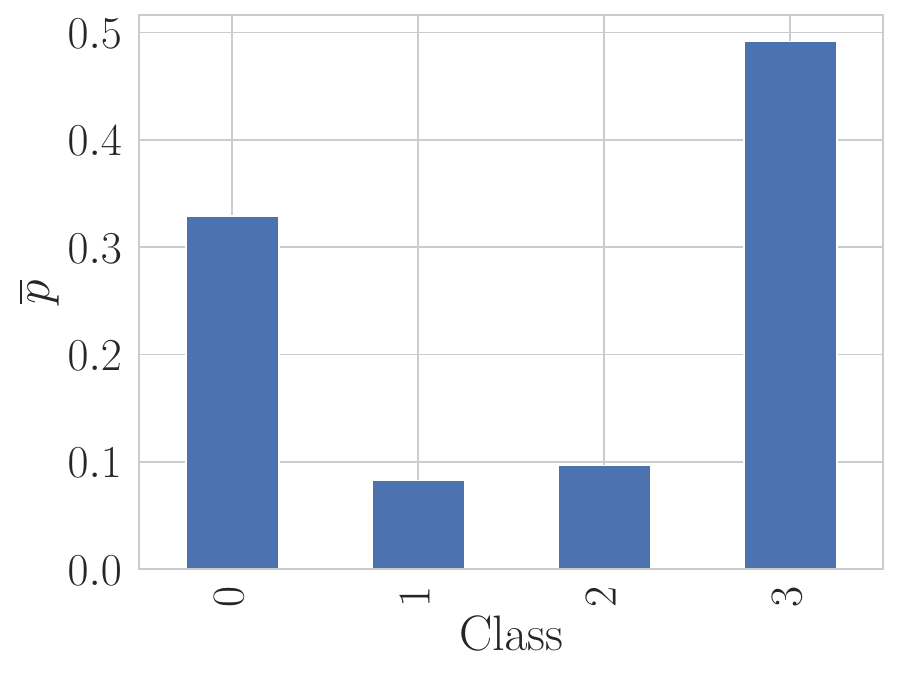}
        \caption*{(b) LEVXSensors}
    \end{minipage}
         \begin{minipage}[b]{0.19\textwidth}
        \centering
        \includegraphics[width=\linewidth]{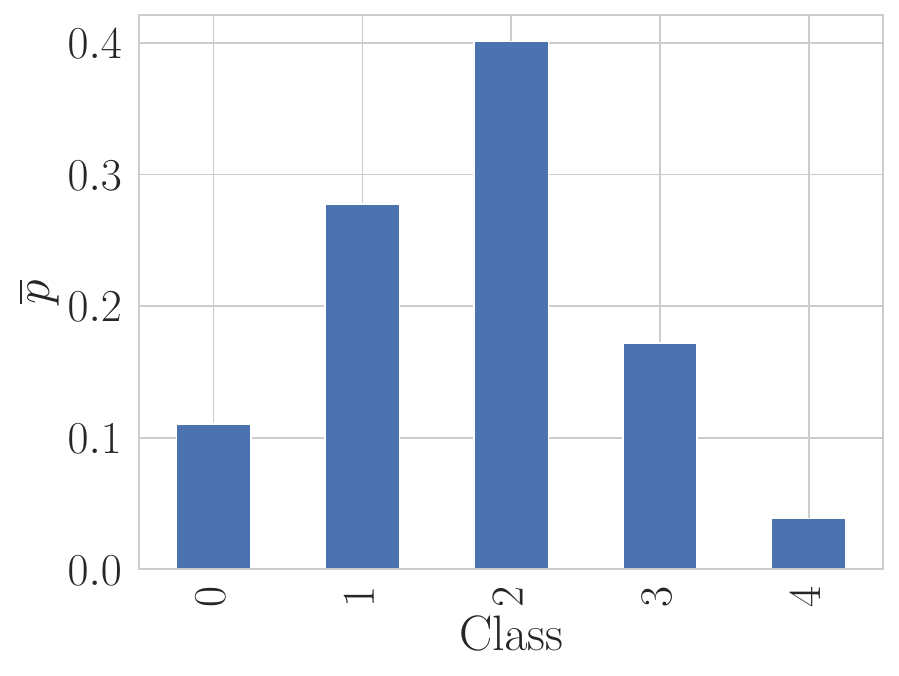}
        \caption*{(c) nhanes}
    \end{minipage}
    \begin{minipage}[b]{0.19\textwidth}
        \centering
        \includegraphics[width=\linewidth]{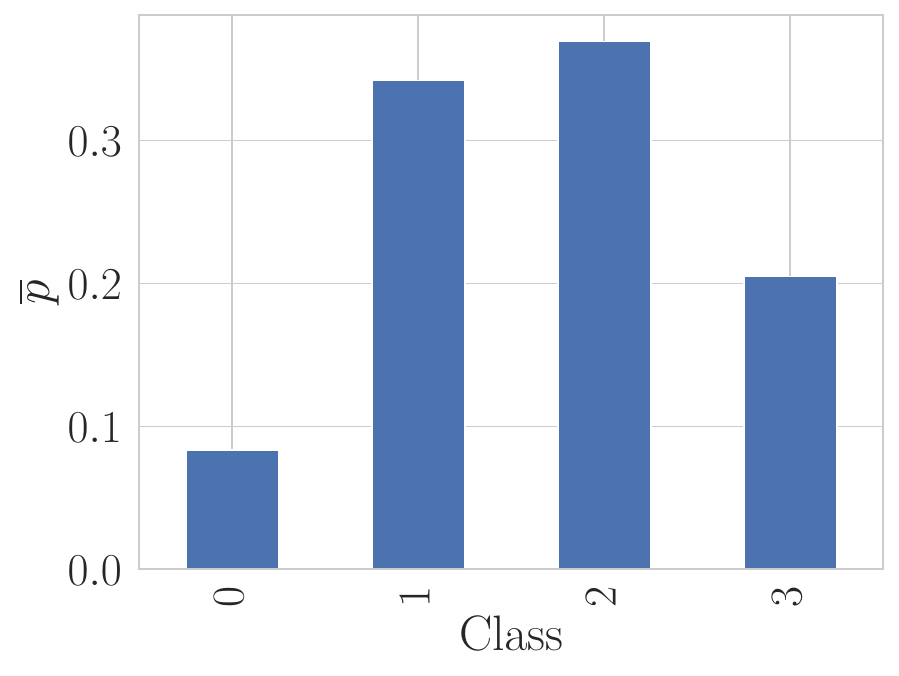}
        \caption*{(d) swd}
    \end{minipage}
                 \begin{minipage}[b]{0.19\textwidth}
        \centering
        \includegraphics[width=\linewidth]{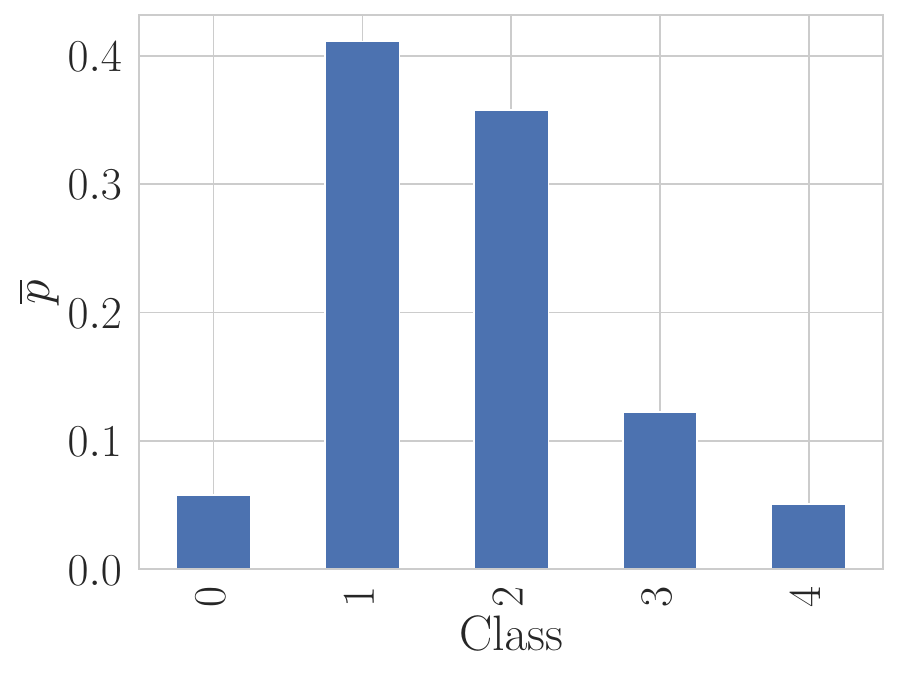}
        \caption*{(e) winequalityRed}
    \end{minipage}
        \begin{minipage}[b]{0.19\textwidth}
        \centering
        \includegraphics[width=\linewidth]{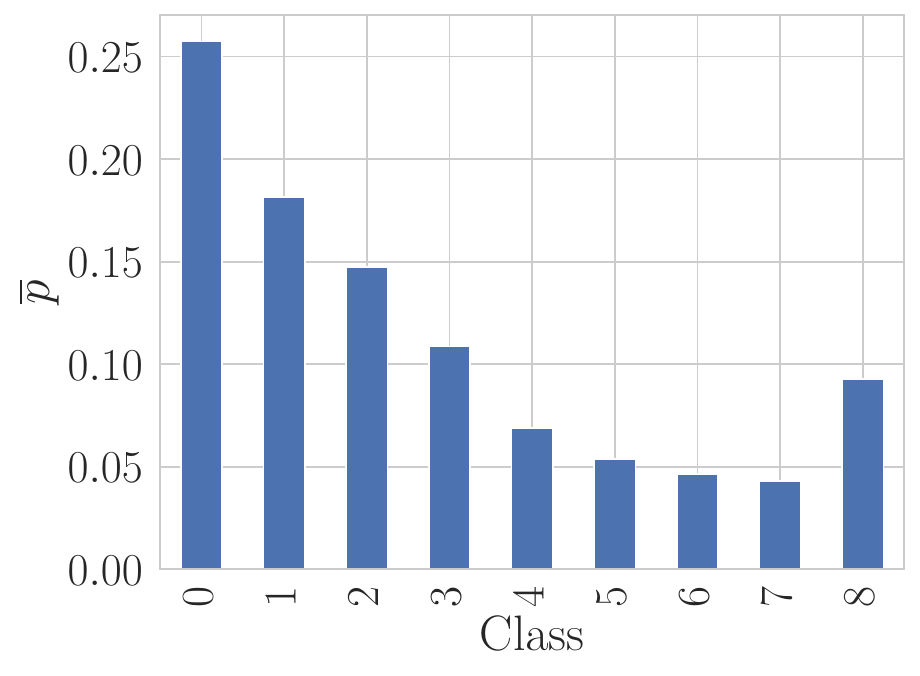}
        \caption*{(f) insurance}
    \end{minipage}
             \begin{minipage}[b]{0.19\textwidth}
        \centering
        \includegraphics[width=\linewidth]{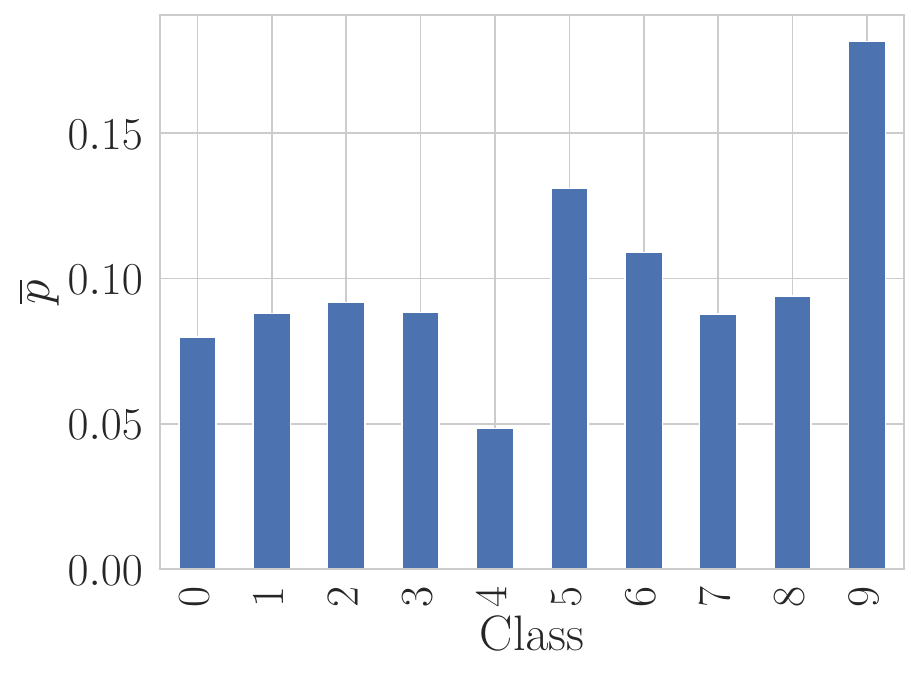}
        \caption*{(g) melbourneAirbnb }
    \end{minipage}
    \begin{minipage}[b]{0.19\textwidth}
        \centering
        \includegraphics[width=\linewidth]{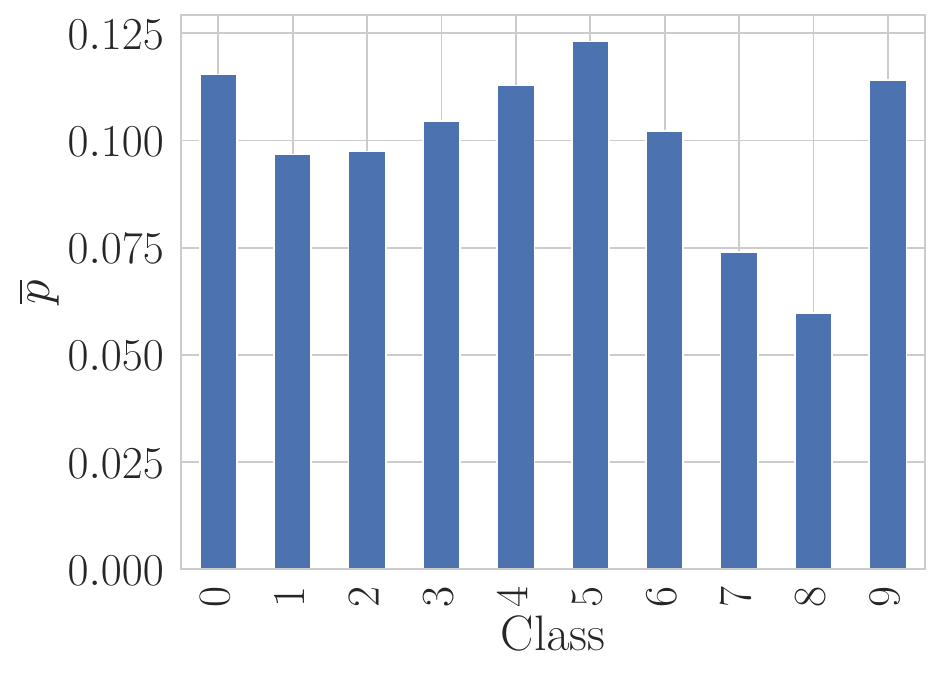}
        \caption*{(h) cancerDeathRate}
    \end{minipage}
                 \begin{minipage}[b]{0.19\textwidth}
        \centering
        \includegraphics[width=\linewidth]{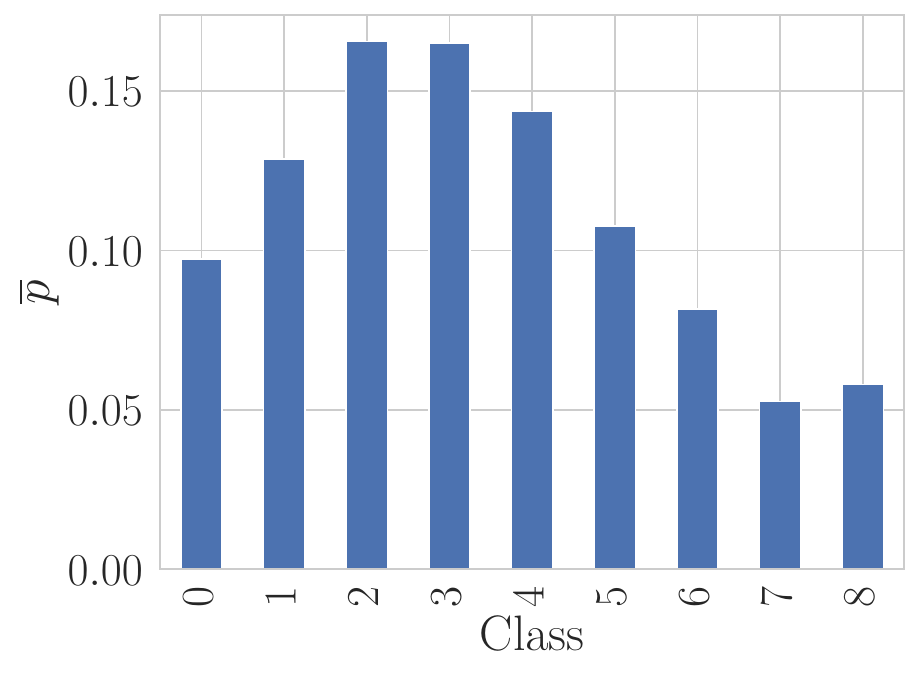}
        \caption*{(i) era }
    \end{minipage}
    \begin{minipage}[b]{0.19\textwidth}
        \centering
        \includegraphics[width=\linewidth]{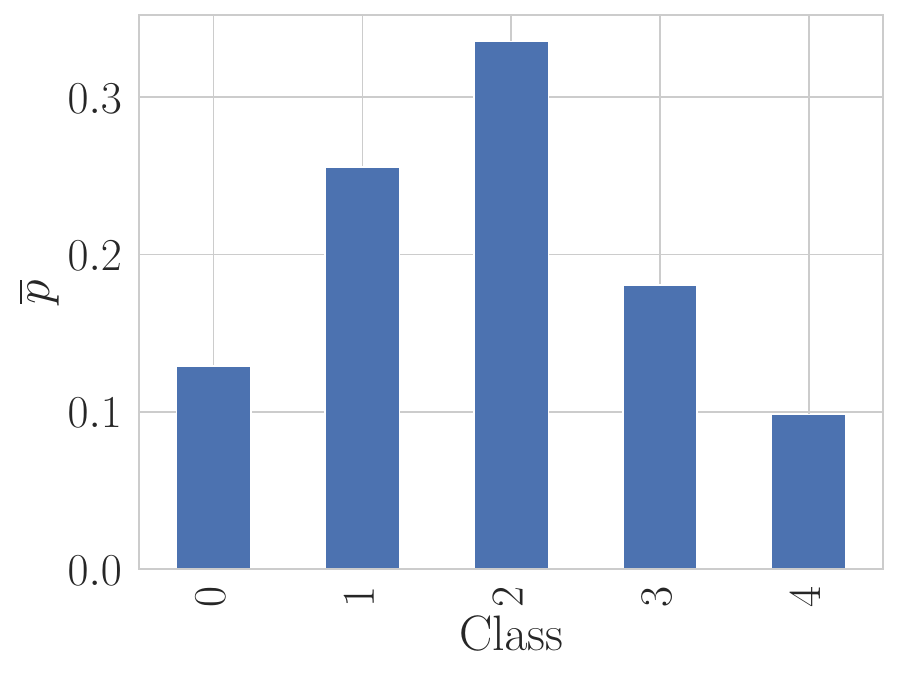}
        \caption*{(j) lev}
    \end{minipage}
     \caption{Mean predictive probabilities (MP) for the different datasets with the CE loss.}
     \label{fig:mps_tabular_datasets}
\end{figure*}

\section{Additional Experiments Using Gradient Boosted Trees}
\label{sec:gbt_experiments}
In this section, we compare the different nonconformity scores using gradient boosted trees (GBTs), a model class of high practical relevance for tabular data~\cite{shwartz2022tabular}.
We select LightGBM~\cite{DBLP:conf/nips/KeMFWCMYL17} as a representative GBT instantiation; other popular implementations such as CatBoost~\cite{DBLP:conf/nips/ProkhorenkovaGV18} or XGBoost~\cite{DBLP:conf/kdd/ChenG16} typically yield similar results.
Among the nonconformity measures, we exclude COPOCA and COPOCL from this comparison, as COPOC~\cite{DBLP:conf/nips/DeyMK23} relies on a specific neural network architecture that is not trivial to replicate with GBTs.
For our experiments, we use six tabular datasets from the TOC-UCO repository~\cite{ayllon2025toc} (Table~\ref{tab:lgbm_datasets}).
The results are reported in Table~\ref{tab:lgbm_results} and, broken down per dataset across all metrics, in the figures below.
They corroborate our earlier findings: RPS remains competitive with min-CPS~\cite{zhang2025provably} and strikes a favorable balance between reducing ordinal miscoverage and maintaining small prediction set sizes.


\begin{table}[ht]
\small
\centering
\caption{Summary of tabular ordinal datasets used in our LightGBM experiments~\cite{ayllon2025toc}.}
\label{tab:lgbm_datasets}
\begin{tabular}{lcccccc}
\toprule
Dataset & \#Samples & \#Features & \#Classes & Class distribution & IR \\
\midrule
heartDisease   & 294   & 13 & 4 & (0.64, 0.13, 0.09, 0.15) & 1.48 \\
mammoexp       & 412   & 5  & 3 & (0.57, 0.25, 0.18)       & 0.92 \\
support        & 730   & 19 & 4 & (0.39, 0.13, 0.08, 0.40) & 1.34 \\
winequalityRed & 1,599 & 11 & 5 & (0.04, 0.43, 0.40, 0.12, 0.01) & 4.88 \\
nhanes         & 5,223 & 30 & 5 & (0.11, 0.28, 0.40, 0.18, 0.04) & 1.72 \\
LEVXSensors    & 5,112 & 6  & 4 & (0.33, 0.08, 0.10, 0.48) & 1.45 \\
\bottomrule
\end{tabular}
\end{table}

\begin{table*}[!h]
\tiny
\centering
\caption{Comparison of the different non-conformity measures using LightGBM over tabular datasets at $\alpha=0.02$, $\alpha=0.05$ and $\alpha=0.1$.}
\label{tab:lgbm_results}
\begin{tabular}{lllcccccccc}
\toprule
Dataset & $\alpha$ & Method & {COV} & {PS} ($\downarrow$) & {CV} ($\downarrow$) & {MW} ($\downarrow$) & {MAMM} ($\downarrow$) & {WAMM} ($\downarrow$) & {MAIE} ($\downarrow$) & {AISL} ($\downarrow$) \\
\midrule
 \multirow{15}{*}{heartDisease} & \multirow{5}{*}{0.02} & APS     & 0.985 $ \pm $ 0.019 & 3.625 $ \pm $ 0.388 & 0.072 $ \pm $ 0.071 & - & - & - & - & - \\
                       &         & LAC     & 0.984 $ \pm $ 0.023 & 3.634 $ \pm $ 0.427 & 0.042 $ \pm $ 0.049 & - & - & - & - & - \\
                       &         & OCDF    & 0.981 $ \pm $ 0.023 & 3.872 $ \pm $ 0.109 & 0.000 $ \pm $ 0.000 & 2.872 $ \pm $ 0.109 & 0.940 $ \pm $ 0.874 & 1.120 $ \pm $ 1.003 & 0.028 $ \pm $ 0.029 & 5.686 $ \pm $ 2.836 \\
                       &         & RPS     & 0.985 $ \pm $ 0.019 & 3.718 $ \pm $ 0.240 & 0.000 $ \pm $ 0.000 & 2.718 $ \pm $ 0.240 & \textbf{0.480} $ \pm $ 0.505 & \textbf{0.480} $ \pm $ 0.505 & \textbf{0.015} $ \pm $ 0.019 & \textbf{4.176} $ \pm $ 1.687 \\
                       &         & min-CPS & 0.984 $ \pm $ 0.022 & \textbf{3.592} $ \pm $ 0.489 & 0.000 $ \pm $ 0.000 & \textbf{2.592} $ \pm $ 0.489 & 1.028 $ \pm $ 1.097 & 1.220 $ \pm $ 1.329 & 0.033 $ \pm $ 0.043 & 5.880 $ \pm $ 3.808 \\
\cmidrule(lr){2-11}
                       & \multirow{5}{*}{0.05} & APS     & 0.967 $ \pm $ 0.031 & 3.033 $ \pm $ 0.407 & 0.149 $ \pm $ 0.056 & - & - & - & - & - \\
                       &         & LAC     & 0.973 $ \pm $ 0.026 & 3.081 $ \pm $ 0.402 & 0.096 $ \pm $ 0.033 & - & - & - & - & - \\
                       &         & OCDF    & 0.966 $ \pm $ 0.034 & 3.774 $ \pm $ 0.109 & 0.000 $ \pm $ 0.000 & 2.774 $ \pm $ 0.109 & 1.101 $ \pm $ 0.696 & 1.320 $ \pm $ 0.844 & 0.048 $ \pm $ 0.049 & 4.686 $ \pm $ 1.871 \\
                       &         & RPS     & 0.967 $ \pm $ 0.027 & 3.363 $ \pm $ 0.307 & 0.000 $ \pm $ 0.000 & 2.363 $ \pm $ 0.307 & \textbf{0.800} $ \pm $ 0.437 & \textbf{0.840} $ \pm $ 0.510 & \textbf{0.035} $ \pm $ 0.031 & \textbf{3.760} $ \pm $ 0.946 \\
                       &         & min-CPS & 0.957 $ \pm $ 0.030 & \textbf{2.693} $ \pm $ 0.204 & 0.000 $ \pm $ 0.000 & \textbf{1.693} $ \pm $ 0.204 & 1.633 $ \pm $ 0.796 & 2.260 $ \pm $ 1.065 & 0.075 $ \pm $ 0.050 & 4.703 $ \pm $ 1.826 \\
\cmidrule(lr){2-11}
                       & \multirow{5}{*}{0.1} & APS     & 0.926 $ \pm $ 0.048 & 2.275 $ \pm $ 0.271 & 0.196 $ \pm $ 0.040 & - & - & - & - & - \\
                       &         & LAC     & 0.918 $ \pm $ 0.051 & 2.226 $ \pm $ 0.275 & 0.152 $ \pm $ 0.036 & - & - & - & - & - \\
                       &         & OCDF    & 0.905 $ \pm $ 0.047 & 3.576 $ \pm $ 0.135 & 0.000 $ \pm $ 0.000 & 2.576 $ \pm $ 0.135 & 1.411 $ \pm $ 0.275 & 2.080 $ \pm $ 0.601 & 0.139 $ \pm $ 0.072 & 5.356 $ \pm $ 1.303 \\
                       &         & RPS     & 0.910 $ \pm $ 0.038 & 2.418 $ \pm $ 0.239 & 0.000 $ \pm $ 0.000 & 1.418 $ \pm $ 0.239 & \textbf{1.345} $ \pm $ 0.172 & \textbf{1.920} $ \pm $ 0.340 & \textbf{0.124} $ \pm $ 0.058 & \textbf{3.893} $ \pm $ 0.955 \\
                       &         & min-CPS & 0.905 $ \pm $ 0.052 & \textbf{2.068} $ \pm $ 0.299 & 0.000 $ \pm $ 0.000 & \textbf{1.068} $ \pm $ 0.299 & 1.542 $ \pm $ 0.513 & 2.640 $ \pm $ 0.802 & 0.149 $ \pm $ 0.073 & 4.051 $ \pm $ 1.195 \\
\midrule
 \multirow{15}{*}{mammoexp} & \multirow{5}{*}{0.02} & APS     & 0.987 $ \pm $ 0.016 & 2.914 $ \pm $ 0.072 & 0.049 $ \pm $ 0.042 & - & - & - & - & - \\
                       &         & LAC     & 0.988 $ \pm $ 0.016 & 2.920 $ \pm $ 0.078 & 0.050 $ \pm $ 0.035 & - & - & - & - & - \\
                       &         & OCDF    & 0.986 $ \pm $ 0.017 & 2.909 $ \pm $ 0.053 & 0.000 $ \pm $ 0.000 & 1.909 $ \pm $ 0.053 & 0.584 $ \pm $ 0.503 & 0.600 $ \pm $ 0.535 & 0.014 $ \pm $ 0.018 & 3.331 $ \pm $ 1.753 \\
                       &         & RPS     & 0.990 $ \pm $ 0.013 & 2.940 $ \pm $ 0.068 & 0.000 $ \pm $ 0.000 & 1.940 $ \pm $ 0.068 & \textbf{0.500} $ \pm $ 0.505 & \textbf{0.500} $ \pm $ 0.505 & \textbf{0.010} $ \pm $ 0.013 & \textbf{2.976} $ \pm $ 1.227 \\
                       &         & min-CPS & 0.990 $ \pm $ 0.012 & \textbf{2.893} $ \pm $ 0.155 & 0.000 $ \pm $ 0.000 & \textbf{1.893} $ \pm $ 0.155 & 0.670 $ \pm $ 0.733 & 0.740 $ \pm $ 0.828 & 0.015 $ \pm $ 0.023 & 3.436 $ \pm $ 2.122 \\
\cmidrule(lr){2-11}
                       & \multirow{5}{*}{0.05} & APS     & 0.948 $ \pm $ 0.032 & 2.667 $ \pm $ 0.143 & 0.147 $ \pm $ 0.045 & - & - & - & - & - \\
                       &         & LAC     & 0.958 $ \pm $ 0.031 & 2.617 $ \pm $ 0.190 & 0.107 $ \pm $ 0.035 & - & - & - & - & - \\
                       &         & OCDF    & 0.961 $ \pm $ 0.029 & 2.801 $ \pm $ 0.069 & 0.000 $ \pm $ 0.000 & 1.801 $ \pm $ 0.069 & 1.120 $ \pm $ 0.347 & 1.480 $ \pm $ 0.614 & 0.046 $ \pm $ 0.032 & 3.633 $ \pm $ 1.206 \\
                       &         & RPS     & 0.957 $ \pm $ 0.028 & 2.739 $ \pm $ 0.092 & 0.000 $ \pm $ 0.000 & 1.739 $ \pm $ 0.092 & \textbf{0.962} $ \pm $ 0.199 & \textbf{0.980} $ \pm $ 0.247 & \textbf{0.044} $ \pm $ 0.028 & \textbf{3.483} $ \pm $ 1.058 \\
                       &         & min-CPS & 0.957 $ \pm $ 0.031 & \textbf{2.551} $ \pm $ 0.081 & 0.000 $ \pm $ 0.000 & \textbf{1.551} $ \pm $ 0.081 & 1.278 $ \pm $ 0.535 & 1.740 $ \pm $ 0.664 & 0.059 $ \pm $ 0.039 & 3.903 $ \pm $ 1.507 \\
\cmidrule(lr){2-11}
                       & \multirow{5}{*}{0.1} & APS     & 0.883 $ \pm $ 0.047 & 2.308 $ \pm $ 0.154 & 0.194 $ \pm $ 0.038 & - & - & - & - & - \\
                       &         & LAC     & 0.905 $ \pm $ 0.047 & \textbf{2.276} $ \pm $ 0.136 & 0.149 $ \pm $ 0.033 & - & - & - & - & - \\
                       &         & OCDF    & 0.904 $ \pm $ 0.046 & 2.651 $ \pm $ 0.081 & 0.000 $ \pm $ 0.000 & 1.651 $ \pm $ 0.081 & \textbf{1.062} $ \pm $ 0.179 & \textbf{1.560} $ \pm $ 0.541 & \textbf{0.103} $ \pm $ 0.048 & 3.714 $ \pm $ 0.883 \\
                       &         & RPS     & 0.901 $ \pm $ 0.031 & 2.323 $ \pm $ 0.173 & 0.000 $ \pm $ 0.000 & \textbf{1.323} $ \pm $ 0.173 & 1.095 $ \pm $ 0.095 & 1.600 $ \pm $ 0.495 & 0.110 $ \pm $ 0.039 & 3.516 $ \pm $ 0.633 \\
                       &         & min-CPS & 0.910 $ \pm $ 0.040 & 2.389 $ \pm $ 0.145 & 0.000 $ \pm $ 0.000 & 1.389 $ \pm $ 0.145 & 1.198 $ \pm $ 0.152 & 1.860 $ \pm $ 0.351 & 0.106 $ \pm $ 0.045 & \textbf{3.510} $ \pm $ 0.770 \\
\midrule
 \multirow{15}{*}{support} & \multirow{5}{*}{0.02} & APS     & 0.984 $ \pm $ 0.016 & \textbf{2.288} $ \pm $ 0.335 & 0.130 $ \pm $ 0.066 & - & - & - & - & - \\
                       &         & LAC     & 0.983 $ \pm $ 0.016 & 2.292 $ \pm $ 0.355 & 0.105 $ \pm $ 0.060 & - & - & - & - & - \\
                       &         & OCDF    & 0.987 $ \pm $ 0.013 & 2.758 $ \pm $ 0.046 & 0.000 $ \pm $ 0.000 & 1.758 $ \pm $ 0.046 & 1.423 $ \pm $ 0.777 & 1.440 $ \pm $ 0.787 & 0.024 $ \pm $ 0.025 & 4.197 $ \pm $ 2.443 \\
                       &         & RPS     & 0.988 $ \pm $ 0.011 & 2.628 $ \pm $ 0.175 & 0.000 $ \pm $ 0.000 & 1.628 $ \pm $ 0.175 & \textbf{0.760} $ \pm $ 0.431 & \textbf{0.760} $ \pm $ 0.431 & \textbf{0.012} $ \pm $ 0.011 & \textbf{2.875} $ \pm $ 0.958 \\
                       &         & min-CPS & 0.985 $ \pm $ 0.014 & 2.383 $ \pm $ 0.510 & 0.000 $ \pm $ 0.000 & \textbf{1.383} $ \pm $ 0.510 & 1.239 $ \pm $ 0.746 & 1.520 $ \pm $ 0.863 & 0.023 $ \pm $ 0.019 & 3.684 $ \pm $ 1.529 \\
\cmidrule(lr){2-11}
                       & \multirow{5}{*}{0.05} & APS     & 0.953 $ \pm $ 0.022 & 1.656 $ \pm $ 0.128 & 0.060 $ \pm $ 0.017 & - & - & - & - & - \\
                       &         & LAC     & 0.954 $ \pm $ 0.023 & \textbf{1.651} $ \pm $ 0.124 & 0.051 $ \pm $ 0.016 & - & - & - & - & - \\
                       &         & OCDF    & 0.952 $ \pm $ 0.024 & 2.627 $ \pm $ 0.066 & 0.000 $ \pm $ 0.000 & 1.627 $ \pm $ 0.066 & 1.654 $ \pm $ 0.168 & 2.000 $ \pm $ 0.000 & 0.078 $ \pm $ 0.037 & 4.739 $ \pm $ 1.409 \\
                       &         & RPS     & 0.956 $ \pm $ 0.022 & 2.327 $ \pm $ 0.077 & 0.000 $ \pm $ 0.000 & 1.327 $ \pm $ 0.077 & \textbf{1.000} $ \pm $ 0.000 & \textbf{1.000} $ \pm $ 0.000 & \textbf{0.044} $ \pm $ 0.022 & \textbf{3.080} $ \pm $ 0.803 \\
                       &         & min-CPS & 0.955 $ \pm $ 0.021 & 1.661 $ \pm $ 0.102 & 0.000 $ \pm $ 0.000 & \textbf{0.661} $ \pm $ 0.102 & 1.363 $ \pm $ 0.163 & 1.960 $ \pm $ 0.198 & 0.061 $ \pm $ 0.030 & 3.088 $ \pm $ 1.100 \\
\cmidrule(lr){2-11}
                       & \multirow{5}{*}{0.1} & APS     & 0.904 $ \pm $ 0.035 & 1.372 $ \pm $ 0.094 & 0.044 $ \pm $ 0.017 & - & - & - & - & - \\
                       &         & LAC     & 0.907 $ \pm $ 0.033 & \textbf{1.346} $ \pm $ 0.117 & 0.043 $ \pm $ 0.015 & - & - & - & - & - \\
                       &         & OCDF    & 0.902 $ \pm $ 0.030 & 2.370 $ \pm $ 0.167 & 0.000 $ \pm $ 0.000 & 1.370 $ \pm $ 0.167 & 1.507 $ \pm $ 0.113 & 2.000 $ \pm $ 0.000 & 0.148 $ \pm $ 0.044 & 4.323 $ \pm $ 0.717 \\
                       &         & RPS     & 0.909 $ \pm $ 0.033 & 1.676 $ \pm $ 0.346 & 0.000 $ \pm $ 0.000 & 0.676 $ \pm $ 0.346 & \textbf{1.120} $ \pm $ 0.126 & \textbf{1.620} $ \pm $ 0.490 & \textbf{0.105} $ \pm $ 0.050 & \textbf{2.786} $ \pm $ 0.709 \\
                       &         & min-CPS & 0.908 $ \pm $ 0.031 & 1.356 $ \pm $ 0.102 & 0.000 $ \pm $ 0.000 & \textbf{0.356} $ \pm $ 0.102 & 1.416 $ \pm $ 0.109 & 2.000 $ \pm $ 0.000 & 0.131 $ \pm $ 0.047 & 2.978 $ \pm $ 0.845 \\
\midrule
 \multirow{15}{*}{winequalityRed} & \multirow{5}{*}{0.02} & APS     & 0.981 $ \pm $ 0.008 & 3.090 $ \pm $ 0.185 & 0.002 $ \pm $ 0.003 & - & - & - & - & - \\
                       &         & LAC     & 0.981 $ \pm $ 0.009 & 3.151 $ \pm $ 0.279 & 0.001 $ \pm $ 0.002 & - & - & - & - & - \\
                       &         & OCDF    & 0.981 $ \pm $ 0.012 & 4.297 $ \pm $ 0.147 & 0.000 $ \pm $ 0.000 & 3.297 $ \pm $ 0.147 & 1.022 $ \pm $ 0.043 & 1.220 $ \pm $ 0.418 & 0.020 $ \pm $ 0.013 & 5.303 $ \pm $ 1.141 \\
                       &         & RPS     & 0.981 $ \pm $ 0.009 & \textbf{2.820} $ \pm $ 0.127 & 0.000 $ \pm $ 0.000 & \textbf{1.820} $ \pm $ 0.127 & \textbf{1.000} $ \pm $ 0.000 & \textbf{1.000} $ \pm $ 0.000 & 0.019 $ \pm $ 0.009 & \textbf{3.720} $ \pm $ 0.815 \\
                       &         & min-CPS & 0.982 $ \pm $ 0.008 & 3.120 $ \pm $ 0.308 & 0.000 $ \pm $ 0.000 & 2.120 $ \pm $ 0.308 & 1.000 $ \pm $ 0.000 & 1.000 $ \pm $ 0.000 & \textbf{0.019} $ \pm $ 0.008 & 3.970 $ \pm $ 0.566 \\
\cmidrule(lr){2-11}
                       & \multirow{5}{*}{0.05} & APS     & 0.950 $ \pm $ 0.015 & 2.374 $ \pm $ 0.127 & 0.006 $ \pm $ 0.004 & - & - & - & - & - \\
                       &         & LAC     & 0.949 $ \pm $ 0.015 & 2.347 $ \pm $ 0.137 & 0.006 $ \pm $ 0.004 & - & - & - & - & - \\
                       &         & OCDF    & 0.948 $ \pm $ 0.015 & 3.951 $ \pm $ 0.061 & 0.000 $ \pm $ 0.000 & 2.951 $ \pm $ 0.061 & 1.027 $ \pm $ 0.032 & 1.460 $ \pm $ 0.503 & 0.054 $ \pm $ 0.016 & 5.104 $ \pm $ 0.571 \\
                       &         & RPS     & 0.949 $ \pm $ 0.014 & \textbf{2.198} $ \pm $ 0.062 & 0.000 $ \pm $ 0.000 & \textbf{1.198} $ \pm $ 0.062 & \textbf{1.000} $ \pm $ 0.000 & \textbf{1.000} $ \pm $ 0.000 & 0.051 $ \pm $ 0.014 & \textbf{3.233} $ \pm $ 0.511 \\
                       &         & min-CPS & 0.950 $ \pm $ 0.017 & 2.354 $ \pm $ 0.142 & 0.000 $ \pm $ 0.000 & 1.354 $ \pm $ 0.142 & 1.000 $ \pm $ 0.000 & 1.000 $ \pm $ 0.000 & \textbf{0.050} $ \pm $ 0.017 & 3.374 $ \pm $ 0.549 \\
\cmidrule(lr){2-11}
                       & \multirow{5}{*}{0.1} & APS     & 0.897 $ \pm $ 0.023 & 1.879 $ \pm $ 0.086 & 0.011 $ \pm $ 0.005 & - & - & - & - & - \\
                       &         & LAC     & 0.897 $ \pm $ 0.024 & \textbf{1.824} $ \pm $ 0.111 & 0.008 $ \pm $ 0.004 & - & - & - & - & - \\
                       &         & OCDF    & 0.897 $ \pm $ 0.024 & 3.664 $ \pm $ 0.087 & 0.000 $ \pm $ 0.000 & 2.664 $ \pm $ 0.087 & 1.018 $ \pm $ 0.019 & 1.540 $ \pm $ 0.503 & 0.105 $ \pm $ 0.026 & 4.766 $ \pm $ 0.439 \\
                       &         & RPS     & 0.896 $ \pm $ 0.022 & 1.831 $ \pm $ 0.091 & 0.000 $ \pm $ 0.000 & 0.831 $ \pm $ 0.091 & \textbf{1.007} $ \pm $ 0.014 & \textbf{1.220} $ \pm $ 0.418 & \textbf{0.104} $ \pm $ 0.023 & \textbf{2.921} $ \pm $ 0.376 \\
                       &         & min-CPS & 0.896 $ \pm $ 0.024 & 1.825 $ \pm $ 0.113 & 0.000 $ \pm $ 0.000 & \textbf{0.825} $ \pm $ 0.113 & 1.016 $ \pm $ 0.019 & 1.480 $ \pm $ 0.505 & 0.106 $ \pm $ 0.026 & 2.951 $ \pm $ 0.411 \\
\midrule
 \multirow{15}{*}{nhanes} & \multirow{5}{*}{0.02} & APS     & 0.982 $ \pm $ 0.006 & 4.143 $ \pm $ 0.066 & 0.010 $ \pm $ 0.003 & - & - & - & - & - \\
                       &         & LAC     & 0.980 $ \pm $ 0.006 & \textbf{4.077} $ \pm $ 0.073 & 0.008 $ \pm $ 0.003 & - & - & - & - & - \\
                       &         & OCDF    & 0.980 $ \pm $ 0.005 & 4.705 $ \pm $ 0.029 & 0.000 $ \pm $ 0.000 & 3.705 $ \pm $ 0.029 & 1.090 $ \pm $ 0.049 & 1.880 $ \pm $ 0.328 & 0.022 $ \pm $ 0.005 & 5.858 $ \pm $ 0.482 \\
                       &         & RPS     & 0.981 $ \pm $ 0.006 & 4.094 $ \pm $ 0.071 & 0.000 $ \pm $ 0.000 & 3.094 $ \pm $ 0.071 & \textbf{1.000} $ \pm $ 0.000 & \textbf{1.000} $ \pm $ 0.000 & \textbf{0.019} $ \pm $ 0.006 & \textbf{4.952} $ \pm $ 0.563 \\
                       &         & min-CPS & 0.981 $ \pm $ 0.006 & 4.091 $ \pm $ 0.065 & 0.000 $ \pm $ 0.000 & \textbf{3.091} $ \pm $ 0.065 & 1.026 $ \pm $ 0.026 & 1.540 $ \pm $ 0.503 & 0.020 $ \pm $ 0.006 & 5.087 $ \pm $ 0.566 \\
\cmidrule(lr){2-11}
                       & \multirow{5}{*}{0.05} & APS     & 0.950 $ \pm $ 0.010 & 3.628 $ \pm $ 0.073 & 0.034 $ \pm $ 0.005 & - & - & - & - & - \\
                       &         & LAC     & 0.950 $ \pm $ 0.008 & 3.540 $ \pm $ 0.058 & 0.025 $ \pm $ 0.003 & - & - & - & - & - \\
                       &         & OCDF    & 0.949 $ \pm $ 0.009 & 4.476 $ \pm $ 0.045 & 0.000 $ \pm $ 0.000 & 3.476 $ \pm $ 0.045 & 1.106 $ \pm $ 0.030 & 2.000 $ \pm $ 0.000 & 0.056 $ \pm $ 0.010 & 5.730 $ \pm $ 0.357 \\
                       &         & RPS     & 0.949 $ \pm $ 0.008 & \textbf{3.536} $ \pm $ 0.041 & 0.000 $ \pm $ 0.000 & \textbf{2.536} $ \pm $ 0.041 & \textbf{1.014} $ \pm $ 0.014 & \textbf{1.560} $ \pm $ 0.501 & \textbf{0.051} $ \pm $ 0.009 & \textbf{4.594} $ \pm $ 0.315 \\
                       &         & min-CPS & 0.950 $ \pm $ 0.008 & 3.552 $ \pm $ 0.039 & 0.000 $ \pm $ 0.000 & 2.552 $ \pm $ 0.039 & 1.030 $ \pm $ 0.015 & 1.920 $ \pm $ 0.274 & 0.052 $ \pm $ 0.008 & 4.622 $ \pm $ 0.299 \\
\cmidrule(lr){2-11}
                       & \multirow{5}{*}{0.1} & APS     & 0.899 $ \pm $ 0.012 & 3.129 $ \pm $ 0.056 & 0.050 $ \pm $ 0.005 & - & - & - & - & - \\
                       &         & LAC     & 0.903 $ \pm $ 0.012 & 3.091 $ \pm $ 0.056 & 0.029 $ \pm $ 0.003 & - & - & - & - & - \\
                       &         & OCDF    & 0.900 $ \pm $ 0.010 & 4.174 $ \pm $ 0.036 & 0.000 $ \pm $ 0.000 & 3.174 $ \pm $ 0.036 & 1.138 $ \pm $ 0.025 & 2.200 $ \pm $ 0.404 & 0.114 $ \pm $ 0.012 & 5.450 $ \pm $ 0.209 \\
                       &         & RPS     & 0.900 $ \pm $ 0.011 & \textbf{3.001} $ \pm $ 0.047 & 0.000 $ \pm $ 0.000 & \textbf{2.001} $ \pm $ 0.047 & \textbf{1.048} $ \pm $ 0.017 & \textbf{2.000} $ \pm $ 0.000 & \textbf{0.104} $ \pm $ 0.012 & \textbf{4.087} $ \pm $ 0.201 \\
                       &         & min-CPS & 0.901 $ \pm $ 0.013 & 3.094 $ \pm $ 0.061 & 0.000 $ \pm $ 0.000 & 2.094 $ \pm $ 0.061 & 1.052 $ \pm $ 0.017 & 2.000 $ \pm $ 0.000 & 0.104 $ \pm $ 0.014 & 4.183 $ \pm $ 0.219 \\
\midrule
 \multirow{15}{*}{LEVXSensors} & \multirow{5}{*}{0.02} & APS     & 0.980 $ \pm $ 0.006 & 3.552 $ \pm $ 0.066 & 0.296 $ \pm $ 0.037 & - & - & - & - & - \\
                       &         & LAC     & 0.981 $ \pm $ 0.007 & 3.464 $ \pm $ 0.069 & 0.252 $ \pm $ 0.029 & - & - & - & - & - \\
                       &         & OCDF    & 0.980 $ \pm $ 0.006 & \textbf{3.411} $ \pm $ 0.055 & 0.000 $ \pm $ 0.000 & \textbf{2.411} $ \pm $ 0.055 & 1.800 $ \pm $ 0.121 & 3.000 $ \pm $ 0.000 & 0.036 $ \pm $ 0.012 & 6.039 $ \pm $ 1.179 \\
                       &         & RPS     & 0.980 $ \pm $ 0.006 & 3.734 $ \pm $ 0.021 & 0.000 $ \pm $ 0.000 & 2.734 $ \pm $ 0.021 & \textbf{1.000} $ \pm $ 0.000 & \textbf{1.000} $ \pm $ 0.000 & \textbf{0.020} $ \pm $ 0.006 & \textbf{4.739} $ \pm $ 0.571 \\
                       &         & min-CPS & 0.981 $ \pm $ 0.005 & 3.482 $ \pm $ 0.052 & 0.000 $ \pm $ 0.000 & 2.482 $ \pm $ 0.052 & 1.520 $ \pm $ 0.153 & 2.960 $ \pm $ 0.198 & 0.029 $ \pm $ 0.010 & 5.356 $ \pm $ 0.927 \\
\cmidrule(lr){2-11}
                       & \multirow{5}{*}{0.05} & APS     & 0.948 $ \pm $ 0.010 & 3.107 $ \pm $ 0.073 & 0.501 $ \pm $ 0.028 & - & - & - & - & - \\
                       &         & LAC     & 0.949 $ \pm $ 0.010 & \textbf{3.001} $ \pm $ 0.070 & 0.426 $ \pm $ 0.026 & - & - & - & - & - \\
                       &         & OCDF    & 0.948 $ \pm $ 0.010 & 3.149 $ \pm $ 0.057 & 0.000 $ \pm $ 0.000 & \textbf{2.149} $ \pm $ 0.057 & 1.970 $ \pm $ 0.089 & 3.000 $ \pm $ 0.000 & 0.102 $ \pm $ 0.022 & 6.242 $ \pm $ 0.824 \\
                       &         & RPS     & 0.948 $ \pm $ 0.010 & 3.552 $ \pm $ 0.060 & 0.000 $ \pm $ 0.000 & 2.552 $ \pm $ 0.060 & \textbf{1.008} $ \pm $ 0.013 & \textbf{1.320} $ \pm $ 0.471 & \textbf{0.052} $ \pm $ 0.010 & \textbf{4.643} $ \pm $ 0.362 \\
                       &         & min-CPS & 0.952 $ \pm $ 0.008 & 3.208 $ \pm $ 0.035 & 0.000 $ \pm $ 0.000 & 2.208 $ \pm $ 0.035 & 1.585 $ \pm $ 0.086 & 3.000 $ \pm $ 0.000 & 0.076 $ \pm $ 0.015 & 5.265 $ \pm $ 0.557 \\
\cmidrule(lr){2-11}
                       & \multirow{5}{*}{0.1} & APS     & 0.900 $ \pm $ 0.012 & 2.590 $ \pm $ 0.059 & 0.638 $ \pm $ 0.015 & - & - & - & - & - \\
                       &         & LAC     & 0.898 $ \pm $ 0.012 & \textbf{2.498} $ \pm $ 0.067 & 0.562 $ \pm $ 0.017 & - & - & - & - & - \\
                       &         & OCDF    & 0.897 $ \pm $ 0.012 & 2.805 $ \pm $ 0.049 & 0.000 $ \pm $ 0.000 & \textbf{1.805} $ \pm $ 0.049 & 2.002 $ \pm $ 0.049 & 3.000 $ \pm $ 0.000 & 0.205 $ \pm $ 0.023 & 5.907 $ \pm $ 0.422 \\
                       &         & RPS     & 0.900 $ \pm $ 0.014 & 3.114 $ \pm $ 0.038 & 0.000 $ \pm $ 0.000 & 2.114 $ \pm $ 0.038 & \textbf{1.224} $ \pm $ 0.035 & \textbf{2.000} $ \pm $ 0.000 & \textbf{0.123} $ \pm $ 0.019 & \textbf{4.579} $ \pm $ 0.346 \\
                       &         & min-CPS & 0.899 $ \pm $ 0.013 & 2.839 $ \pm $ 0.053 & 0.000 $ \pm $ 0.000 & 1.839 $ \pm $ 0.053 & 1.738 $ \pm $ 0.058 & 3.000 $ \pm $ 0.000 & 0.175 $ \pm $ 0.025 & 5.343 $ \pm $ 0.463 \\
\bottomrule
\end{tabular}
\end{table*}

\begin{figure*}
     \begin{minipage}[!htb]{\textwidth}
        \centering
\includegraphics[width=\linewidth]{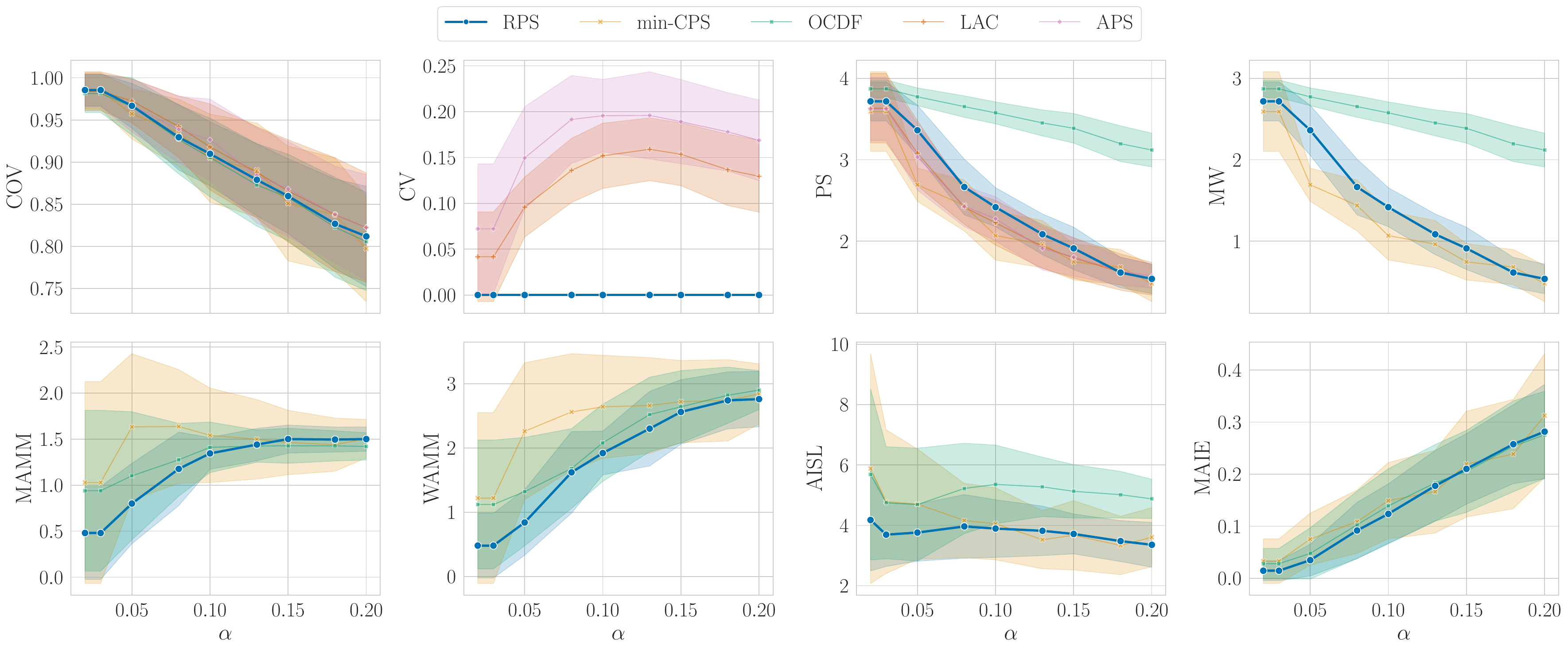}
    \caption{Comparison of prediction sets across methods on heartDisease dataset using LightGBM. Shaded regions indicate standard deviation.}
 \end{minipage}
      \begin{minipage}[!htb]{\textwidth}
        \centering
\includegraphics[width=\linewidth]{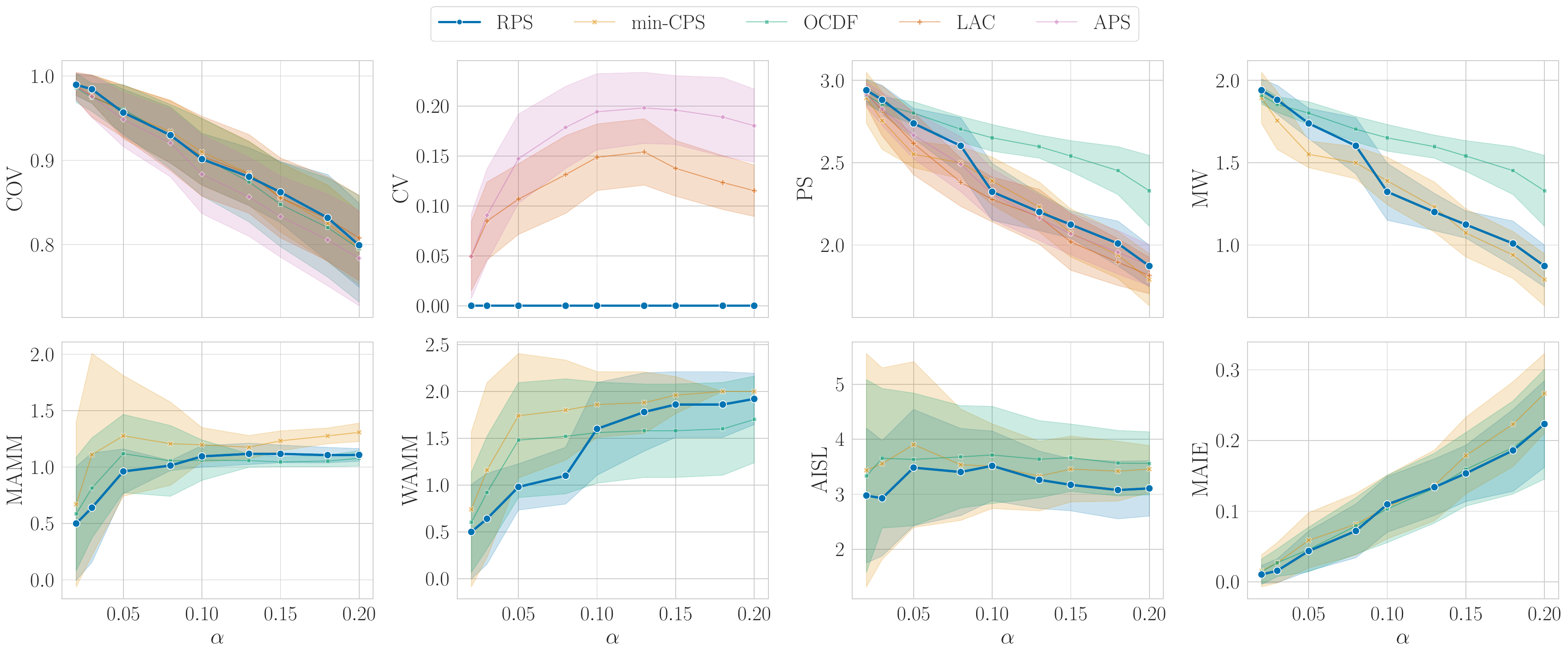}
    \caption{Comparison of prediction sets across methods on mammoexp dataset using LightGBM. Shaded regions indicate standard deviation.}
 \end{minipage}
       \begin{minipage}[!htb]{\textwidth}
        \centering
\includegraphics[width=\linewidth]{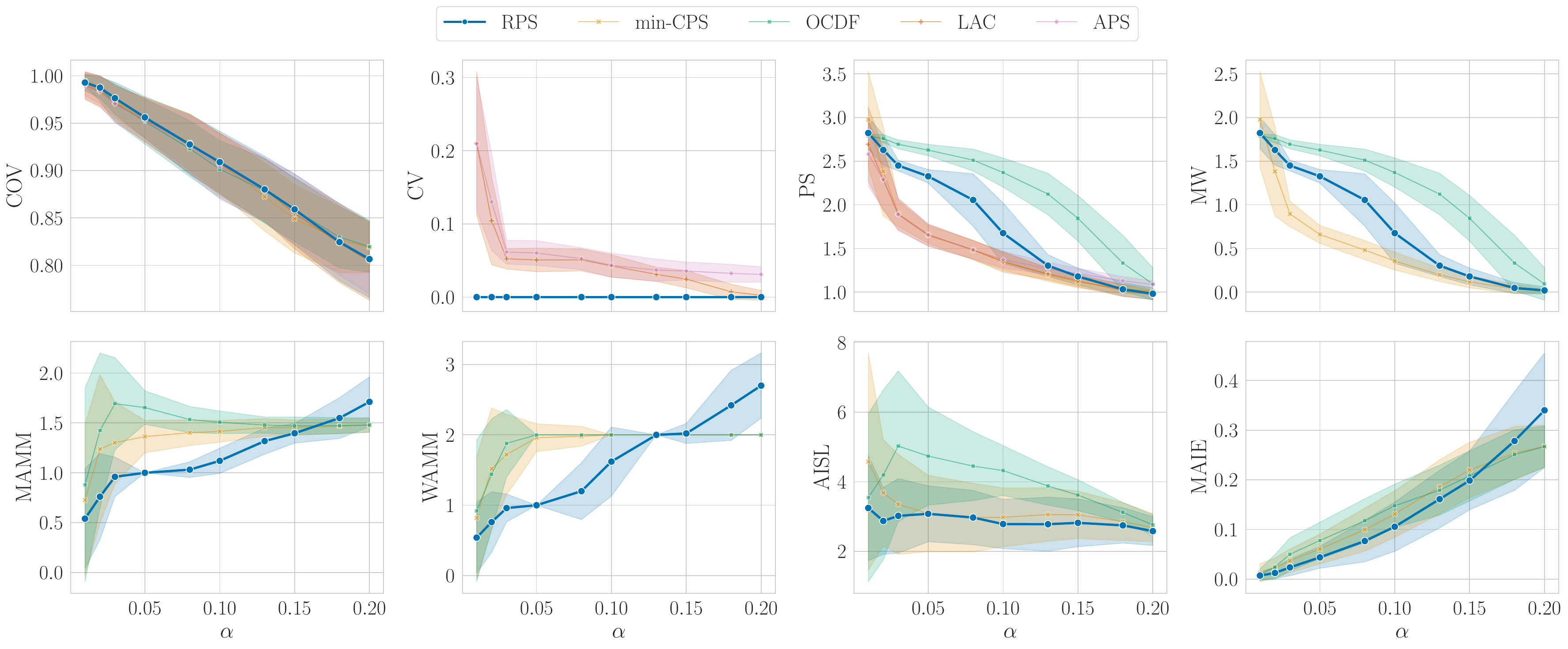}
    \caption{Comparison of prediction sets across methods on support dataset using LightGBM. Shaded regions indicate standard deviation.}
 \end{minipage}
\end{figure*}

\begin{figure*}
        \begin{minipage}[!htb]{\textwidth}
        \centering
\includegraphics[width=\linewidth]{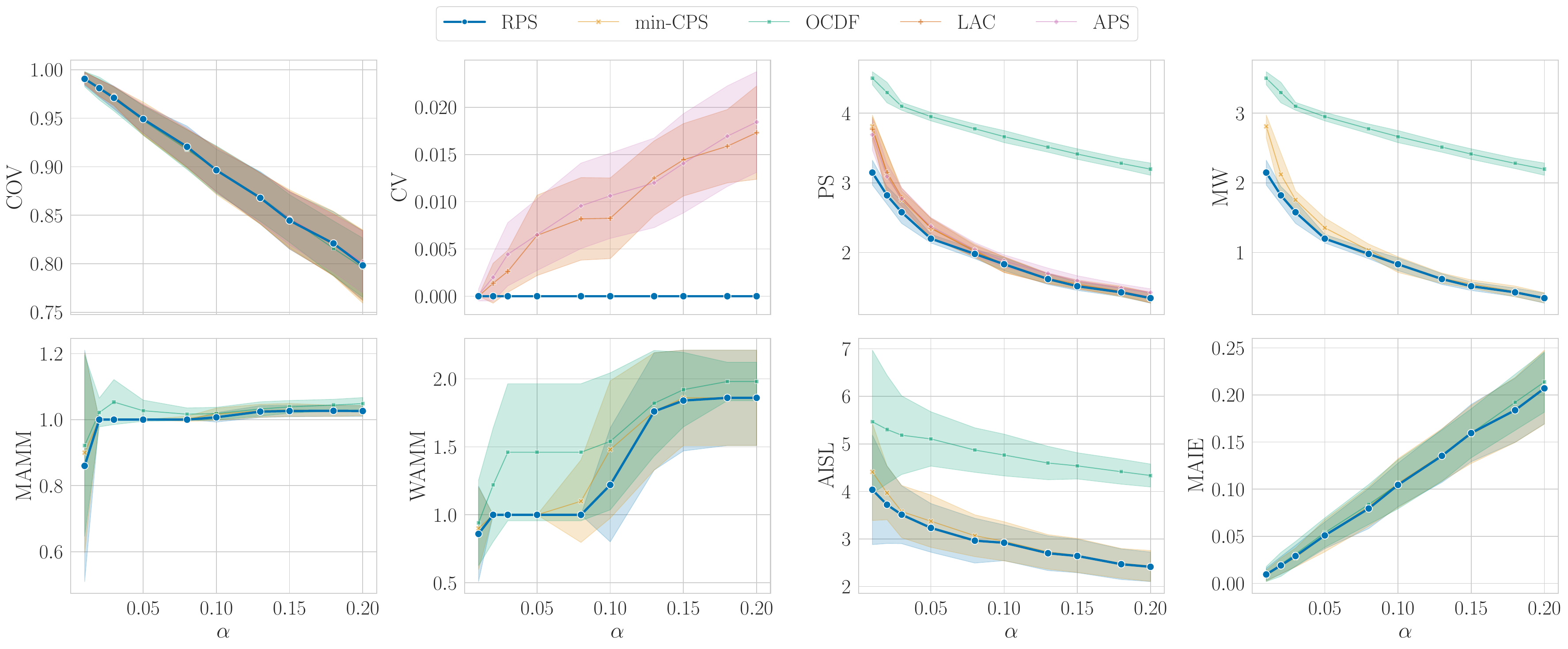}
    \caption{Comparison of prediction sets across methods on winequalityRed dataset using LightGBM. Shaded regions indicate standard deviation.}
 \end{minipage}
         \begin{minipage}[!htb]{\textwidth}
        \centering
\includegraphics[width=\linewidth]{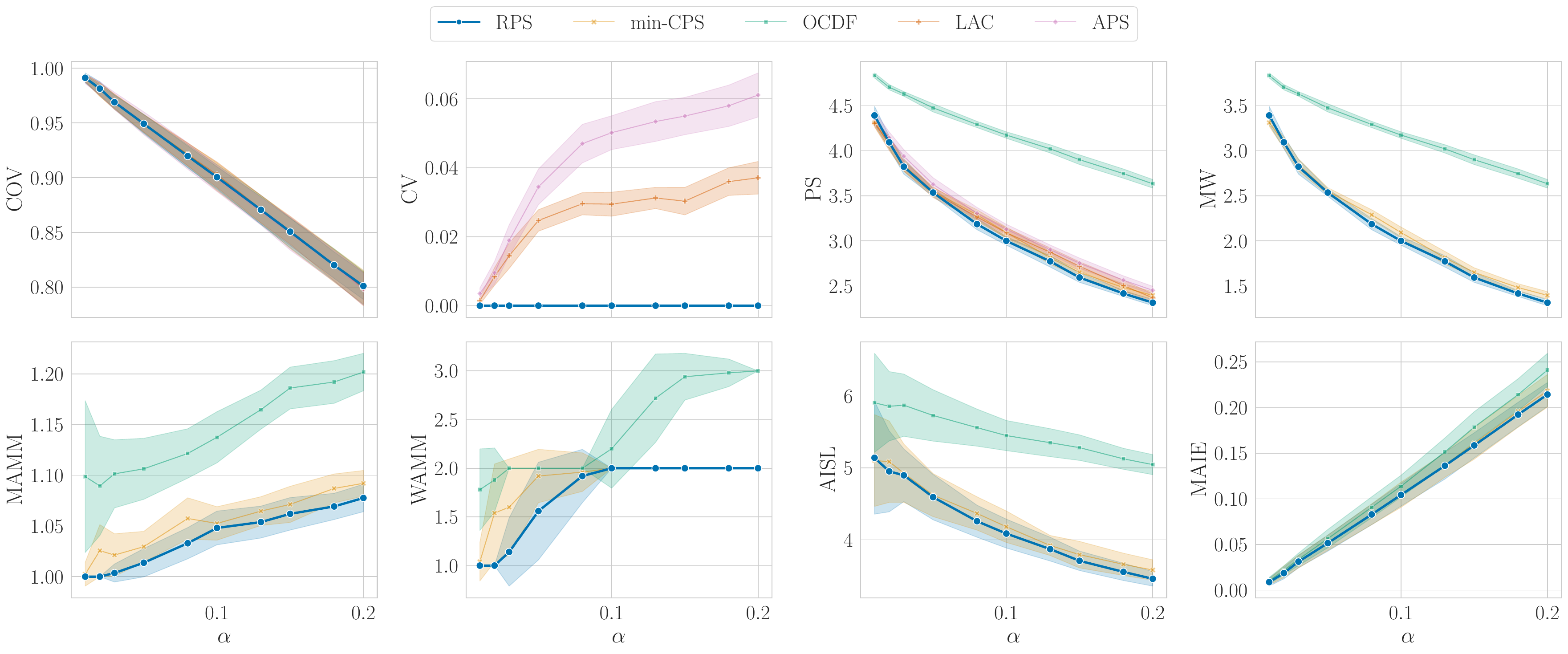}
    \caption{Comparison of prediction sets across methods on nhanes dataset using LightGBM. Shaded regions indicate standard deviation.}
 \end{minipage}
          \begin{minipage}[!htb]{\textwidth}
        \centering
\includegraphics[width=\linewidth]{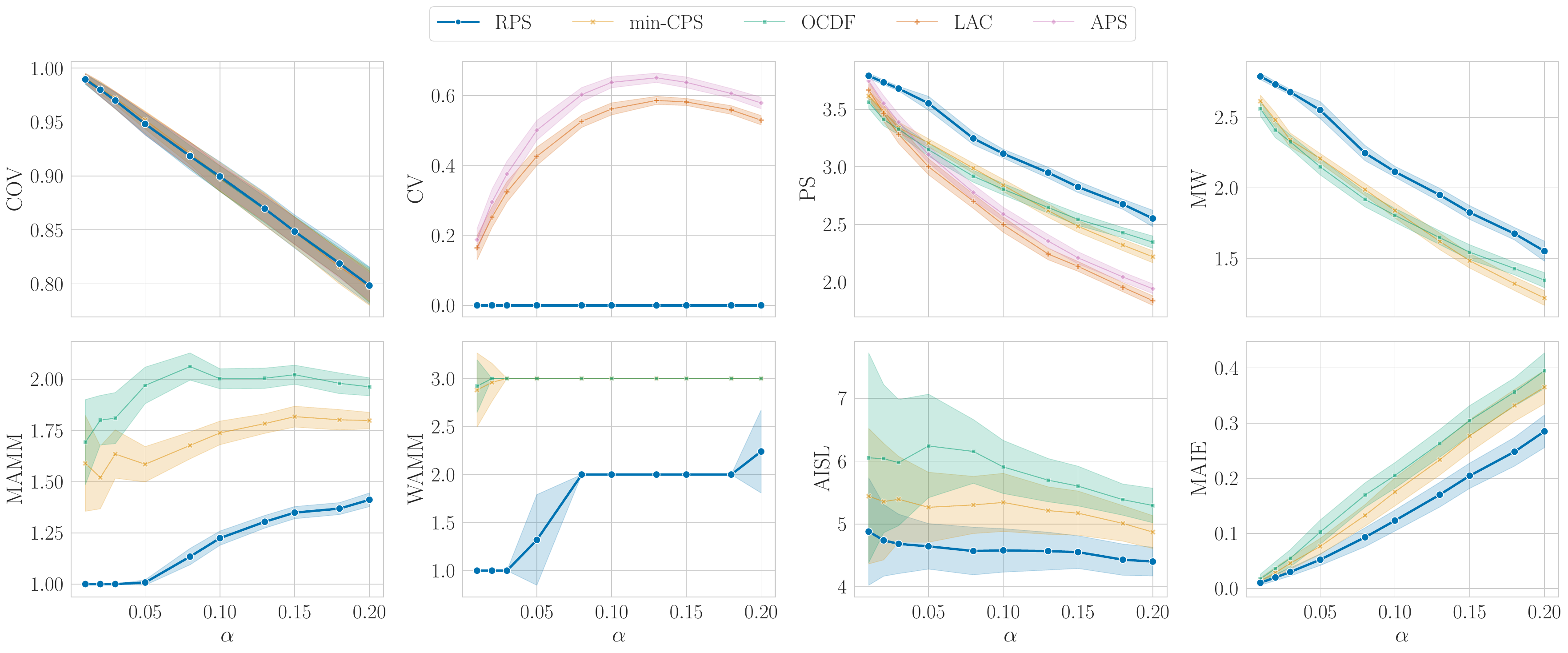}
    \caption{Comparison of prediction sets across methods on LEVXSensors dataset using LightGBM. Shaded regions indicate standard deviation.}
 \end{minipage}
 \end{figure*}




\end{document}